\documentclass{article} % For LaTeX2e
\usepackage{iclr2027_conference,times}

\usepackage{amsmath,amsfonts,bm}

\def\eqref#1{equation~\ref{#1}}
\def\1{\bm{1}}

\def\vzero{{\bm{0}}}

\def\vg{{\bm{g}}}

\def\mA{{\bm{A}}}
\def\mB{{\bm{B}}}

\def\mG{{\bm{G}}}
\def\mH{{\bm{H}}}
\def\mI{{\bm{I}}}

\def\mM{{\bm{M}}}

\def\mS{{\bm{S}}}
\def\mT{{\bm{T}}}
\def\mU{{\bm{U}}}
\def\mV{{\bm{V}}}
\def\mW{{\bm{W}}}
\def\mX{{\bm{X}}}
\def\mY{{\bm{Y}}}
\def\mZ{{\bm{Z}}}

\DeclareMathAlphabet{\mathsfit}{\encodingdefault}{\sfdefault}{m}{sl}
\SetMathAlphabet{\mathsfit}{bold}{\encodingdefault}{\sfdefault}{bx}{n}

\def\gC{{\mathcal{C}}}

\def\gF{{\mathcal{F}}}

\def\gO{{\mathcal{O}}}
\def\gP{{\mathcal{P}}}

\def\sB{{\mathbb{B}}}

\newcommand{\E}{\mathbb{E}}

\newcommand{\R}{\mathbb{R}}

\DeclareMathOperator*{\argmin}{arg\,min}

\usepackage{hyperref}
\usepackage{url}

\usepackage{booktabs}       % professional-quality tables
\usepackage{amsfonts}       % blackboard math symbols
\usepackage{nicefrac}       % compact symbols for 1/2, etc.
\usepackage{microtype}      % microtypography
\usepackage{xcolor}         % colors

\usepackage[linesnumbered,ruled,lined]{algorithm2e}
\usepackage{algpseudocode}
\usepackage{wrapfig}

\usepackage{amsmath}
\usepackage{amssymb}
\usepackage{mathtools}
\usepackage{amsthm}

\usepackage{enumerate}   
\usepackage{enumitem}
\usepackage{multirow}

\usepackage{minitoc}

\theoremstyle{plain}
\newtheorem{theorem}{Theorem}[section]
\newtheorem{proposition}[theorem]{Proposition}
\newtheorem{lemma}[theorem]{Lemma}
\newtheorem{corollary}[theorem]{Corollary}
\newtheorem{definition}[theorem]{Definition}
\newtheorem{assumption}[theorem]{Assumption}

\theoremstyle{remark}
\newtheorem{remark}[theorem]{Remark}

\newcommand\norm[1]{\left\lVert#1\right\rVert}
\SetAlgoNlRelativeSize{0}

\newcommand{\rowlabel}[1]{\rotatebox{90}{\parbox{3cm}{\centering\small #1}}}

\title{Musec: MomentUm SpEctral Clipping for \\ Stable Muon-type Training }

\author{Zhuanghua Liu \\
National University of Singapore
\And
Menglian Wang \\
Wonders Information Co., Ltd. 
\AND
Luo Luo \\
Fudan University
}

\iclrfinalcopy % Uncomment for camera-ready version, but NOT for submission.
\begin{document}

\maketitle

\begin{abstract}
Muon has emerged as a highly effective optimizer for large language model training, often achieving superior convergence and performance compared with the widely adopted Adam and AdamW optimizers. 
Nevertheless, Muon is prone to training instability due to its spectral flattening, manifested by loss spikes and unbounded growth of model weights. 
Existing approaches primarily rely on weight or attention-logit clipping, which require architecture-specific modifications and do not directly address instability across all model components. 
We propose MomentUm SpEctral Clipping (Musec), which replaces Muon's spectral flattening with spectral clipping: rather than setting all singular values of the momentum matrix to approximately one, Musec clips singular values that exceed a threshold while preserving the underlying spectral structure of the momentum. 
Our strategy provides an optimizer-level, architecture-agnostic mechanism for stabilizing Muon training. 
Theoretically, we establish convergence guarantees for Musec in nonconvex nonsmooth stochastic optimization. Practically, we develop Soft Musec, an efficient implementation that uses a smooth spectral saturation function approximated by coupled Newton--Schulz iterations. Empirically, Soft Musec consistently improves training stability over existing Muon variants across a wide range of learning rates and model sizes, remaining stable in settings where existing Muon variants diverge while matching their performance under well-tuned configurations.
\end{abstract}

\section{Introduction}\label{sec:intro}
Matrix-aware optimizers have recently emerged as strong alternatives to Adam \citep{kingma2015adam} and AdamW \citep{loshchilov2019decoupled} for training large language models. Methods such as Shampoo \citep{gupta2018shampoo}, SOAP \citep{vyas2025soap}, Muon \citep{jordan2024muon,liu2025muon}, and Scion \citep{pethick2025training} exploit the two-dimensional structure of weight matrices to achieve faster convergence.
%Adam \citep{kingma2015adam,loshchilov2019decoupled} has been the dominant optimizer for training deep neural networks over the past decade. Recently, its dominance has been challenged by a family of matrix-aware optimizers, including Shampoo \citep{gupta2018shampoo}, SOAP \citep{vyas2025soap}, Muon \citep{jordan2024muon,liu2025muon}, and Scion \citep{pethick2025training}.
Among these, Muon stands out as the simplest and has already been adopted in several large-scale commercial LLM training efforts, including Kimi K2 \citep{team2025kimi}, GLM-5 \citep{zeng2026glm}, and DeepSeek-V4 \citep{xu2026deepseek}.

Despite the fast convergence, Muon training is known to suffer from training instability, frequently exhibiting exploding weight norms and divergent attention logits~\citep{team2025kimi}.
Muon replaces the singular values of its momentum matrix with approximately one, producing an update with a flat singular spectrum. This spectral flattening can inject substantial update magnitude into directions that have relatively small singular values in the original momentum matrix, contributing to unstable weight growth \citep[Appendix E]{team2025kimi}.
Several approaches have been proposed to address this instability. Logit soft-capping \citep{team2024gemma} bounds attention scores after computation, but does not prevent the underlying query-key dot products from growing excessively. QK-Norm \citep{dehghani2023scaling,wortsman2024small} normalizes query and key vectors directly, but is incompatible with architectures such as multi-head latent attention (MLA) \citep{liu2024deepseek}, where key matrices are not explicitly materialized during inference. MuonClip \citep{team2025kimi} takes a different approach by applying weight clipping to the query-key matrices to constrain attention logits.
However, it is inherently architecture-specific and only constrains the query-key weight matrices, leaving the value-output and MLP weights entirely unaddressed.
These approaches treat the symptoms of Muon's instability rather than its root cause, which lies in Muon's spectral transformation itself.

In this paper, we propose MomentUm SpEctral Clipping (Musec), which replaces Muon's spectral flattening with a spectral clipping operator. 
While Muon sets all singular values to one, Musec only clips singular values that exceed a threshold. %, preserving the natural spectral skewness of the momentum matrix. 
The resulting update preserves the spectral skewness of the momentum matrix, thereby avoiding the amplification of directions associated with relatively small momentum singular values. 
Unlike existing remedies that target specific architectural components, such as capping attention logits \citep{team2024gemma} or clipping query-key weights \citep{team2025kimi}, Musec operates directly within the optimizer and is therefore architecture-agnostic, stabilizing all weight matrices uniformly. 
Moreover, we show that Musec's update admits a principled interpretation as constrained steepest descent: it solves a Frobenius-norm-penalized quadratic subproblem subject to a spectral-norm constraint, complementing the unconstrained steepest descent interpretation of Muon established by \citet{bernstein2024old}.
We demonstrate the effectiveness of Musec both theoretically and empirically. 
%Despite the growing empirical adoption of Muon-type optimizers, their convergence properties beyond the smooth setting are less explored. 
Despite the growing empirical adoption of Muon-type optimizers, their convergence properties beyond the smooth setting remain poorly understood. In fact, \citet{jiang2026adaptive} recently showed that Muon does not converge in the general nonconvex nonsmooth setting, highlighting the need for principled modifications to its update rule.
Since training objectives in modern deep neural networks are generally nonconvex and may involve nonsmooth components, such as piecewise linear activations and discrete routing mechanisms, we provide convergence guarantees for Musec in the general nonconvex nonsmooth setting.
We summarize our main contributions:

\begin{itemize}
    \item Theoretically, we consider the following stochastic optimization problem:
    \begin{equation}\label{obj}
        \min_{\mW \in \R^{m\times n}} f(\mW) = \E_{\xi\sim \gP}[F(\mW; \xi)],
\end{equation}
where $\gP$ is some unknown distribution, the objective function $f(\mW)$ is $\rho$-weakly convex and the stochastic component function $F(\mW; \xi)$ is possibly nonconvex and nonsmooth.
We show that Musec converges to a $(\delta, \epsilon)$-Goldstein stationary point of the objective with $\gO(r^{{3}/{2}} \delta^{-1} \epsilon^{-3}  + \delta^2 \rho^3 r^{- {3}/{2}} + r^{{3}/{4}}\delta^{-1})$ stochastic gradient oracle calls, where $r = \min(m, n)$ is the smaller matrix dimension.
As long as the weak convexity parameter satisfies $\rho\leq \gO(r \delta^{-1} \epsilon^{-1})$, the dominant term of our complexity is $\gO(r^{{3}/{2}} \delta^{-1} \epsilon^{-3})$. 
The dependence on the accuracy parameters $\delta$ and $\epsilon$ matches that of the optimal stochastic first-order method for nonconvex nonsmooth optimization~\citep{cutkosky2023optimal}.
While \citet{jiang2026adaptive} showed that Muon with the exact polar update does not converge in the nonconvex nonsmooth setting, \citet{li2026muon} proved that finite Newton--Schulz iterations restore convergence through implicit smoothing. We provide a complementary result: Musec converges in the nonconvex nonsmooth setting through explicit spectral clipping, offering both convergence guarantees and improved training stability.
%While \citet{jiang2026adaptive} showed that Muon itself does not converge in the nonconvex nonsmooth setting, we prove that Musec provides a convergence guarantee for the nonsmooth problem.
%To the best of our knowledge, this constitutes the first convergence analysis of Muon-type algorithms in the nonconvex nonsmooth setting.
\item We further develop an efficient implementation, Soft Musec, which replaces hard spectral clipping with a smooth saturation function and approximates the resulting matrix transformation using coupled Newton-Schulz iterations.
Empirically, Soft Musec achieves more stable training than existing Muon variants across multiple datasets, including FineWeb, OpenWebText, and C4, over a substantially wider range of learning rates. 
Notably, Soft Musec maintains stable convergence in settings where existing Muon variants diverge, while matching their performance under well-tuned configurations.
\end{itemize}

\paragraph{Concurrent work.} Concurrently with our work, \cite{jiang2026enhancing} proposed spectral clipping as an optimizer-agnostic wrapper, SPECTRA, that can be applied to various base optimizers, including AdamW \citep{loshchilov2019decoupled} and Signum \citep{bernstein2018signsgd}. \cite{yi2026mucon} studied singular-value clipping as a replacement for Muon's polar step (MuCon), with a focus on the numerical challenges of SVD-free approximation. 
\cite{jiang2026adaptive} established nonconvex nonsmooth convergence guarantees for Pion and Leon, two matrix-aware optimizers derived from adaptive online learning, and showed that Muon with the exact polar update does not converge in this setting. \cite{li2026muon} subsequently proved that finite Newton--Schulz iterations restore convergence through implicit smoothing of the polar map. We independently identify spectral clipping as a natural replacement for Muon's spectral flattening, motivated by its ability to preserve the spectral structure of momentum matrices while suppressing excessively large singular values. Our work complements these concurrent studies by providing convergence guarantees for spectral clipping-based Muon variants in the nonconvex nonsmooth setting, together with systematic empirical evaluation demonstrating consistent stability improvements across multiple datasets and model sizes.
Please refer to Section~\ref{sec:algo} and Appendix \ref{app:spectra_comp} for a detailed discussion.

\paragraph{Paper Organization.} Section~\ref{sec:related_work} reviews related work on matrix-aware optimizers and nonconvex nonsmooth optimization. Section~\ref{sec:prelim} establishes notation, formalizes assumptions, and reviews the Muon optimizer along with its training instability. Section~\ref{sec:methodology} introduces Musec and presents its convergence analysis. Section~\ref{sec:practical} develops Soft Musec, an efficient SVD-free implementation via coupled Newton--Schulz iterations. Section~\ref{sec:experiments} provides experimental evaluation on NanoGPT models across multiple datasets and model scales. Section~\ref{sec:conclusion} concludes the paper.

\section{Related Work}\label{sec:related_work}

%In this section, we review related work on matrix-aware optimizers and nonconvex nonsmooth optimization.

\paragraph{Matrix-Aware Optimizers}
Matrix-aware optimizers exploit the two-dimensional structure of weight matrices to improve upon coordinate-wise methods such as Adam \citep{kingma2015adam}. Shampoo \citep{gupta2018shampoo} and SOAP \citep{vyas2025soap} leverage Kronecker-factored preconditioning, while Muon \citep{jordan2024muon} orthogonalizes the momentum matrix via Newton--Schulz iterations, which \citet{bernstein2024old} showed corresponds to steepest descent under the spectral norm. Scion \citep{pethick2025training} further develops this viewpoint through the Frank--Wolfe framework. Subsequent work has explored variants of Muon, such as NorMuon \citep{li2025normuon}, which introduces neuron-wise adaptive scaling after orthogonalization. Concurrently, SPECTRA \citep{jiang2026enhancing} proposed spectral clipping as an optimizer-agnostic wrapper with convex smooth convergence analysis, and \cite{yukhimchuk2026gradient} proposed spectral gradient clipping for heavy-tailed noise. On the theoretical side, \citet{shen2025convergence} and Gluon \citep{riabinin2025gluon} established convergence guarantees under smoothness or generalized smoothness, and \citet{yang2026convergence} analyzed Spectral Descent under convexity and sharpness. Concurrently, \citet{li2026muon} showed that finite Newton--Schulz iterations restore convergence in the nonsmooth nonconvex setting through implicit smoothing. Our work provides a complementary perspective, establishing convergence guarantees for spectral clipping-based Muon variants in this regime.

\paragraph{Nonconvex Nonsmooth Optimization}
The foundations of nonsmooth optimization trace back to the seminal work of \cite{clarke1975generalized} and \cite{goldstein1977optimization}. Non-asymptotic complexity guarantees for general nonconvex nonsmooth settings remained elusive until \cite{zhang2020complexity} established the first non-asymptotic complexity guarantees for subgradient methods converging to Goldstein stationary points. This breakthrough led to a series of subsequent developments \citep{davis2022gradient,tian2022finite,kornowski2022oracle,cutkosky2023optimal,jordan2023deterministic}.
In particular, \cite{cutkosky2023optimal} introduced the online-to-nonconvex (O2NC) framework, achieving the first optimal convergence rates for stochastic nonconvex nonsmooth optimization. Our analysis considers the class of weakly convex functions \citep{duchi2018stochastic,davis2019stochastic,davis2019proximally}, which encompasses a broad class of objectives arising in neural network training. 
For this class of problem, \citet{ji2026derandomized} showed that the O2NC framework can be derandomized while preserving the optimal $\gO(\delta^{-1} \epsilon^{-3})$ complexity. 
Recently, \citet{jiang2026adaptive} established nonconvex nonsmooth convergence guarantees for Pion and Leon, two matrix-aware optimizers derived from adaptive online learning over the spectral-norm ball. 
Notably, they also demonstrated that Muon with the exact polar update does not converge in this setting, while \citet{li2026muon} proved that finite Newton--Schulz iterations restore convergence through implicit smoothing of the polar map. We provide a complementary result: Musec converges in the nonconvex nonsmooth setting through explicit spectral clipping, offering both convergence guarantees and improved training stability.
% We address this gap by extending O2NC to the matrix-valued setting with spectral clipping, establishing the first convergence guarantees for Muon-type algorithms in the nonconvex nonsmooth regime.

\section{Preliminaries and Background}\label{sec:prelim}

In this section, we establish notation, formalize the problem setting and assumptions, and review the Muon optimizer along with its training instability.

\subsection{Notations}
Throughout this paper, $\norm{\cdot}_2$ and $\norm{\cdot}_F$ denote the spectral norm and the Frobenius norm of a matrix, respectively,
and $\langle \cdot, \cdot \rangle$ denotes the Frobenius inner product. 
For a set of matrices $\Omega \subseteq \R^{m \times n}$, we denote ${\rm dist}(\vzero, \Omega) \coloneqq \inf_{\mM \in \Omega} \norm{\mM}_F$ and ${\rm conv}(\Omega)$ for the convex hull of $\Omega$.
For any positive integer $N$, we abbreviate $[N] = \{1, \dots, N\}$.
We use $\sB_{\delta}(\mW) = \{\mV \in \R^{m \times n} \colon \norm{\mV - \mW}_F \leq \delta \}$ to denote the closed ball of radius $\delta$ centered at $\mW \in \R^{m \times n}$.
For a scalar $w \geq 0$ and a threshold~$D > 0$, we define the clipping operator as ${\rm clip}(w, D) =  \min \{w, D\}$.
For a diagonal matrix $\mM \in \R^{r \times r}$, we apply the clipping entrywise:
\begin{align*}
    \mathbf{Clip}(\mM, D) = \mathbf{Diag}({\rm clip}(\mM_{1,1}, D), \dots, {\rm clip}(\mM_{r,r}, D)),
\end{align*}
where $\mM_{i,i}$ is the $(i,i)$-th entry of the matrix $\mM$. 

\subsection{Problem Formulation}

We now formalize the problem setting and state the assumptions used throughout this paper. We begin with the definitions of Lipschitz continuity and weak convexity.

\begin{definition}
    We say a function $h \colon \R^{m \times n} \to \R$ is $L$-Lipschitz continuous if 
    \begin{align*}
        |h(\mW_1) - h(\mW_2)| \leq L \norm{\mW_1 - \mW_2}_F,
    \end{align*}
    for all $\mW_1, \mW_2 \in \R^{m \times n}$.
\end{definition}

The Clarke subdifferential \citep{Clarke1990} of a Lipschitz function $h$ at $\mW \in \R^{m \times n}$ is denoted by~$\partial h(\mW)$.
We next define weakly convex functions.

\begin{definition}
    We say a function $h \colon \R^{m \times n} \to \R$ is $\rho$-weakly convex for some $\rho > 0$ if the quadratically regularized function $h(\cdot) + (\rho / 2) \norm{\cdot}_F^2$ is convex, or equivalently
    \begin{align*}
        h(\mW_1) \geq h(\mW_2) + \langle \vg, \mW_1 - \mW_2 \rangle - \frac{\rho}{2} \norm{\mW_1 - \mW_2}_F^2,
    \end{align*}
    for all $\mW_1, \mW_2 \in \R^{m \times n}, \vg \in \partial h(\mW_2)$.
\end{definition}

In the remainder of this paper, we suppose Problem (\ref{obj}) satisfies the following assumptions.

\begin{assumption}{\label{asm:lip}}
    For any $\xi \sim \gP$, the loss function $F(\mW; \xi)$ is $L$-Lipschitz continuous with respect to its first argument.
    Furthermore, the objective $f(\mW)$ is $\rho$-weakly convex.
\end{assumption}

\begin{assumption}{\label{asm:bound_var}}
    For each $\mW \in \R^{m \times n}$, we can access to a stochastic oracle $G(\mW; \xi)$ that is an unbiased estimator of a Clarke subgradient such that $\E[G(\mW; \xi)] \in \partial f(\mW)$.  Given a positive constant $\sigma > 0$, we further assume bounded variance: 
    \begin{align*}
     \E[\norm{G(\mW; \xi) - \E[G(\mW; \xi)]}_F^2] \leq \sigma^2.
    \end{align*}
\end{assumption}

We assume that the objective is bounded below, i.e., $f^* = \inf_{\mW \in \R^{m \times n}}f(\mW) > - \infty$ and we denote $\Delta_{f} \coloneqq f(\mW_0) - f^*$.

For general nonsmooth nonconvex objectives, convergence is standardly measured via Goldstein stationarity \citep{goldstein1977optimization}.

\begin{definition}
    The Goldstein $\delta$-subdifferential of a Lipschitz function $f$ at a point $\mW \in \R^{m\times n}$ is the convex hull of all Clarke subgradients at points in a $\delta$-ball around $\mW$, i.e.,
    \begin{align*}
        \partial_\delta f(\mW) \coloneqq {\rm conv}\left\{\bigcup_{\mV \in \sB_{\delta}(\mW)} \partial f(\mV)\right\}.
    \end{align*}
    A point $\mW \in \R^{m\times n}$ is called a $(\delta, \epsilon)$-stationary point if ${\rm dist}(\vzero, \partial_\delta f(\mW)) \leq \epsilon$.
\end{definition}

\subsection{The Muon Optimizer}
Muon \citep{jordan2024muon} is an optimizer designed for matrix-valued parameters. At each iteration, it orthogonalizes the momentum matrix, replacing its nonzero singular values with one and thereby flattening the singular spectrum. The complete procedure is presented in Algorithm \ref{alg:muon}. Since computing the full SVD is prohibitively expensive in practice, Muon is typically implemented using Newton--Schulz iterations \citep{higham2008functions} to approximate this orthogonalization \citep{jordan2024muon,bernstein2024old}.

\paragraph{Training instability of Muon.} Let $\mM_n=\mU_n\mS_n\mV_n^\top$ denote the SVD of the momentum matrix at step $n$. The Muon update $\mU_n \mV_n^\top$ replaces all nonzero singular values of $\mM_n$ with one, discarding the spectral magnitude information entirely.
Consequently, directions associated with small momentum singular values receive updates of the same magnitude as those associated with large singular values, amplifying weakly represented directions and contributing to weight growth and training instability~\citep[Appendix E]{team2025kimi}. 
These observations motivate replacing spectral flattening with spectral clipping, which preserves the spectral structure of the momentum while constraining only excessively large singular values.

\IncMargin{1em}
\begin{algorithm}[!t]
\caption{Muon \citep{jordan2024muon}}
\label{alg:muon}
\textbf{Input:} Initial point $\mW_0$, momentum parameter $\beta \in [0, 1)$, learning rate $\eta$.\vspace{0.05cm}

\For {$n = 0, 1, \dots, N-1$}
{
Sample $\xi_{n} \sim \gP$ 

$\mG_n =  G(\mW_n; \xi_{n})$

\eIf{$n = 0$}
{
  $\mM_0 = \mG_0$
}
{
$\mM_n = (1 - \beta) \mM_{n-1} + \beta \mG_n$
}

$(\mU_n, \mS_n, \mV_n) = {\rm SVD}({\mM}_n)$ 

$\mW_{n + 1} = \mW_n - \eta \mU_n \mV_n^\top$ 
}
\end{algorithm}
\DecMargin{1em}

\section{Methodology} \label{sec:methodology}
In this section, we introduce MomentUm SpEctral Clipping (Musec) and establish its convergence guarantees for nonconvex nonsmooth optimization.
\subsection{The Algorithm}\label{sec:algo}
\IncMargin{1em}
\begin{algorithm}[!t]
\caption{Musec: MomentUm SpEctral Clipping}
\label{alg:musec}
\textbf{Input:} Initial point $\mW_0$, momentum parameter $\beta \in [0, 1)$, learning rate $\eta$, clipping threshold $D>0$, positive integers $K$ and $T$ \vspace{0.05cm}

$N = K \times T$ \vspace{0.05cm}

\For {$n = 0, 1, \dots, N-1$}
{
Sample $\xi_{n} \sim \gP$ \vspace{0.05cm}

$\mG_n =  G(\mW_n; \xi_{n})$ \vspace{0.05cm}

\eIf{$n = 0$}
{
  $\widehat{\mM}_0 = \mG_0$ \vspace{0.05cm}
}
{
  $\widehat{\mM}_n = (1 - \beta) \mM_{n-1} + \beta \mG_n$ \vspace{0.05cm}
}

$(\mU_n, \widehat{\mS}_n, \mV_n) = {\rm SVD}(\widehat{\mM}_n)$\label{svd_op} \vspace{0.05cm}

$\mS_n = {\rm clip}(\widehat{\mS}_n, D)$ \vspace{0.05cm}

$\mM_n = \mU_n \mS_n \mV_n^\top$\label{clip_op} \vspace{0.05cm}

$\mW_{n + 1} = \mW_n - \eta \mM_n$\label{update_step} \vspace{0.05cm}
}

 Set $\overline{\mW}^{(k)} = \frac{1}{T} \sum_{t=0}^{T-1} \mW_t^{(k)}$ where  $\mW_t^{(k)} = \mW_{(k-1)T + t}$ for $\forall k \in [K]$. \label{return_1} \vspace{0.05cm}

 \textbf{Return:} $\overline{\mW}_T \sim {\rm Uniform} (\{\overline{\mW}^{(k)}: k \in [K]\})$. \label{return_2} \vspace{0.05cm}
\end{algorithm}

We present the complete procedure of Musec in Algorithm \ref{alg:musec}.
The key difference from Muon lies in Step \ref{clip_op}, where we replace the orthogonalized momentum with a spectrally clipped momentum:
\begin{align*}
    \mM_n = \mU_n \mS_n \mV_n^\top,~~~{\rm where}~~~\mS_n = \mathbf{Clip}(\widehat{\mS}_n, D).
\end{align*}
Here, $(\mU_n, \widehat{\mS}_n, \mV_n) = {\rm SVD}(\widehat{\mM}_n)$ is the singular value
decomposition of the momentum $\widehat{\mM}_n$, with~$\widehat{\mS}_n$ the diagonal matrix
of its singular values, and $D > 0$ is the clipping threshold. 
Unlike Muon's polar step, which maps all nonzero singular values to one, the clipping operator only truncates singular values exceeding the threshold $D$ while leaving smaller singular values unchanged. 
Thus, Musec preserves the spectral structure of the momentum while controlling large spectral components.

In contrast to SPECTRA \citep{jiang2026enhancing}, which applies spectral clipping as a post-processing step to the accumulated momentum while carrying the unclipped momentum state forward, Musec feeds the clipped momentum $\mM_{n-1}$ back into the exponential moving average (EMA) update at the next step, integrating spectral clipping directly into the momentum recurrence. 
This ensures that the momentum state itself remains spectrally bounded, allowing singular values to decay in directions where recent gradients are small, rather than being dominated by historical accumulation.

The return procedure in Steps \ref{return_1}--\ref{return_2} partitions the iterates into $K$ consecutive epochs of length $T$, averages within each epoch, and returns one epoch average uniformly at random. 
This two-level averaging scheme is standard in the nonconvex nonsmooth optimization literature \citep{cutkosky2023optimal, ji2026derandomized} and is required to establish convergence to Goldstein stationary points. In practice, it is sufficient to return the final iterate $\mW_n$, as is conventional in deep learning.

\subsection{Musec as Constrained Steepest Descent}

\citet{bernstein2024old} showed that Muon's polar update solves the unconstrained steepest descent problem under the spectral norm. We establish an analogous characterization for Musec: its spectrally clipped update arises as the solution to a Frobenius-norm-penalized quadratic subproblem subject to a spectral-norm constraint.

\begin{proposition}\label{prop:constrained}
    Let $\mG_1, \dots, \mG_L$ be gradient matrices, let  $\lambda > 0$ be a sharpness parameter, and let~$D > 0$. For each $l = 1, \dots, L$, consider the problem
     \begin{align}\label{constrain_prob}
        \min_{\bm{\Delta}_l} \left[ \langle \mG_l, \bm{\Delta}_l\rangle + \frac{\lambda}{2} \norm{\bm{\Delta}_l}_F^2\right], ~~~{\rm s.t.}~~~\norm{\bm{\Delta}_l}_2 \leq D / \lambda.
    \end{align}
    where $\langle \cdot, \cdot \rangle$ denotes the Frobenius inner product and $\bm{\Delta}_l$ has the same shape as $\mG_l$.
    Let $\mG_l \in \R^{m_l \times n_l}$ have reduced SVD of the form $\mG_l = \mU_l \mS_l \mV_l^{\top}$ with $\mS_l = \mathbf{Diag}(\sigma_1, \dots, \sigma_{d_l})$ and $d_l = \min(m_l, n_l)$. Then Problem (\ref{constrain_prob}) is solved by
    \begin{align*}
        \bm{\Delta}_l = - \eta \cdot \mU_l \hat{\mS}_l \mV_l^{\top}, ~~~{\rm where}~~~\eta = \frac{1}{\lambda} ,~~~\hat{\mS}_l = \mathbf{Clip}(\mS_l, D).
    \end{align*}
\end{proposition} 
The spectral-norm constraint $\norm{\Delta \mW_l}_2 \leq D / \lambda$ clips singular values that exceed $D$ while preserving smaller ones.
Thus, rather than being an ad hoc modification of Muon, Musec admits a principled interpretation as constrained steepest descent under the Frobenius norm with a spectral-norm budget.

\subsection{Convergence Analysis}
In this subsection, we present the convergence analysis of Musec for solving the nonconvex nonsmooth optimization problem.

We define the filtration $\gF_n = \sigma(\mG_1, \mG_2, \dots, \mG_{n})$.
Under Assumption \ref{asm:bound_var}, we can infer that $\mH_n \coloneqq \E[\mG_n \mid \gF_{n-1}] \in \partial f(\mW_n)$.
Analogous to $\mW_t^{(k)}$, we define variables $\mM_t^{(k)} = \mM_{(k-1) T + t}$ and $\mH_t^{(k)} = \mH_{(k-1) T + t}$ for any $k \in [K]$ and $t \in \{0\} \cup [T - 1]$, and we denote the epoch average $\overline{\mH}^{(k)} =  \sum_{t=1}^T\mH_t^{(k)} / T$.
We first establish the following convergence result for a single epoch $k$.

\begin{lemma}\label{lemma:core_recur}
      Let $\gamma = \beta / \eta$ and $\eta \leq 1 / \rho$ where $\beta \leq 1/8$, then for any $k\in [K]$, the sequence $\{\mW_t^{(k)}\}_{t=1}^T$ generated by Algorithm \ref{alg:musec} satisfies
  \begin{align*}
        &  \E\left[f(\mW_T^{(k)}) - f(\mW_0^{(k)}) \right] \\
        \leq  & -\E\left[\frac{\beta D T}{\gamma}  \norm{\overline{\mH}^{(k)}}_F + \sum_{t=0}^{T-1} \frac{\beta}{8\gamma}\norm{\mM_t^{(k)}}_F^2\right] + \frac{\beta D \sigma\sqrt{T} }{\gamma} + \left(\frac{\beta T}{\gamma} + \frac{1}{\gamma} \right)D^2 + \frac{\beta^2 L^2 T}{\gamma} + \frac{r D^2}{\gamma}.
    \end{align*}
\end{lemma}

By combining Lemma \ref{lemma:core_recur} with a bound relating $\norm{\overline{\mH}^{(k)}}_F$ to the Goldstein $\delta$-subdifferential ${\rm dist}(0, \partial_\delta f(\overline{\mW}_k))$, we obtain the following convergence guarantee for Musec (Algorithm~\ref{alg:musec}) in finding a $(\delta, \epsilon)$-Goldstein stationary point of nonconvex nonsmooth problem. 

\begin{theorem}\label{lemma:dist_goldstein}
    Let $\gamma = \beta / \eta$ and $\eta \leq 1 / \rho$ where $\beta \leq 1/8$,
    Then for any $\delta \geq \eta T D\sqrt{r}$, the sequence $\{\overline{\mW}^{(k)} \}_{k=1}^K$ generated by Algorithm~\ref{alg:musec} satisfies
    \begin{align*}
       \E\left[\frac{1}{K}\sum_{k=1}^K  {\rm dist}(\vzero, \partial_\delta f(\overline{\mW}^{(k)})) \right] \leq \frac{\sigma}{\sqrt{T}} + \left(1 + \frac{2r}{\beta T} \right)D + \frac{\beta L^2 }{D} + \frac{\gamma \Delta_{f} }{\beta D T K}.
   \end{align*}
\end{theorem}

The following corollary, which follows directly from Theorem~\ref{lemma:dist_goldstein}, establishes the oracle complexity of Musec for finding a $(\delta, \epsilon)$-Goldstein stationary point.

\begin{corollary}\label{corollary:complexity}
        By choosing the parameters as
    \begin{align*}
        & T = {(\delta N)^{{2}/{3}}}, \quad K  = \delta^{-{2}/{3}}N^{{1}/{3}}, \quad D = (\delta N)^{-{1}/{3}},  \\
        & \beta = \frac{r^{{1}/{2}}}{(\delta N)^{{2}/{3}} } ,  \quad \gamma = \frac{r}{\delta^{{4}/{3}} N^{{1}/{3}}  }, \quad \eta = \frac{\delta^{{2}/{3}}}{\sqrt{r} N^{{1}/{3}}},
    \end{align*}
    then it takes Algorithm \ref{alg:musec} at most 
    \begin{align*}
      N = \gO\left( \frac{r^{{3}/{4}}}{\delta} + \frac{\delta^2 \rho^3} { r^{{3}/{2}}} + \frac{r^{{3}/{2}} (\sigma^3 + L^6 + \Delta_{f}^3)}{\delta \epsilon^3}\right)
    \end{align*}
    to obtain a $(\delta, \epsilon)$-Goldstein stationary point.
\end{corollary}

\begin{remark}
    When the weak convexity parameter satisfies $\rho \leq \gO(r \delta^{-1} \epsilon^{-1})$, the dominant term of the complexity reduces to $\gO(r^{{3}/{2}} \delta^{-1} \epsilon^{-3})$. The dependence on the accuracy parameters $\delta$ and $\epsilon$ matches that of the optimal stochastic first-order method for nonconvex nonsmooth optimization~\citep{cutkosky2023optimal}. 
    The factor of $r^{{3}/{2}}$ arises from the matrix-valued structure of the updates. 
    For comparison, the convergence analysis of Muon under smoothness assumptions by \citet{shen2025convergence} incurs a factor of $r^2$ in the stochastic setting, though we note that the two results target different stationarity concepts and operate under different assumptions.
\end{remark}

\section{A Practical Implementation}\label{sec:practical}
Computing the exact spectral clipping operator requires a full SVD decomposition at each iteration (step \ref{svd_op} in Algorithm~\ref{alg:musec}), which is prohibitively expensive for the large weight matrices encountered in LLM training. In this section, we develop an efficient SVD-free approximation of the clipping operation. We term the resulting algorithm Soft Musec.

Specifically, we replace the hard clipping function ${\rm clip}(w, D) = \min(w, D)$ with the smooth saturation function
\begin{align*}
h(w, D) = \frac{D   w}{\sqrt{w^2 + D^2}},
\end{align*}
which satisfies that $h(w, D) \approx w$ when $w \ll D$ and $h(w, D) \approx D$ when $w \gg D$.
Applying $h(\cdot, D)$ to each singular value of a matrix $\mM \in \R^{m \times n}$ while preserving its singular vectors yields the soft spectral clipping operator:
\begin{align}\label{spectral_clip_matrix}
    H(\mM, D) = D(\mM \mM^\top + D^2 \mI)^{-1/2} \mM.
\end{align}
The key computational challenge in evaluating the operator (\ref{spectral_clip_matrix}) is the matrix inverse square root $(\mM \mM^\top + D^2 \mI)^{-1/2}$. We approximate this using coupled Newton--Schulz iterations (\citealt[Chapter~6]{higham2008functions}; \citealt{jiang2026enhancing}), which involve only matrix--matrix multiplications and are therefore well-suited to modern GPU hardware. The addition of $D^2 \mI$ shifts all eigenvalues away from zero, ensuring numerical stability and rapid convergence of the iteration. In practice, we find that five Newton--Schulz iterations yield strong training performance in bfloat16 arithmetic.
The complete Soft Musec algorithm is provided in Algorithm~\ref{alg:soft_musec} in Appendix~\ref{app:practical_soft}.

\section{Experiments}\label{sec:experiments}

In this section, we evaluate the effectiveness and stability of Soft Musec on LLM training.

\subsection{Experimental Setup}
We compare Soft Musec against three baselines: Muon \citep{jordan2024muon}, MuonClip \citep{team2025kimi}, and SPECTRA \citep{jiang2026enhancing}. For SPECTRA, we use SGDM (stochastic gradient descent with momentum) as the base optimizer, which corresponds to Muon without the orthogonalization step. This isolates the effect of spectral clipping from Muon's orthogonalization.
We note that \citet{jiang2026enhancing} evaluated SPECTRA on AdEMAMix \citep{pagliardini2025ademamix}, AdamW \citep{loshchilov2019decoupled}, and Signum \citep{bernstein2018signsgd}; to our knowledge, our experiments provide the first empirical evaluation of spectral clipping applied to Muon in this setting.
For all methods, we tune weight decay over $\{0.096,0.3,0.6,1.2\}$. 
All methods use five Newton--Schulz iterations.
For Soft Musec and SPECTRA, we additionally tune the clipping threshold 
$D$ over $\{0.05,0.125,0.25,0.5,0.75,1.0\}$.

We adopt the modded-nanogpt codebase\footnote{\url{https://github.com/kellerjordan/modded-nanogpt}}, a decoder-only Transformer with modern architectural modifications, including RMSNorm \citep{zhang2019root}, Rotary Positional Embeddings (RoPE) \citep{su2024roformer}, and square ReLU activations \citep{so2021searching}.
We evaluate three model configurations—NanoGPT-Small (491M parameters), NanoGPT-Medium (613M parameters), and NanoGPT-Wide (1.63B parameters)—across three datasets: FineWeb \citep{penedo2024fineweb}, OpenWebText \citep{Gokaslan2019OpenWeb}, and C4 \citep{raffel2020exploring}.
For all experiments, the respective optimizer is applied to matrix parameters in non-embedding layers (attention and MLP weight matrices), while AdamW is used for embedding layers and vector parameters. 
%We use the same training budget and model configuration across optimizers.
%For each model configuration, we evaluate all methods over a common grid of learning rates and report the resulting validation loss, enabling a direct comparison of optimization performance and training stability across learning rates.
Full architectural and hyperparameter details are provided in Appendix~\ref{app:exp_details}.

\subsection{Stability Across Learning Rates}\label{sec:lr_stability}

We evaluate each method across learning rates ranging from $0.01$ to $0.8$ for each model configuration on FineWeb. The exact grid for each configuration is provided in Appendix \ref{app:hyperparameters}.
Results are shown in Figure~\ref{fig:lr_sweep_fineweb}; runs that diverge during training are omitted.
For NanoGPT-Small, all methods achieve comparable validation loss at smaller learning rates. As the learning rate increases, Soft Musec maintains stable convergence, while Muon and MuonClip begin to degrade. 
At the largest learning rate, Muon fails to converge across all model sizes. MuonClip also fails to converge on NanoGPT-Small and NanoGPT-Medium, and converges to substantially worse validation loss on NanoGPT-Wide. In contrast, Soft Musec and SPECTRA remain stable across all configurations.
%The stability gap becomes more pronounced at larger model scales: for NanoGPT-Medium and NanoGPT-Wide, Muon and MuonClip converge to substantially worse validation loss at learning rates greater than $0.1$, while Soft Musec and SPECTRA continue to converge reliably.
Across all configurations, Soft Musec achieves performance comparable to SPECTRA, with slight improvements in several settings. 
Consistent results on OpenWebText and C4 are reported in Appendix \ref{app:lr_sweep}.

\begin{figure}[htbp]
    \centering
    \begin{tabular}{ccc}
        % First Row
        \includegraphics[width=0.31\textwidth]{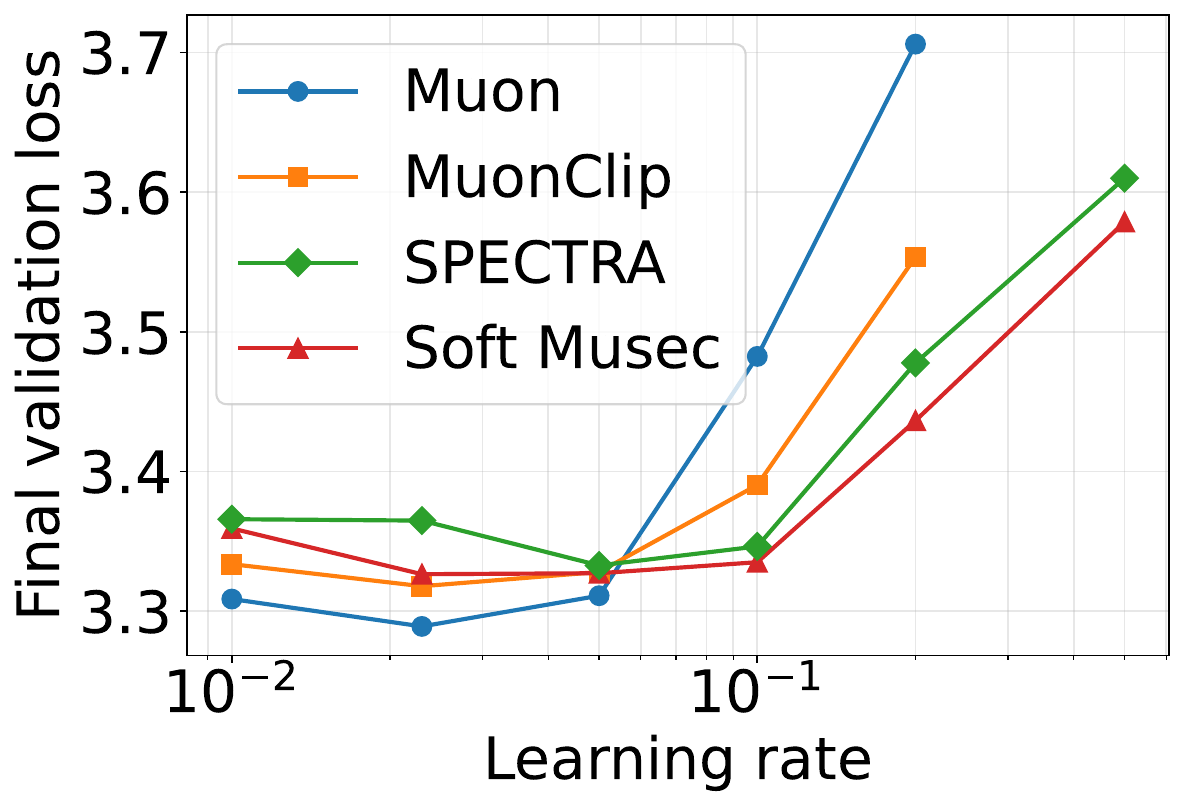} &
        \includegraphics[width=0.31\textwidth]{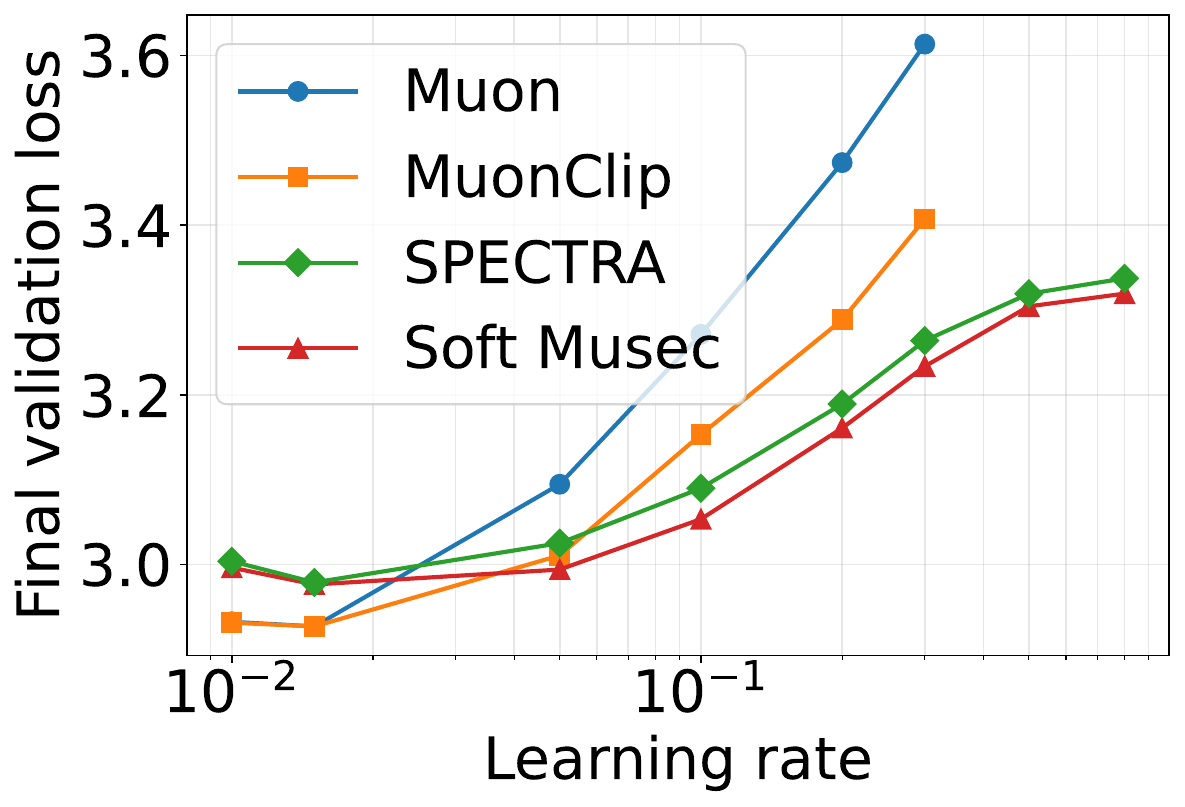} &
        \includegraphics[width=0.31\textwidth]{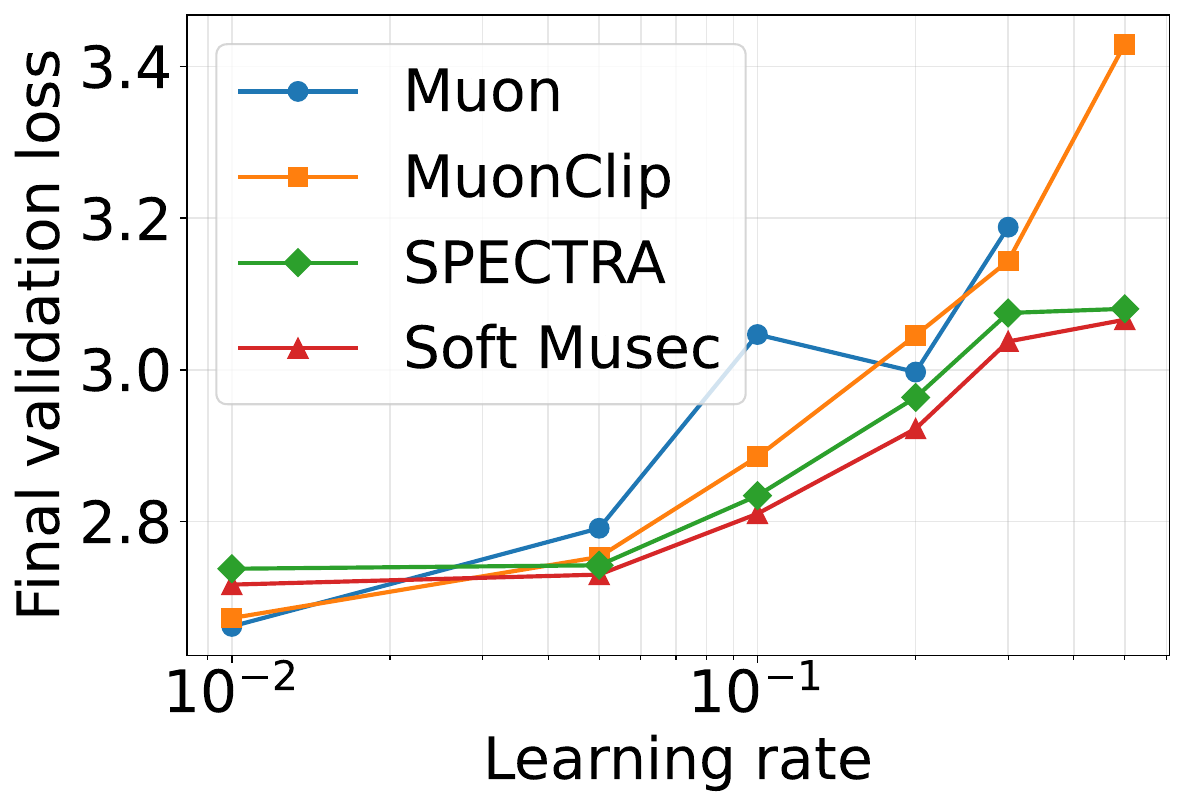} \\
        (a) NanoGPT-small & (b) NanoGPT-medium & (c) NanoGPT-wide \\[1em]
    \end{tabular}
    \caption{Validation loss versus learning rate on FineWeb across three model configurations.}
    \label{fig:lr_sweep_fineweb}
\end{figure}

\subsection{Training Dynamics}

To examine training behavior in greater detail, we train NanoGPT-Medium at a learning rate of $0.2$ on all three datasets and plot the validation loss over training steps in Figure~\ref{fig:training_dynamics_main}.

All methods exhibit a transient loss spike near step 400, which coincides with a scheduled transition in the modded-nanogpt training pipeline: the learning rate increases by $52\%$, the batch size doubles, and the attention window widens simultaneously. 
Notably, Soft Musec and SPECTRA recover from this spike quickly and smoothly, while Muon recovers slowly and converges to a substantially higher final validation loss. 
MuonClip improves over Muon but displays noticeable oscillations throughout training, suggesting that QK-weight clipping alone does not fully stabilize the optimization trajectory. In contrast, both Soft Musec and SPECTRA converge smoothly and reach comparable final validation loss.
Complete training dynamics across all datasets and model configurations are reported in Appendix~\ref{app:training_dynamics}.
% These results are consistent across all three datasets, illustrating that spectral clipping not only widens the range of stable learning rates (as shown in Section~\ref{sec:lr_stability}), but also improves resilience to transient training perturbations within a single run.
\begin{figure}[htbp]
    \centering
    \begin{tabular}{ccc}
        % First Row
        \includegraphics[width=0.31\textwidth]{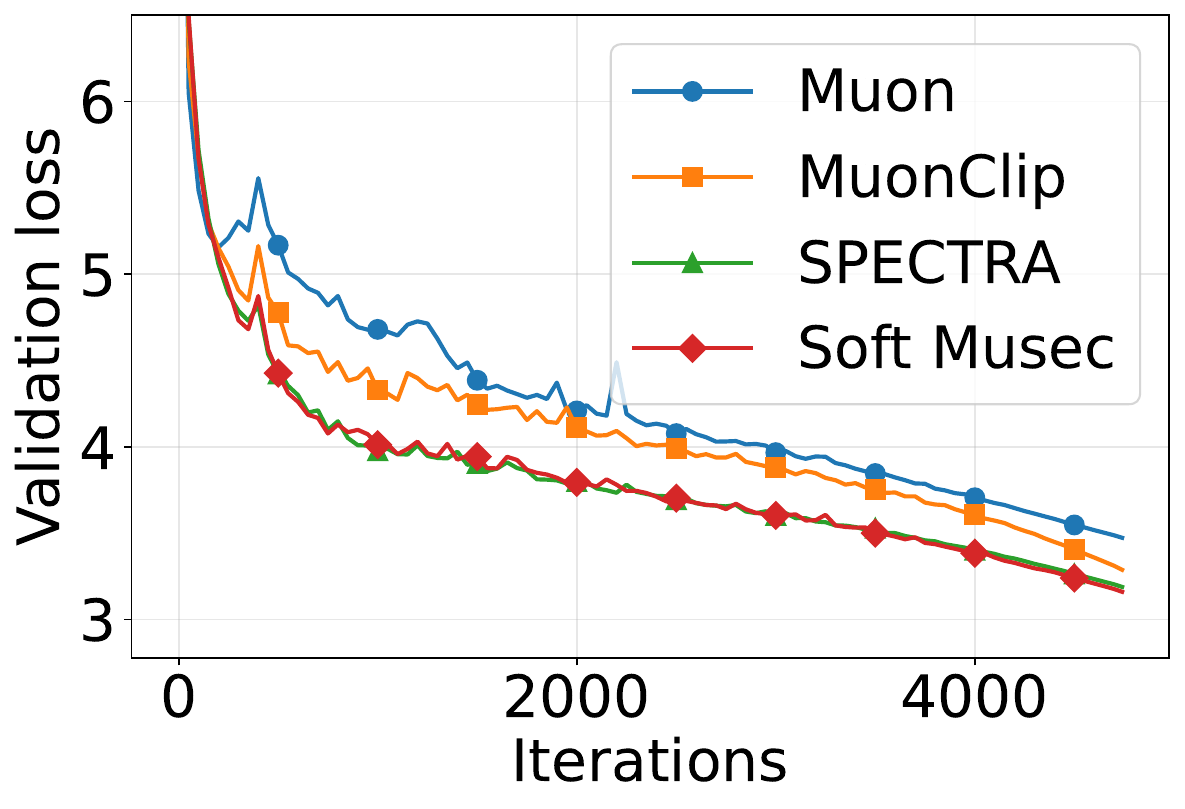} &
        \includegraphics[width=0.31\textwidth]{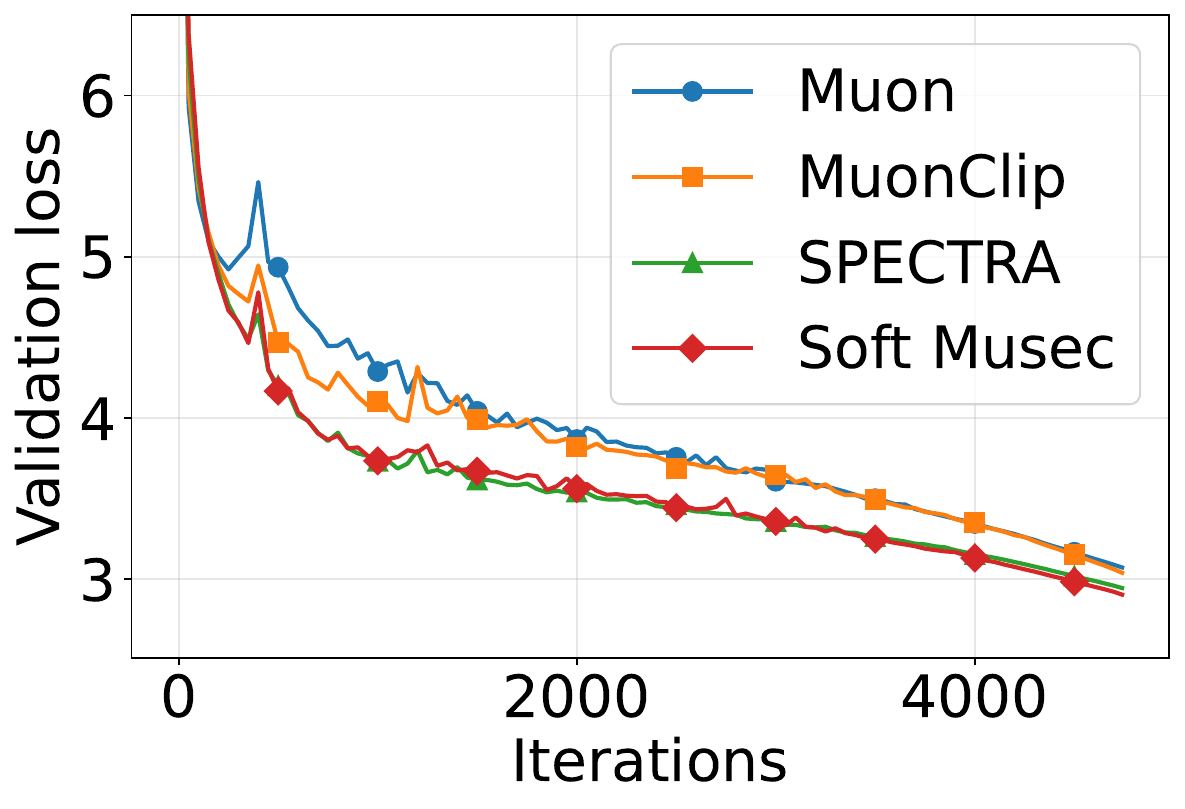} &
        \includegraphics[width=0.31\textwidth]{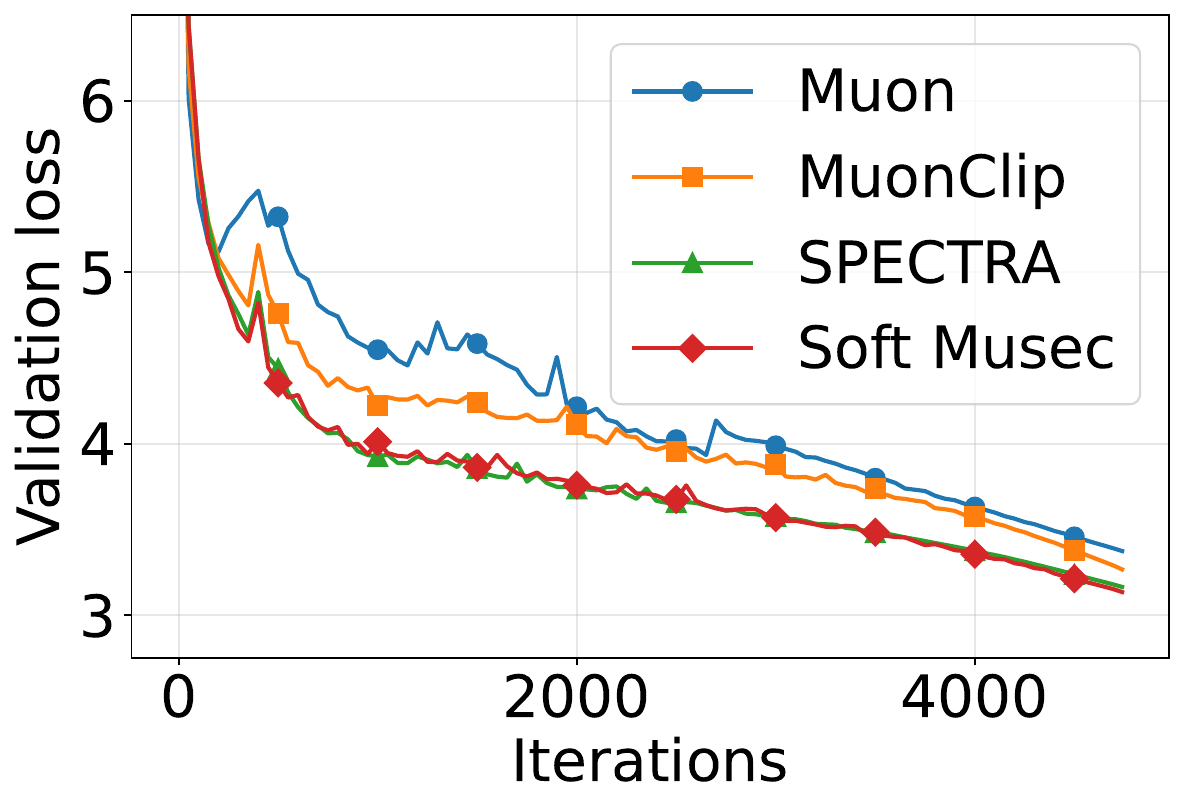} \\
        (a) FineWeb & (b) OpenWebText & (c) C4 \\[1em]
    \end{tabular}
    \caption{Validation loss versus steps of training NanoGPT-Medium on the three datasets.}
    \label{fig:training_dynamics_main}
\end{figure}

\subsection{Weight Norm Analysis}

To understand the source of the stability improvements, we track the spectral norm of weight matrices across three component types: query-key (QK) projections, value-output (VO) projections, and MLP weights, during NanoGPT-Medium training on FineWeb at learning rate~$\eta=0.2$. Results are shown in Figure~\ref{fig:weight_norm}.
Muon exhibits severe spectral-norm inflation across all three components, with norms reaching 300--400 and displaying large oscillations. This is consistent with the spectral flattening mechanism discussed in Section~\ref{sec:prelim}: by assigning comparable update magnitudes to all singular directions, Muon amplifies weakly represented directions, leading to substantial weight growth. The unchecked growth of QK norms is particularly concerning, as it can drive attention logit explosion \citep{team2025kimi}. MuonClip partially mitigates QK norm growth but leaves VO and MLP norms elevated, confirming that it only regulates the query-key matrices. In contrast, both Soft Musec and SPECTRA maintain spectral norms below 10 across all components, demonstrating that optimizer-level spectral clipping provides uniform stabilization across the entire model.

\begin{figure}[htbp]
    \centering
    \begin{tabular}{ccc}
        % First Row
        \includegraphics[width=0.31\textwidth]{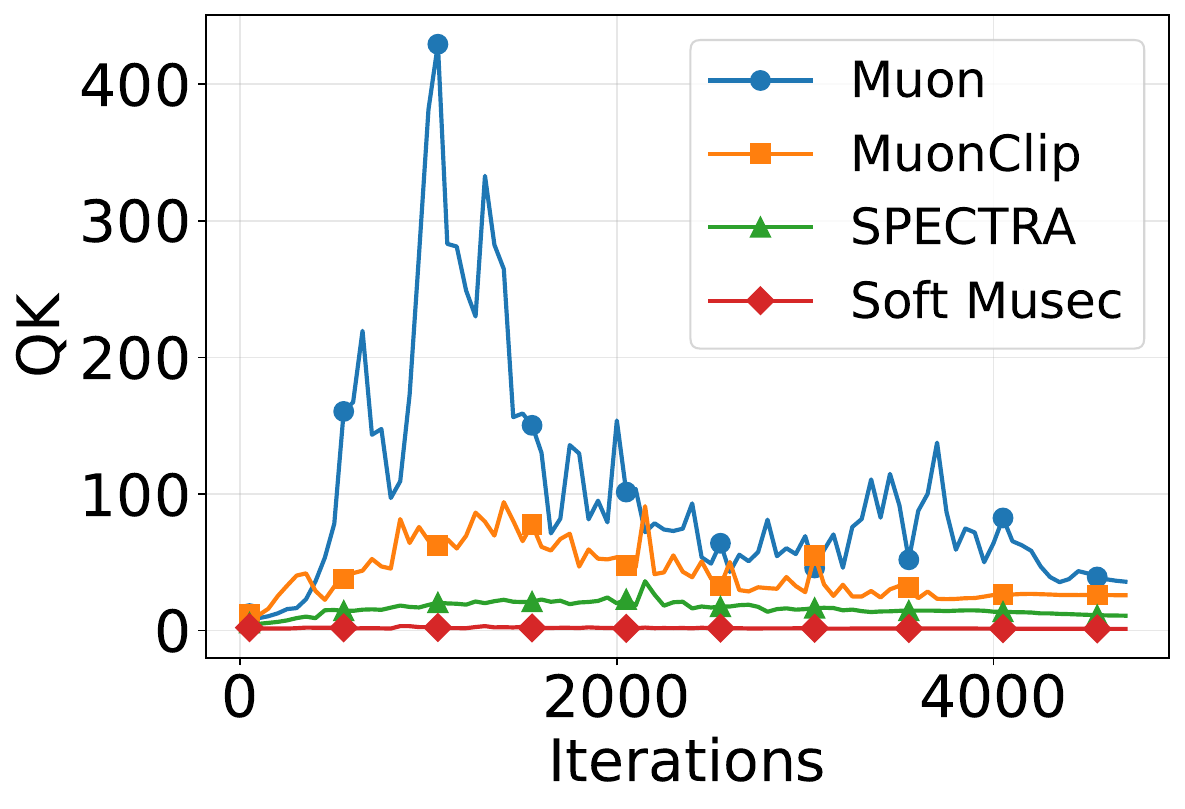} &
        \includegraphics[width=0.31\textwidth]{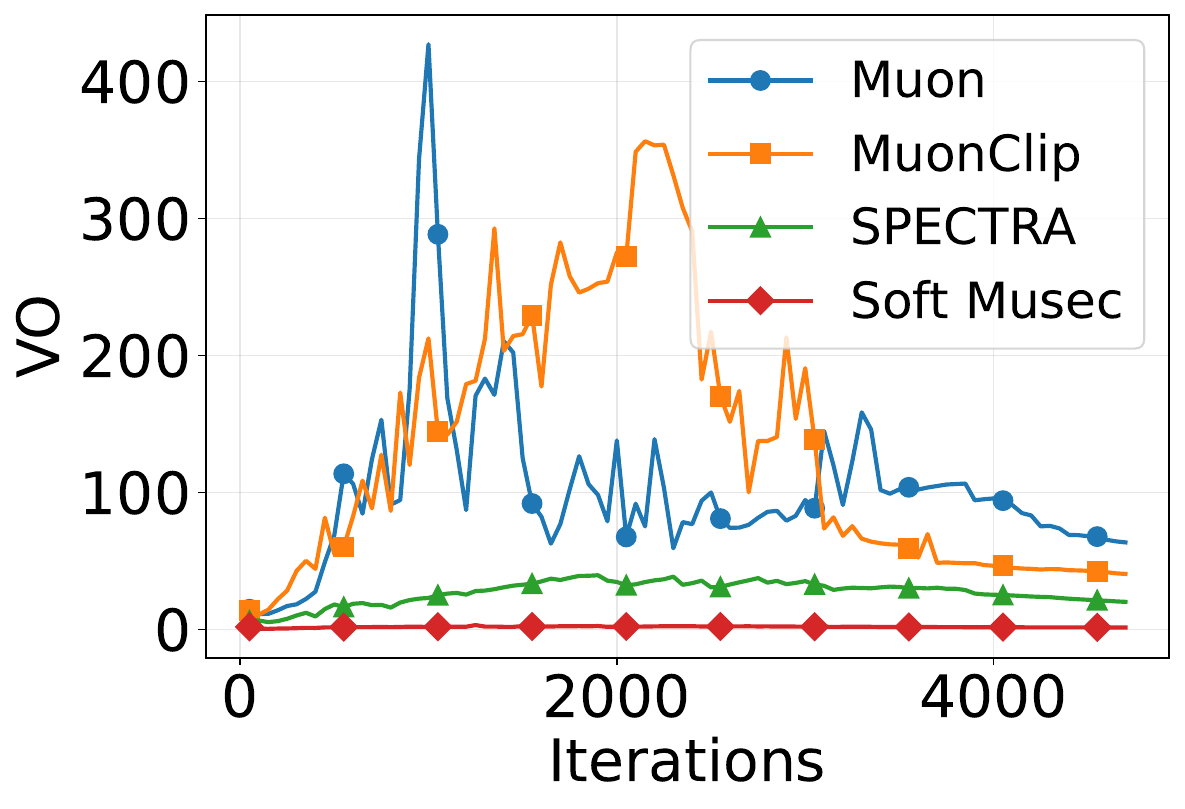} &
        \includegraphics[width=0.31\textwidth]{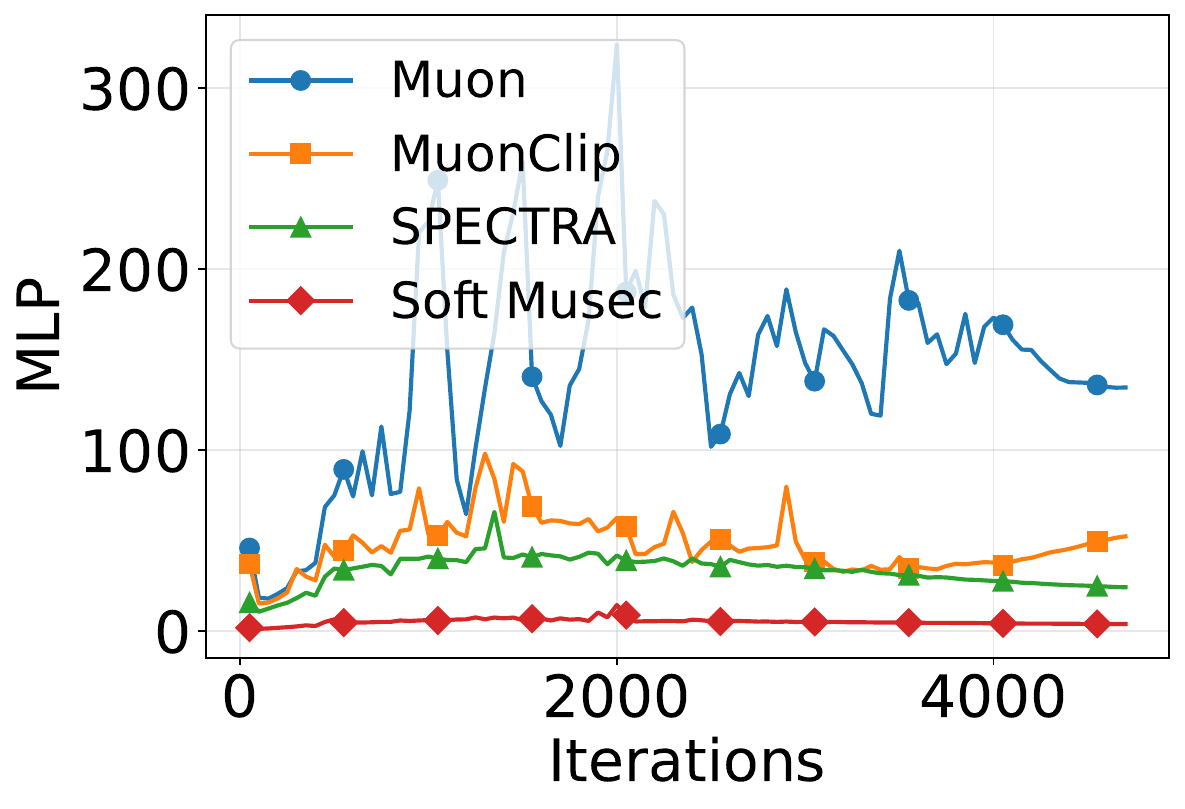} \\
        (a) QK projections & (b) VO projections & (c) MLP weights \\[1em]
    \end{tabular}
    \caption{Spectral norm of weight matrices over training steps for NanoGPT-Medium on FineWeb.}
    \label{fig:weight_norm}
\end{figure}

% Muon exhibits severe spectral-norm inflation across all three component types, with norms reaching $300$--$400$ and exhibiting large oscillations throughout training. This behavior is consistent with the spectral flattening mechanism discussed in Section~\ref{sec:prelim}: by assigning comparable update magnitudes to all singular directions, Muon can amplify directions with initially small spectral components, contributing to substantial weight growth. The unchecked growth of QK spectral norms is particularly concerning, as it can lead to excessively large attention logits~\citep{team2025kimi}.

% MuonClip partially mitigates QK spectral-norm growth by explicitly clipping the query-key matrices, but its VO and MLP norms remain elevated and volatile. This highlights the limitation discussed in Section~\ref{sec:intro}: MuonClip specifically targets the query-key matrices and does not directly regulate the remaining weight matrices, including VO and MLP weights.

% In contrast, both Soft Musec and SPECTRA maintain spectral norms below 10 across all three component types throughout training, compared with values exceeding 300 for Muon. 
% This demonstrates that optimizer-level spectral clipping provides uniform stabilization across the entire model.
% These results directly validate our central argument: by preserving the spectral structure of the momentum and suppressing excessively large singular components, spectral clipping mitigates the weight-norm inflation associated with Muon's spectral flattening.

\section{Conclusion}\label{sec:conclusion}

In this paper, we introduced Musec, which replaces Muon's spectral flattening with a spectral clipping operator that preserves the spectral structure of the momentum matrix while constraining excessively large singular values. 
We established the first convergence analysis of Muon-type algorithms in the nonconvex nonsmooth setting, with a rate matching the optimal dependence on the accuracy parameters for stochastic nonconvex nonsmooth optimization. 
On the practical side, we developed Soft Musec, an efficient SVD-free implementation based on coupled Newton--Schulz iterations, and demonstrated that it substantially improves training stability over existing Muon variants across multiple datasets and model sizes. %In particular, Soft Musec maintains stable convergence in settings where existing Muon variants diverge, while achieving comparable or better performance under well-tuned configurations.

%Several directions remain for future work. On the theoretical side, it would be valuable to investigate whether the weak convexity assumption can be relaxed to accommodate broader classes of nonconvex nonsmooth objectives. On the practical side, an interesting direction is to extend Musec beyond matrix-valued parameters by developing analogous spectral-clipping mechanisms for vector and tensor parameters, potentially leading to a unified optimizer across different parameter types.

\bibliography{iclr2027_conference}
\bibliographystyle{iclr2027_conference}

\newpage
\appendix
\section*{Appendix}
The appendix is organized as follows. 
Section~\ref{app:constrained} proves that Musec's update solves a Frobenius-norm-penalized quadratic subproblem subject to a spectral-norm constraint. 
Section~\ref{app:musec_convergence} provides the detailed convergence analysis of Musec for nonconvex nonsmooth optimization. 
Section~\ref{app:practical_soft} presents the complete Soft Musec algorithm with coupled Newton--Schulz iterations. 
Section~\ref{app:spectra_comp} compares the update rules of Musec and SPECTRA.
Section~\ref{app:exp_details} details the architecture configurations, training configurations, and hyperparameter choices for the NanoGPT experiments. 
Section~\ref{app:additional_exp} provides additional learning rate sweep results, training dynamics, effective rank analysis, computational overhead, and ablation studies.

\section{Proof of Proposition \ref{prop:constrained}}\label{app:constrained}
In this section, we give a formal proof of Proposition \ref{prop:constrained}. 
We first restate the proposition as follows.
\begin{proposition}
    Let $\mG_1, \dots, \mG_L$ be gradient matrices, let  $\lambda > 0$ be a sharpness parameter, and let~$D > 0$. For each $l = 1, \dots, L$, consider the problem
     \begin{align}
        \argmin_{\bm{\Delta}_l} \left[ \langle \mG_l, \bm{\Delta}_l\rangle + \frac{\lambda}{2} \norm{\bm{\Delta}_l}_F^2\right], ~~~{\rm s.t.}~~~\norm{\bm{\Delta}_l}_2 \leq \frac{D}{\lambda}.
    \end{align}
    where $\langle \cdot, \cdot \rangle$ denotes the Frobenius inner product and $\bm{\Delta}_l$ has the same shape as $\mG_l$.
    Let $\mG_l \in \R^{m_l \times n_l}$ have reduced SVD of the form $\mG_l = \mU_l \mS_l \mV_l^{\top}$ with $\mS_l = \mathbf{Diag}(\sigma_1, \dots, \sigma_{d_l})$ and $d_l = \min(m_l, n_l)$. Then Problem (\ref{constrain_prob}) is solved by
    \begin{align*}
        \bm{\Delta}_l = - \eta \cdot \mU_l \hat{\mS}_l \mV_l^{\top}, ~~~{\rm where}~~~\eta = \frac{1}{\lambda} ,~~~\hat{\mS}_l = \mathbf{Clip}(\mS_l, D).
    \end{align*}
\end{proposition} 

\begin{proof}
The Frobenius norm $\norm{\bm{\Delta}_l}_{F}$ and the spectral norm $\norm{\bm{\Delta}_l}_2$  depend on $\bm{\Delta}_l$ only through its singular values, while by Von Neumann's trace inequality the inner product $\langle \mG_l, \bm{\Delta}_l\rangle$ is minimized, for any fixed spectrum of $\bm{\Delta}_l$, by aligning $\bm{\Delta}_l$'s singular vectors with those of $-\mG_l$. It therefore suffices to restrict attention to
\begin{align*}
\bm{\Delta}_l = -\mU_l \mA_l \mV_l^{\top}, \quad \mA_l = \mathrm{diag}(a_1, \dots, a_{d_l}), \quad a_1 \geq a_2 \geq \dots \geq a_{d_l}\geq 0.
\end{align*}
Under this parametrization,
\begin{align*}
\langle \mG_l, \bm{\Delta}_l\rangle = -\sum_{i=1}^{d_l} \sigma_i a_i, \qquad \norm{\bm{\Delta}_l}_F^2 = \sum_{i=1}^{d_l} a_i^2, \qquad \norm{\bm{\Delta}_l}_2 = a_1.
\end{align*}
Since $a_1 \geq \dots \geq a_{d_l} \geq 0$, the constraint $a_1 \leq D / \lambda$ is equivalent to $0 \leq a_l \leq D / \lambda$ for every $i$, and Problem~(\ref{constrain_prob}) reduces to the separable box-constrained program
\begin{align}\label{equiv_prob}
\min_{ a_i \in [0, D/\lambda]} \sum_{i=1}^{d_l} \left[-\sigma_i a_i + \frac{\lambda}{2} a_i^2\right].
\end{align}
Each scalar subproblem is a quadratic in $a_i$ with unconstrained minimizer $\sigma_i / \lambda \geq 0$; clipping to the interval $[0, D/\lambda]$ gives
\begin{align*}
a_i^* = \min\left(\frac{\sigma_i}{\lambda}, \frac{D}{\lambda}\right) = \frac{1}{\lambda}\min(\sigma_i, D).
\end{align*}
The ordering $\sigma_1 \geq \dots \geq \sigma_{d_l}$ ensures $a_1^* \geq \dots \geq a_{d_l}^*$, so the monotonicity of the parametrization is automatically satisfied. Since $a_i^* = \min(\sigma_i, D)/\lambda$, we have $\mA_l^* = \frac{1}{\lambda} \hat{\mS_l}$ with $\hat{\mS}_l = \mathbf{Clip}(\mS_l, D)$, and hence 
\begin{align*}
\bm{\Delta}_l = -\mU_l \mA_l^* \mV_l^{\top} = -\eta \cdot \mU_l \hat{\mS}_l \mV_l^{\top},
\end{align*}
with $\eta = 1/ \lambda$, as claimed.
\end{proof}

\section{Convergence Analysis}\label{app:musec_convergence}
In this section, we provide the detailed convergence analysis of the Musec algorithm.
\subsection{Regret Analysis of Online Gradient Descent}
Before presenting the convergence analysis of the Musec algorithm, we first establish a convergence result for online gradient descent on a quadratic objective, which will subsequently serve as the basis for analyzing the clipping operation in Musec.

Consider a quadratic function of the form: 
\begin{align*}
    f_t(\mX)\coloneqq -\frac{\beta}{\gamma}\langle \mG_{t+1}, \mX\rangle + \frac{\beta}{2 \gamma} \norm{\mX}_F^2
\end{align*}
 in the matrix-valued variable $\mX$. 
We analyze the following standard online gradient descent (OGD) method with stepsize $\gamma > 0$ over a convex compact set $\gC$:
\begin{align}\label{eq:ogd_step}
\begin{split}
\mX_{t+1} 
\coloneqq & \Pi_\gC (\mX_t - \gamma\nabla f_t(\mX_t)) \\
= & \Pi_\gC (\mX_t + \beta \mG_{t+1} - \beta \mX_t)  \\
= & \Pi_\gC \left((1 - \beta)\mX_t + \beta \mG_{t+1}\right),
\end{split}
\end{align}
where $\Pi_\gC(\cdot)$ denotes the Euclidean Projection associated with the convex and compact set $\gC$.
The OGD step (\ref{eq:ogd_step}) is consistent with the update of the momentum variable $\mM_n$ in Algorithm \ref{alg:musec} when we choose $\gC= \{\mX: \norm{\mX} \leq D \}$.

Let ${\rm Regret}_T(\overline{\mX})$ denote the regret of the OGD algorithm with respect to some $\overline{\mX} \in \gC$ after $T$ iterations, defined as:
\begin{align*}
    {\rm Regret}_T(\overline{\mX}) \coloneqq \sum_{t=0}^{T-1} f_t(\mX_t) - \sum_{t=0}^{T-1} f_t(\overline{\mX}).
\end{align*}
We present the regret analysis of the OGD algorithm as follows.

\begin{lemma}\label{lemma:regret_bound}
    Suppose that $\beta \leq 1/8$, then the OGD update (\ref{eq:ogd_step}) applied to the sequence of quadratic functions $\{f_t (\mX)\}_{t=0}^{T-1}$ over the convex constraint set $\gC$ guarantees that for all $T \geq 1$ and all $\overline{\mX} \in \gC$:
    \begin{align*}
        & {\rm Regret}_T(\overline{\mX}) \\
        \leq & \sum_{t=0}^{T-1} \left(-\frac{\beta}{8\gamma}\norm{\mX_t}_F^2 + \frac{\beta}{2 \gamma}\norm{\overline{\mX}}_F^2 + \frac{\beta^2}{\gamma} \norm{\mG_{t+1}}_F^2\right) + \frac{\norm{\mX_0 }_F^2 + \norm{\overline{\mX}}_F^2}{\gamma}.
    \end{align*}
\end{lemma}
\begin{proof}
    Fix any $\overline{\mX}\in \gC$. By the non-expansiveness of the projection operator,
    \begin{align*}
       & \norm{\mX_{t+1} - \overline{\mX}}_F^2 \\
       = & \norm{\Pi_\gC(\mX_t - \gamma \nabla f_t(\mX_t)) - \overline{\mX}}_F^2 \\
        \leq & \norm{\mX_t - \gamma \nabla f_t(\mX_t) - \overline{\mX}}_F^2 \\
        = & \norm{\mX_t - \overline{\mX}}_F^2 + \gamma^2 \norm{\nabla f_t(\mX_t)}_F^2 - 2\gamma \langle \nabla f_t(\mX_t), \mX_t - \overline{\mX}\rangle.
    \end{align*}
    Rearranging the terms yields
    \begin{align}\label{eq:inner_prod_bound}
        \langle \nabla f_t(\mX_t), \mX_t - \overline{\mX}\rangle \leq \frac{\norm{\mX_t - \overline{\mX}}_F^2 - \norm{\mX_{t+1} - \overline{\mX}}_F^2}{2 \gamma} + \frac{\gamma \norm{\nabla f_t(\mX_t)}_F^2}{2}.
    \end{align}
    Accordingly, we can show that
    \begin{align*}
        & {\rm Regret}_T(\overline{\mX})\\
        = &\sum_{t=0}^{T-1} f_t(\mX_t) - \sum_{t=0}^{T-1} f_t(\overline{\mX}) \\
        \leq & \sum_{t=0}^{T-1} \left(\langle \nabla f_t(\mX_t), \mX_t - \overline{\mX} \rangle - \frac{\beta}{2 \gamma} \norm{\mX_t - \overline{\mX}}_F^2\right) \\
        \leq & - \frac{\beta}{2 \gamma} \sum_{t=0}^{T-1} \norm{\mX_t -\overline{\mX}}_F^2 + \frac{\norm{\mX_0 - \overline{\mX}}_F^2}{2\gamma} - \frac{\norm{\mX_{T} - \overline{\mX}}_F^2}{2 \gamma} + \sum_{t=0}^{T-1}\frac{\gamma \norm{\nabla f_t(\mX_t)}_F^2}{2}\\
        \leq & - \frac{\beta}{2 \gamma} \sum_{t=0}^{T-1} \norm{\mX_t -\overline{\mX}}_F^2 + \frac{\norm{\mX_0 }_F^2 + \norm{\overline{\mX}}_F^2}{\gamma} + \sum_{t=0}^{T-1}\frac{\gamma \norm{- \frac{\beta}{\gamma} \mG_{t+1} + \frac{\beta}{\gamma} \mX_t}_F^2}{2}\\
         \leq & - \frac{\beta}{2 \gamma} \sum_{t=0}^{T-1} \left(\frac{\norm{\mX_t}_F^2}{2} - \norm{\overline{\mX}}_F^2\right) + \frac{\norm{\mX_0 }_F^2 + \norm{\overline{\mX}}_F^2}{\gamma} + \sum_{t=0}^{T-1}\left(\frac{\beta^2 \norm{\mG_{t+1}}_F^2}{\gamma} + \frac{\beta^2}{\gamma}\norm{\mX_t}_F^2\right)\\
         =&  \sum_{t=0}^{T-1} \left(\frac{\beta}{\gamma}(\beta - \frac{1}{4})\norm{\mX_t}_F^2 + \frac{\beta}{2 \gamma}\norm{\overline{\mX}}_F^2 + \frac{\beta^2}{\gamma} \norm{\mG_{t+1}}_F^2\right) + \frac{\norm{\mX_0 }_F^2 + \norm{\overline{\mX}}_F^2}{\gamma}.
    \end{align*}
    The first inequality is due to the $\beta / \gamma$-strong convexity of $f_t(\cdot)$; the second inequality follows from~(\ref{eq:inner_prod_bound}); the third inequality follows from Young's inequality. The last inequality is due to the reverse Young's inequality: $\norm{\mA - \mB}_F^2 \geq \norm{\mA}_F^2 / 2 - \norm{\mB}_F^2$.
    Since $\beta \leq 1/8$, then we have
    \begin{align*}
        {\rm Regret}_T(\overline{\mX}) \leq \sum_{t=0}^{T-1} \left(-\frac{\beta}{8\gamma}\norm{\mX_t}_F^2 + \frac{\beta}{2 \gamma}\norm{\overline{\mX}}_F^2 + \frac{\beta^2}{\gamma} \norm{\mG_{t+1}}_F^2\right) + \frac{\norm{\mX_0 }_F^2 + \norm{\overline{\mX}}_F^2}{\gamma}.
    \end{align*}
\end{proof}

\subsection{Proof of Lemma \ref{lemma:core_recur}}
We first restate Lemma \ref{lemma:core_recur} as follows.
\begin{lemma}
      Let $\gamma = \beta / \eta$ and $\eta \leq 1 / \rho$ where $\beta \leq 1/8$, then for any $k\in [K]$, the sequence $\{\mW_t^{(k)}\}_{t=1}^T$ generated by Algorithm \ref{alg:musec} satisfies
  \begin{align*}
        &  \E\left[f(\mW_T^{(k)}) - f(\mW_0^{(k)}) \right] \\
        \leq  & -\E\left[\frac{\beta D T}{\gamma}  \norm{\overline{\mH}^{(k)}}_F + \sum_{t=0}^{T-1} \frac{\beta}{8\gamma}\norm{\mM_t^{(k)}}_F^2\right] + \frac{\beta D \sigma\sqrt{T} }{\gamma} + \left(\frac{\beta T}{\gamma} + \frac{1}{\gamma} \right)D^2 + \frac{\beta^2 L^2 T}{\gamma} + \frac{r D^2}{\gamma}.
    \end{align*}
\end{lemma}

\begin{proof}
    Let us consider the filtration $\gF_n = \sigma(\mG_1, \mG_2, \dots, \mG_{n})$, where $\sigma(\cdot)$ denotes the generated $\sigma$-field.
    From Assumption \ref{asm:bound_var}, we can infer that 
    \begin{align*}
        \mH_n \coloneqq \E[\mG_n \mid \gF_{n-1}] \in \partial f(\mW_n).
    \end{align*}
    The weak convexity of $f(\cdot)$ implies that for all $n \geq 1$ and $\eta^{-1} \geq \rho$,
    \begin{align*}
        f(\mW_{n+1}) - f(\mW_{n}) \leq & \langle \mH_{n+1}, \mW_{n+1} - \mW_{n} \rangle + \frac{1}{2 \eta} \norm{\mW_{n+1} - \mW_{n}}_F^2 \\
        = & \E\left[\langle \mG_{n+1}, \mW_{n+1} - \mW_{n} \rangle + \frac{1}{2 \eta} \norm{\mW_{n+1} - \mW_{n}}_F^2 \mid \gF_{n}\right].
    \end{align*}
    The law of total expectation implies that
    \begin{align*}
        \E [f(\mW_{n+1}) - f(\mW_{n})] \leq \E\left[\langle \mG_{n+1}, \mW_{n+1} - \mW_{n} \rangle + \frac{1}{2 \eta} \norm{\mW_{n+1} - \mW_{n}}_F^2 \right].
    \end{align*}
    Fix an arbitrary $k \in [K]$. Recall that $\mW_t^{(k)} = \mW_{(k-1)T + t}$ where $t\in\{0\}\cup[T-1]$ and we similarly define the notions of $\mG_t^{(k)}$, $\mH_t^{(k)}$ and $\mM_t^{(k)}$.
    We can express the previous inequality as 
    \begin{align*}
         \E [f(\mW_{t+1}^{(k)}) - f(\mW_{t}^{(k)})] \leq \E\left[\langle \mG_{t+1}^{(k)}, \mW_{t+1}^{(k)} - \mW_{t}^{(k)} \rangle + \frac{1}{2 \eta} \norm{\mW_{t+1}^{(k)} - \mW_{t}^{(k)}}_F^2 \right].
    \end{align*}
    Summing the above inequality over $t = 0, \dots, T-1$ yields that
    \begin{align}\label{sum_eq}
        \E[f(\mW_T^{(k)}) - f(\mW_0^{(k)})] \leq \E\left[\sum_{t=0}^{T-1} \left(\langle \mG_{t+1}^{(k)}, \mW_{t+1}^{(k)} - \mW_{t}^{(k)} \rangle + \frac{1}{2 \eta} \norm{\mW_{t+1}^{(k)} - \mW_{t}^{(k)}}_F^2\right) \right].
    \end{align}
    From the update rule of $\mW_{t+1}^{(k)}$ (step \ref{update_step} in Algorithm \ref{alg:musec}), we have
    \begin{align*}
        \mM_{t}^{(k)} = \frac{\mW_{t}^{(k)} - \mW_{t+1}^{(k)}}{\eta},
    \end{align*}
    then we can rewrite the inequality (\ref{sum_eq}) as:
     \begin{align}\label{regret_eq1}
        \E[f(\mW_T^{(k)}) - f(\mW_0^{(k)})] \leq \E\left[\sum_{t=0}^{T-1} \left(- \eta \langle \mG_{t+1}^{(k)}, \mM_{t}^{(k)}\rangle + \frac{ \eta}{2 } \norm{\mM_{t}^{(k)} }_F^2\right) \right].
    \end{align}
    If we choose $\beta = \eta \gamma$, and we denote the quadratic function
    \begin{align}\label{eq:quadra_func}
        f_t^{(k)}(\mX) = - \frac{\beta}{\gamma} \langle \mG_{t+1}^{(k)}, \mX\rangle + \frac{\beta}{2 \gamma} \norm{\mX}_F^2.
    \end{align}
    
    Since $\beta \leq 1 / 8$, we can apply Lemma \ref{lemma:regret_bound} for the regret analysis with respect to the quadratic function $f_t^{(k)}(\mX)$ by taking $\mX_t = \mM_t^{(k)}$ and $\overline{\mX} = \overline{\mM}^{(k)}$, $\gC = \{\mX:\norm{\mX}_2 \leq D\}$:
    \begin{align}\label{eq:regret_analysis}
    \begin{split}
        & {\rm Regret}_T(\overline{\mM}) \\
        \coloneqq & \sum_{t=0}^{T-1} f_t^{(k)}(\mM_t^{(k)}) - \sum_{t=0}^{T-1} f_t^{(k)}(\overline{\mM}^{(k)}) \\
        \leq & \sum_{t=0}^{T-1} \left(-\frac{\beta}{8\gamma}\norm{\mM_t^{(k)}}_F^2 + \frac{\beta}{2 \gamma}\norm{\overline{\mM}^{(k)}}_F^2 + \frac{\beta^2}{\gamma} \norm{\mG_{t+1}^{(k)}}_F^2\right) + \frac{\norm{\mM_0^{(k)} }_F^2 + \norm{\overline{\mM}^{(k)}}_F^2}{\gamma}.
    \end{split}
    \end{align}
    Substitute the definition of the quadratic function (\ref{eq:quadra_func}) into (\ref{eq:regret_analysis}):
    {\small
    \begin{align*}
         & \sum_{t=0}^{T-1} \left(- \frac{\beta}{\gamma} \langle \mG_{t+1}^{(k)}, \mM_t^{(k)}\rangle + \frac{\beta}{2 \gamma} \norm{\mM_t^{(k)}}_F^2 \right) \\
         \leq & \sum_{t=0}^{T-1} \left(- \frac{\beta}{\gamma} \langle \mG_{t+1}^{(k)}, \overline{\mM}^{(k)}\rangle + \frac{\beta}{2 \gamma} \norm{\overline{\mM}^{(k)}}_F^2 \right) + \sum_{t=0}^{T-1} \left(-\frac{\beta}{8\gamma}\norm{\mM_t^{(k)}}_F^2 + \frac{\beta}{2 \gamma}\norm{\overline{\mM}^{(k)}}_F^2 + \frac{\beta^2}{\gamma} \norm{\mG_{t+1}^{(k)}}_F^2\right) \\
         & \qquad + \frac{\norm{\mM_0^{(k)} }_F^2 + \norm{\overline{\mM}^{(k)}}_F^2}{\gamma}.
    \end{align*}}
    Taking expectations on both sides yields that:
    \begin{align}\label{intermediate_step_2}
    \begin{split}
      &  \E\left[\sum_{t=0}^{T-1} \left(  - \frac{\beta}{\gamma} \langle \mG_{t+1}^{(k)}, \mM_{t}^{(k)}\rangle + \frac{\beta}{2 \gamma} \norm{\mM_{t}^{(k)}}_F^2\right) \right] \\
      \leq &  \E\left[\sum_{t=0}^{T-1} \left(  - \frac{\beta}{\gamma} \langle \mG_{t+1}^{(k)}, \overline{\mM}^{(k)}\rangle + \frac{\beta}{2 \gamma} \norm{\overline{\mM}^{(k)}}_F^2\right) \right. \\
      &\left. + \sum_{t=0}^{T-1} \left(-\frac{\beta}{8\gamma}\norm{\mM_{t}^{(k)}}_F^2 + \frac{\beta}{2 \gamma}\norm{\overline{\mM}^{(k)}}_F^2 + \frac{\beta^2}{\gamma} \norm{\mG_{t+1}^{(k)}}_F^2\right)   + \frac{\norm{\mM_0^{(k)} }_F^2 + \norm{\overline{\mM}^{(k)}}_F^2}{\gamma}\right] \\
      = & \E\left[ - \frac{\beta}{\gamma} \sum_{t=0}^{T-1}  \langle \mG_{t+1}^{(k)} - \mH_{t+1}^{(k)}, \overline{\mM}^{(k)}\rangle - \frac{\beta}{\gamma} \sum_{t=0}^{T-1} \langle \mH_{t+1}^{(k)}, \overline{\mM}^{(k)} \rangle \right.\\
      & \left. + \sum_{t=0}^{T-1} \left(-\frac{\beta}{8\gamma}\norm{\mM_t^{(k)}}_F^2 + \frac{\beta}{ \gamma}\norm{\overline{\mM}^{(k)}}_F^2 + \frac{\beta^2}{\gamma} \norm{\mG_{t+1}^{(k)}}_F^2\right) + \frac{\norm{\mM_0^{(k)} }_F^2 + \norm{\overline{\mM}^{(k)}}_F^2}{\gamma} \right].
      \end{split}
    \end{align}
    We choose
    \begin{align}\label{bar_M_def}
        \overline{\mM}^{(k)} = D \frac{\sum_{t=1}^T \mH_t^{(k)}}{\norm{\sum_{t=1}^T \mH_t^{(k)}}_F},
    \end{align}
    then we can show that
    \begin{align}\label{bar_M_bound}
        \norm{\overline{\mM}^{(k)}}_2 \leq \norm{\overline{\mM}^{(k)}}_F = D.
    \end{align}
    In addition, using Jensen's inequality and the fact that $\{\mG_{t+1}^{(k)} - \mH_{t+1}^{(k)}\}$ is a martingale-difference sequence with variance bounded by $\sigma^2$ (Assumption~\ref{asm:bound_var}):
    \begin{align}\label{intermediate_step_0} 
    \begin{split}
        \E\left[- \frac{\beta}{\gamma} \sum_{t=0}^{T-1}  \langle \mG_{t+1}^{(k)} - \mH_{t+1}^{(k)}, \overline{\mM}^{(k)}\rangle\right] \leq & \E\left[\frac{\beta }{\gamma}\norm{\sum_{t=0}^{T-1}  \left(\mG_{t+1}^{(k)} - \mH_{t+1}^{(k)}\right)}_F \norm{\overline{\mM}^{(k)}}_F\right] \\
        \leq & \frac{\beta D \sigma\sqrt{T} }{\gamma},
    \end{split}
    \end{align}
    where the last inequality follows from Assumption \ref{asm:bound_var}.
    In addition, the definition (\ref{bar_M_def}) implies that
    \begin{align}\label{intermediate_step_1}
    \begin{split}
        \E\left[-\frac{\beta}{\gamma} \sum_{t=0}^{T-1}\langle \mH_{t+1}^{(k)} , \overline{\mM}^{(k)}\rangle\right] = & -  \E\left[ \frac{\beta}{\gamma}\left\langle \sum_{t=1}^T \mH_t^{(k)}, \frac{D \sum_{t=1}^T \mH_t^{(k)}}{\norm{\sum_{t=1}^T \mH_t^{(k)}}_F} \right\rangle \right]\\
        =& - \E\left[\frac{\beta D}{\gamma} \norm{\sum_{t=1}^T \mH_t^{(k)}}_F\right].
        \end{split}
    \end{align}
    Furthermore, Assumption \ref{asm:lip} implies that $\norm{\mG_{t+1}^{(k)}}_F \leq L$. Step (\ref{clip_op}) in Algorithm \ref{alg:musec} yields that $\norm{\mM_0^{(k)}}_F \leq \sqrt{r}D$.
     Consequently, substituting (\ref{bar_M_bound}), (\ref{intermediate_step_0}), (\ref{intermediate_step_1}) into (\ref{intermediate_step_2}), we have
    {\small
    \begin{align}\label{regret_eq2}
    \begin{split}
        &  \E\left[\sum_{t=0}^{T-1} \left(  - \frac{\beta}{\gamma} \langle \mG_{t+1}^{(k)}, \mM_{t}^{(k)}\rangle + \frac{\beta}{2 \gamma} \norm{\mM_{t}^{(k)}}_F^2\right) \right] \\
        \leq &  \frac{\beta D \sigma\sqrt{T} }{\gamma}   -\E\left[\frac{\beta D}{\gamma}  \norm{\sum_{t=1}^T \mH_t^{(k)}}_F + \sum_{t=0}^{T-1} \frac{\beta}{8\gamma}\norm{\mM_t^{(k)}}_F^2\right] + \left(\frac{\beta T D^2}{\gamma} + \frac{\beta^2 L^2 T}{\gamma} \right)  + \frac{r D^2}{\gamma} + \frac{D^2}{\gamma} \\
        = & -\E\left[\frac{\beta D T}{\gamma}  \norm{\overline{\mH}^{(k)}}_F + \sum_{t=0}^{T-1} \frac{\beta}{8\gamma}\norm{\mM_t^{(k)}}_F^2\right] + \frac{\beta D \sigma\sqrt{T} }{\gamma} + \left(\frac{\beta T}{\gamma} + \frac{1}{\gamma} \right)D^2 + \frac{\beta^2 L^2 T}{\gamma} + \frac{r D^2}{\gamma}.
        \end{split}
    \end{align}
    }
    where in the last equality follows from the definition $\overline{\mH}^{(k)} = \frac{1}{T}\sum_{t=1}^{T} \mH_t^{(k)}$.
    Recall that $\beta = \eta \gamma$, then combining~(\ref{regret_eq1}) and (\ref{regret_eq2}) yields that
    \begin{align*}
        & \E[f(\mW_T^{(k)}) - f(\mW_0^{(k)})] \\ \leq  & -\E\left[\frac{\beta D T}{\gamma}  \norm{\overline{\mH}^{(k)}}_F + \sum_{t=0}^{T-1} \frac{\beta}{8\gamma}\norm{\mM_t^{(k)}}_F^2\right] + \frac{\beta D \sigma\sqrt{T} }{\gamma} + \left(\frac{\beta T}{\gamma} + \frac{1}{\gamma} \right)D^2 + \frac{\beta^2 L^2 T}{\gamma} + \frac{r D^2}{\gamma}.
    \end{align*}
\end{proof}

Before proceeding, we establish an auxiliary lemma that will be used to control the deviation of the iterates $\mW_t^{(k)}$ from their average.

\begin{lemma}\label{lemma:avg_dist}
    Let $\mX_1,\dots, \mX_n$ be a set of matrices, and let $\overline{\mX} = \frac{1}{n}\sum_{i=1}^n \mX_i$.
    Then for all $i \in [n]$,
    \begin{align*}
        \norm{\mX_i - \overline{\mX}}_F^2 \leq \frac{1}{n}\sum_{i'=1}^n \norm{\mX_i - \mX_{i'}}_F^2 \leq n \sum_{j=1}^n \norm{\Delta_j}_F^2,
    \end{align*}
    where $\Delta_j \coloneqq \mX_j-\mX_{j-1}$ and $\mX_0$ can be chosen arbitrarily.
\end{lemma}
\begin{proof}
    Fix any $i \in [n]$, we have
    \begin{align}
    \begin{split}
        \norm{\mX_i - \overline{\mX}}_F^2 =& \norm{\mX_i - \frac{1}{n}\sum_{i'=1}^n \mX_{i'}}_F^2 \leq \frac{1}{n} \sum_{i'=1}^n \norm{\mX_i - \mX_{i'}}_F^2 \\
        = & \frac{1}{n} \sum_{i'=1}^n\norm{\sum_{j=i\wedge i' + 1}^{i \vee i'}(\mX_{j-1} - \mX_j)}_F^2 \\
        \leq & \frac{1}{n} \sum_{i'=1}^n\left(\sum_{j=i\wedge i' + 1}^{i \vee i'}\norm{\Delta_j}_F\right)^2 \\
        \leq & \left(\sum_{j=2}^n \norm{\Delta_j}_F\right)^2 \leq n \sum_{j=1}^n \norm{\Delta_j}_F^2,
        \end{split}
    \end{align}
    where the first and the last inequalities are due to the Cauchy--Schwarz inequality; The second inequality follows from the triangle inequality.
\end{proof}

\subsection{Proof of Theorem \ref{lemma:dist_goldstein}}
We now provide the proof of the main convergence theorem of the Musec. We first restate Theorem~\ref{lemma:dist_goldstein} as follows:
\begin{theorem}
    Let $\gamma = \beta / \eta$ and $\eta \leq 1 / \rho$ where $\beta \leq 1/8$,
    Then for any $\delta \geq \eta T D\sqrt{r}$, the sequence $\{\overline{\mW}^{(k)} \}_{k=1}^K$ generated by Algorithm~\ref{alg:musec} satisfies
    \begin{align*}
       \E\left[\frac{1}{K}\sum_{k=1}^K  {\rm dist}(0, \partial_\delta f(\overline{\mW}^{(k)})) \right] \leq \frac{\sigma}{\sqrt{T}} + \left(1 + \frac{2r}{\beta T} \right)D + \frac{\beta L^2 }{D} + \frac{\gamma\Delta_{f} }{\beta D T K}.
   \end{align*}
\end{theorem}

\begin{proof}
   For any $k \in [K]$, according to Lemma \ref{lemma:core_recur}, one has
   \begin{align*}
      & \E[f(\mW_T^{(k)}) - f(\mW_0^{(k)})]\\
     \leq & -\E\left[\frac{\beta D T}{\gamma}  \norm{\overline{\mH}^{(k)}}_F + \sum_{t=0}^{T-1} \frac{\beta}{8\gamma}\norm{\mM_t^{(k)}}_F^2\right] + \frac{\beta D \sigma\sqrt{T} }{\gamma}\\
       & + \left(\frac{\beta T}{\gamma} + \frac{1}{\gamma} \right)D^2 + \frac{\beta^2 L^2 T}{\gamma} + \frac{r D^2}{\gamma} \\
       \leq & -\E\left[\frac{\beta D T}{\gamma}  \norm{\overline{\mH}^{(k)}}_F \right]+ \frac{\beta D \sigma\sqrt{T} }{\gamma} + \left(\frac{\beta T}{\gamma} + \frac{2r}{\gamma} \right)D^2 + \frac{\beta^2 L^2 T}{\gamma} .
   \end{align*}
   By the definition that $\mW_T^{(k)} = \mW_0^{(k+1)}$, we obtain
   \begin{align*}
      & \E[f(\mW_0^{(k+1)}) - f(\mW_0^{(k)})]\\
       \leq & -\E\left[\frac{\beta D T}{\gamma}  \norm{\overline{\mH}^{(k)}}_F \right]+ \frac{\beta D \sigma\sqrt{T} }{\gamma} + \left(\frac{\beta T}{\gamma} + \frac{2r}{\gamma} \right)D^2 + \frac{\beta^2 L^2 T}{\gamma} .
   \end{align*}
   Rearranging the terms on both sides of the inequality, we have
   \begin{align*}
       \E\left[\frac{\beta D T}{\gamma}  \norm{\overline{\mH}^{(k)}}_F \right] \leq \frac{\beta D \sigma\sqrt{T} }{\gamma} + \left(\frac{\beta T}{\gamma} + \frac{2r}{\gamma} \right)D^2 + \frac{\beta^2 L^2 T}{\gamma} + \E[f(\mW_0^{(k)}) - f(\mW_0^{(k+1)}) ].
   \end{align*}
   By telescoping both sides through $k \in [K]$, we have
   {\small
    \begin{align*}
       \E\left[\frac{\beta D T}{\gamma}  \sum_{k=1}^K\norm{\overline{\mH}^{(k)}}_F \right] \leq & \frac{\beta D \sigma\sqrt{T} K}{\gamma} + \left(\frac{\beta T }{\gamma} + \frac{2r}{\gamma} \right)D^2 K + \frac{\beta^2 L^2 T K}{\gamma} + \E[f(\mW_0^{(0)}) - f(\mW_0^{(K+1)}) ] \\
       \leq & \frac{\beta D \sigma\sqrt{T} K}{\gamma} + \left(\frac{\beta T }{\gamma} + \frac{2r}{\gamma} \right)D^2 K + \frac{\beta^2 L^2 T K}{\gamma} + \Delta_{f},
   \end{align*}
   }
   where the last inequality follows from the definition of $\Delta_f$.
   Dividing both sides by $\beta D T K / \gamma$ yields that
   \begin{align*}
       \E\left[\frac{1}{K}  \sum_{k=1}^K\norm{\overline{\mH}^{(k)}}_F \right] \leq \frac{\sigma}{\sqrt{T}} + \left(1 + \frac{2r}{\beta T} \right)D + \frac{\beta L^2 }{D} + \frac{\gamma \Delta_{f} }{\beta D T K}.
   \end{align*}
   To translate the bound on $\norm{\overline{\mH}^{(k)}}_F$ into a Goldstein-stationarity bound at $\overline{\mW}^{(k)}$, we need every iterate $\mW_t^{(k)}$ to lie within the $\delta$-ball of $\overline{\mW}^{(k)}$.
   Applying Lemma~\ref{lemma:avg_dist}, for any $t \in [T]$, we can show that
   \begin{align*}
       \norm{\mW_t^{(k)} - \overline{\mW}^{(k)}}_F \leq \sqrt{T \sum_{t=1}^T \norm{\mW_t^{(k)} - \mW_{t-1}^{(k)}}_F^2} \leq \sqrt{T \eta^2 \sum_{t=0}^{T-1} \norm{\mM_t^{(k)}}_F^2} \leq \eta T D\sqrt{r} \leq \delta,
   \end{align*}
   where the second inequality follows from the step (\ref{update_step}) in Algorithm \ref{alg:musec} and the third inequality is due to $\norm{\mM_t^{(k)}}_F^2 \leq r D^2$.
   It implies that
   \begin{align*}
      {\rm dist}(\vzero, \partial_\delta f(\overline{\mW}^{(k)})) \leq  \norm{\frac{1}{T}\sum_{t=1}^T \mH_t^{(k)}}_F = \norm{\overline{\mH}^{(k)}}_F  .
   \end{align*}
   Consequently, we obtain
   \begin{align*}
       \E\left[\frac{1}{K}\sum_{k=1}^K  {\rm dist}(\vzero, \partial_\delta f(\overline{\mW}^{(k)})) \right] \leq \frac{\sigma}{\sqrt{T}} + \left(1 + \frac{2r}{\beta T} \right)D + \frac{\beta L^2 }{D} + \frac{\gamma \Delta_{f}}{\beta D T K}.
   \end{align*}
\end{proof}

\subsection{Proof of Corollary \ref{corollary:complexity}}
We restate the complexity bound of Musec as follows.
\begin{corollary}
        By choosing the parameters as
    \begin{align*}
        & T = {(\delta N)^{{2}/{3}}}, \quad K  = \delta^{-{2}/{3}}N^{{1}/{3}}, \quad D = (\delta N)^{-{1}/{3}},  \\
        & \beta = \frac{r^{{1}/{2}}}{(\delta N)^{{2}/{3}} } ,  \quad \gamma = \frac{r}{\delta^{{4}/{3}} N^{{1}/{3}}  }, \quad \eta = \frac{\delta^{{2}/{3}}}{\sqrt{r} N^{{1}/{3}}},
    \end{align*}
    then it takes at most 
    \begin{align*}
      N = \gO\left( \frac{r^{{3}/{4}}}{\delta} + \frac{\delta^2 \rho^3} { r^{{3}/{2}}} + \frac{r^{{3}/{2}} (\sigma^3 + L^6 + \Delta_{f}^3)}{\delta \epsilon^3}\right)
    \end{align*}
    to obtain a $(\delta, \epsilon)$-Goldstein stationary point.
\end{corollary}

\begin{proof}
The theorem imposes four constraints on the parameters:
    \begin{align*}
        \eta T D\sqrt{r} \leq \delta, \quad \beta = \eta \gamma, \quad \beta \leq \frac{1}{8}, \quad \eta^{-1} \geq \rho
    \end{align*}
We first specify the parameters, then verify that each constraint is satisfied under a mild lower bound on $N$, and finally translate the theorem's bound into the stated oracle complexity.

\noindent\textbf{Step 1: Parameter choices.} Set
\begin{align}\label{param_choice}
\begin{split}
        & T = {(\delta N)^{{2}/{3}}}, \quad K = {\frac{N}{T}} = \delta^{-{2}/{3}}N^{{1}/{3}}, \quad D = (\delta N)^{-{1}/{3}},  \\
        & \beta = \frac{r^{{1}/{2}}}{(\delta N)^{{2}/{3}} } , \quad \gamma  = \frac{r}{\delta^{{4}/{3}} N^{{1}/{3}}  }, \quad \eta = \frac{\delta^{{2}/{3}}}{\sqrt{r} N^{{1}/{3}}}
    \end{split}
\end{align}
\medskip
\noindent\textbf{Step 2: Verifying the constraints.} A direct computation gives
\begin{align*}
\eta \gamma &= \frac{\delta^{{2}/{3}}}{\sqrt{r} N^{{1}/{3}}} \cdot \frac{r}{\delta^{{4}/{3}} N^{{1}/{3}}  } = \beta,
\end{align*}
so $\beta = \eta \gamma$ holds by construction. For the step-size constraint,
\begin{align*}
\eta T D \sqrt{r} = \frac{\delta^{{2}/{3}}}{\sqrt{r} N^{{1}/{3}}}(\delta N)^{{1}/{3}}\sqrt{r}\leq \delta,
\end{align*}
so $\eta T D \sqrt{r} \leq \delta$ also holds.
The remaining two constraints $\beta \leq \tfrac{1}{8}$ and $\eta^{-1} \geq \rho$ translate into lower bounds on $N$:
\begin{align*}
 N \geq \gO\left(\frac{r^{{3}/{4}}}{\delta} \right), \qquad  N \geq \frac{\rho^3 \delta^2}{r^{{3}/{2}}}.
\end{align*}

\noindent\textbf{Step 3: Evaluating the theorem's bound.} By the theorem~\ref{lemma:dist_goldstein} implies that
    \begin{align*}
       \E\left[\frac{1}{K}\sum_{k=1}^K  {\rm dist}(\vzero, \partial_\delta f(\overline{\mW}^{(k)})) \right] \leq & \frac{\sigma}{\sqrt{T}} + \left(1 + \frac{2r}{\beta T} \right)D + \frac{\beta L^2 }{D} + \frac{\Delta_{f} \gamma}{\beta D T K} \\
       \leq & \frac{\sigma}{(\delta N)^{\frac{1}{3}}} + \frac{1}{(\delta N)^{\frac{1}{3}}} + \frac{2r^{\frac{1}{2}}}{( \delta N)^{\frac{1}{3}}} + \frac{r^{\frac{1}{2}}L^2}{(\delta N)^{\frac{1}{3}}} + \frac{\Delta_{f} r^{\frac{1}{2}}}{(\delta N)^{\frac{1}{3}}},
   \end{align*}
   where the last inequality follows from the parameter choices in (\ref{param_choice}).
    To obtain $(\delta, \epsilon)$-Goldstein stationary point such that it satisfies 
    \begin{align*}
        \E\left[\frac{1}{K}\sum_{k=1}^K  {\rm dist}(\vzero, \partial_\delta f(\overline{\mW}^{(k)})) \right] \leq \epsilon,
    \end{align*} 
    it takes at most
    \begin{align*}
     N = \gO\left( \frac{r^{\frac{3}{4}}}{\delta} + \frac{\delta^2 \rho^3} { r^{\frac{3}{2}}} + \frac{r^{\frac{3}{2}} (\sigma^3 + L^6 + \Delta_{f}^3)}{\delta \epsilon^3}\right)
    \end{align*}
    stochastic gradient oracle queries.
\end{proof}

\section{Detailed Algorithm of Soft Musec}\label{app:practical_soft}

\IncMargin{1em}
\begin{algorithm}[!t]
\caption{Soft Spectral Clipping (SSC)}
\label{alg:ssc}
\textbf{Input:} Matrix $\mM \in \mathbb{R}^{m \times n}$, clipping threshold $D > 0$,  number of Newton--Schulz iterations~$K$. \vspace{0.05cm}

$\mA = \mM \mM^\top + D^2 \mI$ \vspace{0.05cm}

$\alpha = \norm{\mA}_F$ \vspace{0.05cm}

$\mY_0 = \mA / \alpha$ \vspace{0.05cm}

$\mZ_0 =\mI$ \vspace{0.05cm}

\For {$k = 0, 1, \dots, K-1$}
{
    $\mT_{k}  = \frac{1}{2}(3\mI - \mZ_k\mY_k)$ \vspace{0.05cm}

    $\mY_{k+1} = \mY_k \mT_k$ \vspace{0.05cm}

    $\mZ_{k+1} = \mT_k \mZ_k$ \vspace{0.05cm}

}

\textbf{Return: } $D \cdot (\mZ_K / \sqrt{\alpha}) \cdot \mM$.
\end{algorithm}
\DecMargin{1em}

In Section~\ref{sec:practical}, we introduced Soft Musec, which replaces exact SVD-based spectral clipping with a smooth saturation function approximated via coupled Newton--Schulz iterations. Here we provide the full algorithmic details.

Recall from Section~\ref{sec:practical} that the soft spectral clipping operator is given by
\begin{align*}
H(\mM, D) = D(\mM\mM^\top + D^2 \mI)^{-1/2} \mM,
\end{align*}
where $\mM \in \R^{m \times n}$ with $m \leq n$ w.l.o.g.
We approximate the matrix inverse square root 
\begin{align*}
    (\mM\mM^\top + D^2 \mI)^{-1/2}
\end{align*}
using coupled Newton--Schulz iterations \citep[Chapter~6]{higham2008functions}, which involve only matrix--matrix multiplications and are therefore well-suited to GPU computation. The complete procedure is presented in Algorithm~\ref{alg:ssc}.
The normalization by $\norm{\mM\mM^\top + D^2 \mI}_F$ ensures that the initial iterate $\mY_0$ has unit Frobenius norm, placing the eigenvalues in a range where the Newton--Schulz iteration converges. 
The addition of $D^2 \mI$ shifts all eigenvalues away from zero, ensuring numerical stability.

\IncMargin{1em}
\begin{algorithm}[!t]
\caption{Soft Musec}
\label{alg:soft_musec}
\textbf{Input:} Initial point $\mW_0$, momentum parameter $\beta \in [0, 1)$, learning rate $\eta$, clipping threshold $D>0$, positive integers $K$ and $T$, number of Newton--Schultz steps $K_{\rm ns}$ \vspace{0.05cm} \vspace{0.05cm}

$N = K \times T$ \vspace{0.05cm}

\For {$n = 0, 1, \dots, N-1$}
{
Sample $\xi_{n} \sim \gP$  \vspace{0.05cm}

$\mG_n =  G(\mW_n; \xi_{n})$ \vspace{0.05cm}

\eIf{$n = 0$}
{
  $\widehat{\mM}_0 = \mG_0$ \vspace{0.05cm}
}
{
  $\widehat{\mM}_n = (1 - \beta) \mM_{n-1} + \beta \mG_n$ \vspace{0.05cm}
}

$\mM_n = {\rm SSC}(\widehat{\mM}_n, D, K_{\rm ns}) $   \tcp{Algorithm~\ref{alg:ssc}} \vspace{0.05cm}

$\mW_{n + 1} = \mW_n - \eta \mM_n$ \vspace{0.05cm}
}

 Set $\overline{\mW}^{(k)} = \frac{1}{T} \sum_{t=0}^{T-1} \mW_t^{(k)}$ where  $\mW_t^{(k)} = \mW_{(k-1)T + t}$ for $\forall k \in [K]$ \vspace{0.05cm}
% Set $\mW_t^{(k)} = \mW_{(k-1)T + t}, \forall k \in [K], t \in \{0\}\cup[T-1]$ and $\overline{\mW}^{(k)} = \frac{1}{T} \sum_{t=0}^{T-1} \mW_t^{(k)}$.

 \textbf{Return:} $\overline{\mW}_T \sim {\rm Uniform} (\{\overline{\mW}^{(k)}: k \in [K]\})$.
\end{algorithm}

We present the full Soft Musec procedure in Algorithm~\ref{alg:soft_musec}, which integrates the soft spectral clipping subroutine (Algorithm~\ref{alg:ssc}) into the Musec framework by replacing the exact SVD and hard clipping in Steps~\ref{svd_op}--\ref{update_step} of Algorithm~\ref{alg:musec}.

\section{Comparison with SPECTRA}\label{app:spectra_comp}
We provide a detailed comparison of the Musec and SPECTRA update rules to clarify their algorithmic differences.

Given a matrix $\mM$ with reduced SVD $\mM = \mU \mS \mV^\top$, we define the spectral clipping operator $\mathbf{Clip}_{\rm sc} (\mM) = \mU \mathbf{Clip}(\mS,D) \mV^{\top}$.
Both Musec and SPECTRA employ this operator, but they differ in how the clipped output interacts with the momentum state.

Musec maintains a clipped momentum buffer and computes:
\begin{align*}
\begin{cases}
&\widehat{\mM}_n = (1 - \beta) \mM_{n-1} + \beta \mG_n, \\
&\mM_n = \mathbf{Clip}_{\rm sc} (\widehat{\mM}_n, D), \\
&\mW_{n + 1} = \mW_n - \eta \mathbf{Clip}_{\rm sc} \mM_n,
\end{cases}
\end{align*}
Crucially, the clipped momentum $\mM_n$ is propagated to the next iteration. This ensures that the momentum state is always spectrally bounded: $\norm{\mM_n} \leq D$ at every iteration.

Ignoring weight decay, the SPECTRA update \citep[Eq.~(4)]{jiang2026enhancing} maintains an unclipped momentum buffer and applies spectral clipping only when computing the parameter update:
\begin{align*}
\begin{cases}
    &{\mM}_n = (1 - \beta) \mM_{n-1} + \beta \mG_n. \\
&\mW_{n + 1} = \mW_n - \eta \mathbf{Clip}_{\rm sc} ({\mM}_n, D),
\end{cases}
\end{align*}
Here, the unclipped momentum $\mM_n$ is carried forward to the next iteration. Although the weight update is spectrally clipped, the momentum state itself is unconstrained.

\paragraph{Key difference.} The distinction lies in whether spectral clipping is integrated into the momentum recurrence. In Musec, the momentum state satisfies $\norm{\mM_n} \leq D$ at every iteration, so the input to the next EMA step is always spectrally bounded. In SPECTRA, the momentum state is unconstrained: $\norm{\mM_n}$ can grow well beyond $D$ over successive iterations, as clipping is only applied to the output update and does not feed back into the recurrence. Consequently, SPECTRA's momentum can retain large spectral components accumulated over many past iterations, even when recent gradients are small in those directions. Musec's feedback mechanism prevents this accumulation, providing tighter spectral control over the optimization trajectory.

\section{Experimental Details}\label{app:exp_details}
This section provides the detailed architecture, training, and hyperparameter configurations for our NanoGPT experiments.

\subsection{Architecture Configurations}

\begin{table}[ht]
\centering
\caption{Architecture specifications across model scales.}
\label{tab:architecture}
\smallskip
\begin{tabular}{lrrr}
\toprule
& \textbf{Small} & \textbf{Medium} & \textbf{Wide} \\
\midrule
Layers              & 11  & 16  & 16 \\
Attention heads     & 6   & 8   & 16 \\
Head dimension      & 128 & 128 & 128 \\
Model dimension     & 768 & 1{,}024 & 2{,}048 \\
MLP hidden dimension & 3{,}072 & 4{,}096 & 8{,}192 \\
Max sequence length  & 2{,}048 & 4{,}096 & 4{,}096 \\
Core parameters     & ${\sim}$491M & ${\sim}$613M & ${\sim}$1.63B \\
\bottomrule
\end{tabular}
\end{table}

In our experiments, we evaluate three model configurations, including NanoGPT-small, NanoGPT-Medium, and NanoGPT-Wide.
All three model configurations follow a decoder-only Transformer architecture based on the modded-nanogpt codebase.\footnote{\url{https://github.com/KellerJordan/modded-nanogpt}} The architecture incorporates several modern modifications on top of the standard GPT-2 backbone, including RMSNorm \citep{zhang2019root}, squared ReLU activations \citep{so2021searching}, Rotary Position Embeddings \citep{su2024roformer}, sliding-window causal attention via Flash Attention 3 \citep{shah2024flashattention}, and multi-token prediction as an auxiliary training objective. We refer to the codebase for full architectural details. The three configurations differ primarily in their width scaling. Table~\ref{tab:architecture} summarizes their specifications.

\subsection{Training Configuration}
\label{sec:training}

\begin{table}[ht]
\centering
\caption{Training configuration across model scales.}
\label{tab:training}
\smallskip
\begin{tabular}{lccc}
\toprule
& \textbf{Small} & \textbf{Medium} & \textbf{Wide} \\
\midrule
Total steps          & 1{,}390     & 4{,}740    & 8{,}040 \\
Final batch size (tokens) & 393K   & 524K       & 1{,}049K \\
Precision            & bfloat16    & bfloat16   & bfloat16 \\
Hardware             &  H800 &  H800 & H800 \\
\bottomrule
\end{tabular}
\end{table}

All models are trained using the GPT-2 BPE tokenizer in bfloat16 precision on NVIDIA H800 GPUs. 
We train all configurations on FineWeb \citep{penedo2024fineweb}, OpenWebText \citep{Gokaslan2019OpenWeb}, and C4 \citep{raffel2020exploring}.
Each configuration employs a multi-phase training schedule that jointly ramps up batch size and learning rate. The learning rate follows a warmup-then-decay schedule. Table~\ref{tab:training} summarizes the key training parameters. Full schedule details are available in the modded-nanogpt codebase.

\subsection{Hyperparameters}
\label{app:hyperparameters}

\begin{table}[ht]
\centering
\caption{Selected weight decay $\lambda$ and clipping threshold $D$ for each method and learning rate on NanoGPT-Small. Each entry corresponds to the best validation loss over the sweep. Soft Musec and SPECTRA share the same hyperparameter space.}
\label{tab:hparams_small}
\smallskip
\begin{tabular}{lcccccccccc}
\toprule
& \multicolumn{3}{c}{\textbf{Muon}} & \multicolumn{3}{c}{\textbf{MuonClip}} & \multicolumn{4}{c}{\textbf{Soft Musec / SPECTRA}} \\
\cmidrule(lr){2-4} \cmidrule(lr){5-7} \cmidrule(lr){8-11}
LR & $\lambda_{\text{QK}}$ & $\lambda_{\text{VO}}$ & $\lambda_{\text{MLP}}$ & $\lambda_{\text{QK}}$ & $\lambda_{\text{VO}}$ & $\lambda_{\text{MLP}}$ & $\lambda_{\text{QK}}$ & $\lambda_{\text{VO}}$ & $\lambda_{\text{MLP}}$ & $D$ \\
\midrule
0.01 & 1.2 & 1.2 & 0.3 & 1.2 & 1.2 & 0.3 & 1.2 & 1.2 & 0.3 & 0.5 \\
0.023 & 1.2 & 1.2 & 1.2 & 1.2 & 1.2 & 1.2 & 1.2 & 1.2 & 0.3 & 0.5 \\
0.05 & 1.2 & 1.2 & 1.2 & 1.2 & 1.2 & 1.2 & 1.2 & 1.2 & 0.3 & 0.1 \\
0.1  & 1.2 & 1.2 & 0.3 & 1.2 & 1.2 & 0.3 & 1.2 & 1.2 & 0.3 & 0.1 \\
0.2  & 1.2 & 1.2 & 0.3 & 1.2 & 1.2 & 0.3 & 1.2 & 1.2 & 0.3 & 0.1 \\
0.5  & 1.2 & 1.2 & 0.3 & 1.2 & 1.2 & 0.3 & 1.2 & 1.2 & 0.3 & 0.1 \\
\bottomrule
\end{tabular}
\end{table}

\begin{table}[ht]
\centering
\caption{Selected weight decay $\lambda$ and clipping threshold $D$ for each method and learning rate on NanoGPT-Medium. Each entry corresponds to the best validation loss over the sweep. Soft Musec and SPECTRA share the same hyperparameter space.}
\label{tab:hparams_medium}
\smallskip
\begin{tabular}{lrrrrrrrrrrr}
\toprule
& \multicolumn{3}{c}{\textbf{Muon}} & \multicolumn{3}{c}{\textbf{MuonClip}} & \multicolumn{4}{c}{\textbf{Soft Musec / SPECTRA}} \\
\cmidrule(lr){2-4} \cmidrule(lr){5-7} \cmidrule(lr){8-11}
LR & $\lambda_{\text{QK}}$ & $\lambda_{\text{VO}}$ & $\lambda_{\text{MLP}}$ & $\lambda_{\text{QK}}$ & $\lambda_{\text{VO}}$ & $\lambda_{\text{MLP}}$ & $\lambda_{\text{QK}}$ & $\lambda_{\text{VO}}$ & $\lambda_{\text{MLP}}$ & $D$ \\
\midrule
0.01 & 1.2 & 1.2 & 1.2 & 1.2 & 1.2 & 1.2 & 1.2 & 1.2 & 0.3 & 0.25 \\
0.015 & 1.2 & 1.2 & 1.2 & 1.2  & 1.2  & 1.2  & 1.2 & 1.2 & 0.3 & 0.25 \\
0.05 & 1.2 & 1.2 & 1.2 & 1.2  & 1.2  & 1.2  & 1.2 & 1.2 & 0.3 & 0.05 \\
0.1  & 1.2 & 1.2 & 0.3 & 1.2  & 1.2  & 1.2  & 1.2 & 1.2 & 0.3 & 0.05\\
0.2  & 1.2 & 1.2 & 0.3 & 1.2  & 1.2  & 1.2  & 1.2 & 1.2 & 0.3 & 0.05 \\
0.3  & 1.2 & 1.2 & 0.3 & 1.2  & 1.2  & 0.3 & 1.2 & 1.2 & 0.3 & 0.05 \\
0.5  & 0.3 & 0.3 & 0.3 & 0.3  & 0.3  & 0.3 & 0.3 & 0.3 & 0.3 & 0.05 \\
0.8  & 0.3 & 0.3 & 0.1 & 0.3 & 0.3 & 0.1 & 0.3 & 0.3 & 0.1 & 0.05 \\
\bottomrule
\end{tabular}
\end{table}

\begin{table}[ht]
\centering
\caption{Selected weight decay $\lambda$ and clipping threshold $D$ for each method and learning rate on NanoGPT-Wide. Each entry corresponds to the best validation loss over the sweep. Soft Musec and SPECTRA share the same hyperparameter space.}
\label{tab:hparams_wide}
\smallskip
\begin{tabular}{lcccccccccc}
\toprule
& \multicolumn{3}{c}{\textbf{Muon}} & \multicolumn{3}{c}{\textbf{MuonClip}} & \multicolumn{4}{c}{\textbf{Soft Musec / SPECTRA}} \\
\cmidrule(lr){2-4} \cmidrule(lr){5-7} \cmidrule(lr){8-11}
LR & $\lambda_{\text{QK}}$ & $\lambda_{\text{VO}}$ & $\lambda_{\text{MLP}}$ & $\lambda_{\text{QK}}$ & $\lambda_{\text{VO}}$ & $\lambda_{\text{MLP}}$ & $\lambda_{\text{QK}}$ & $\lambda_{\text{VO}}$ & $\lambda_{\text{MLP}}$ & $D$ \\
\midrule
0.01 & 1.2 & 1.2 & 1.2 & 1.2 & 1.2 & 1.2 & 1.2 & 1.2 & 0.3 & 0.25 \\
0.05 & 1.2 & 1.2 & 1.2 & 1.2 & 1.2 & 1.2 & 1.2 & 1.2 & 0.3 & 0.05 \\
0.1 & 1.2 & 1.2 & 0.3 & 1.2 & 1.2 & 1.2 & 1.2 & 1.2 & 0.3 & 0.05 \\
0.2  & 1.2 & 1.2 & 0.3 & 1.2 & 1.2 & 1.2 & 1.2 & 1.2 & 0.3 & 0.05 \\
0.3  & 1.2 & 1.2 & 0.3 & 1.2 & 1.2 & 0.3 & 1.2 & 1.2 & 0.3 & 0.05 \\
0.5  & 0.3 & 0.3 & 0.3 & 0.3 & 0.3 & 0.3 & 0.3 & 0.3 & 0.3 & 0.05 \\
\bottomrule
\end{tabular}
\end{table}

For all methods, the respective optimizer is applied to attention and MLP projection matrices, while AdamW is used for embeddings, the language-model head, and scalar parameters. AdamW hyperparameters are held fixed across all methods. For each method, model configuration, and learning rate, we tune the weight decay over $\{0.096,0.3,0.6,1.2\}$. For Soft Musec and SPECTRA, we additionally tune the clipping threshold 
$D$ over $\{0.05,0.125,0.25,0.5,0.75,1.0\}$. We report the best validation loss over these sweeps. 
Tables~\ref{tab:hparams_small}, \ref{tab:hparams_medium}, and~\ref{tab:hparams_wide} list the selected hyperparameters for NanoGPT-Small, NanoGPT-Medium, and NanoGPT-Wide, respectively.

\section{Additional Experimental Results}\label{app:additional_exp}

This section presents additional experimental results comparing Soft Musec against baseline methods on NanoGPT training.

\subsection{Learning Rate Sweep}\label{app:lr_sweep}

\begin{figure}[htbp]
    \centering
    \begin{tabular}{ccc}
        % First Row
        \includegraphics[width=0.31\textwidth]{figures/lr_sweep/fineweb_small.pdf} &
        \includegraphics[width=0.31\textwidth]{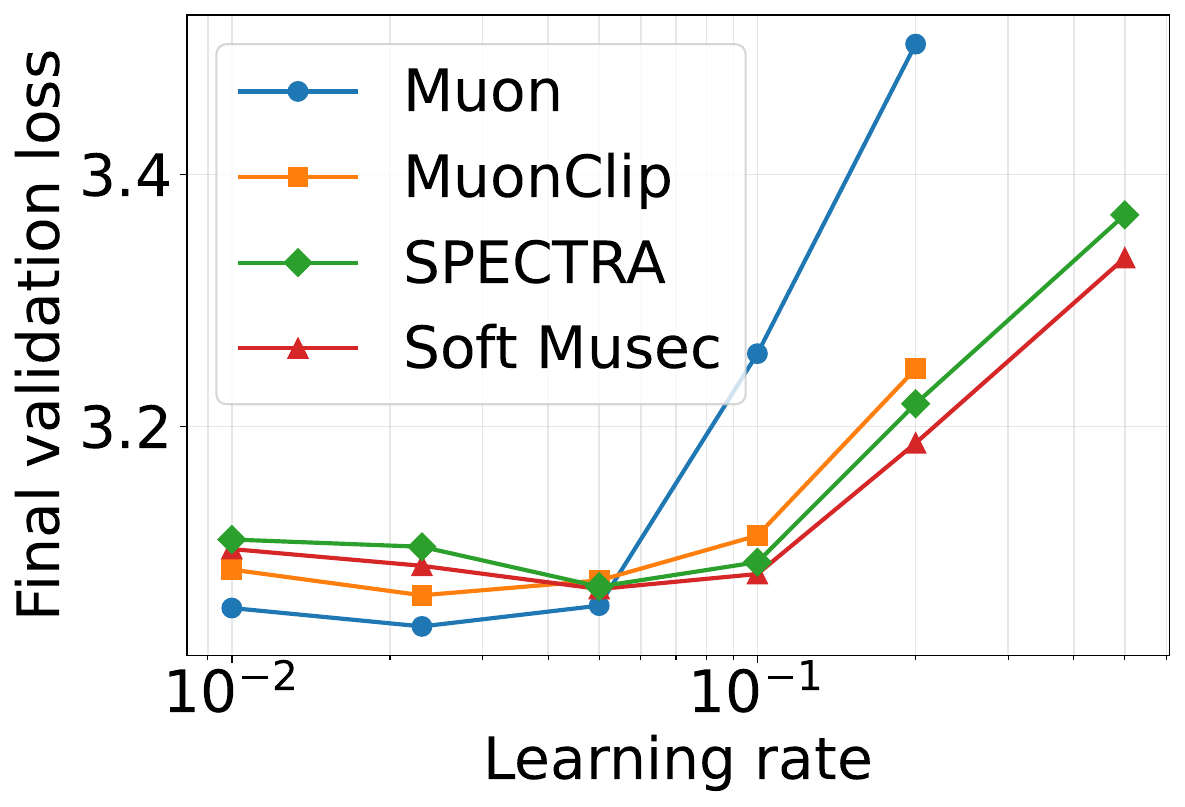} &
        \includegraphics[width=0.31\textwidth]{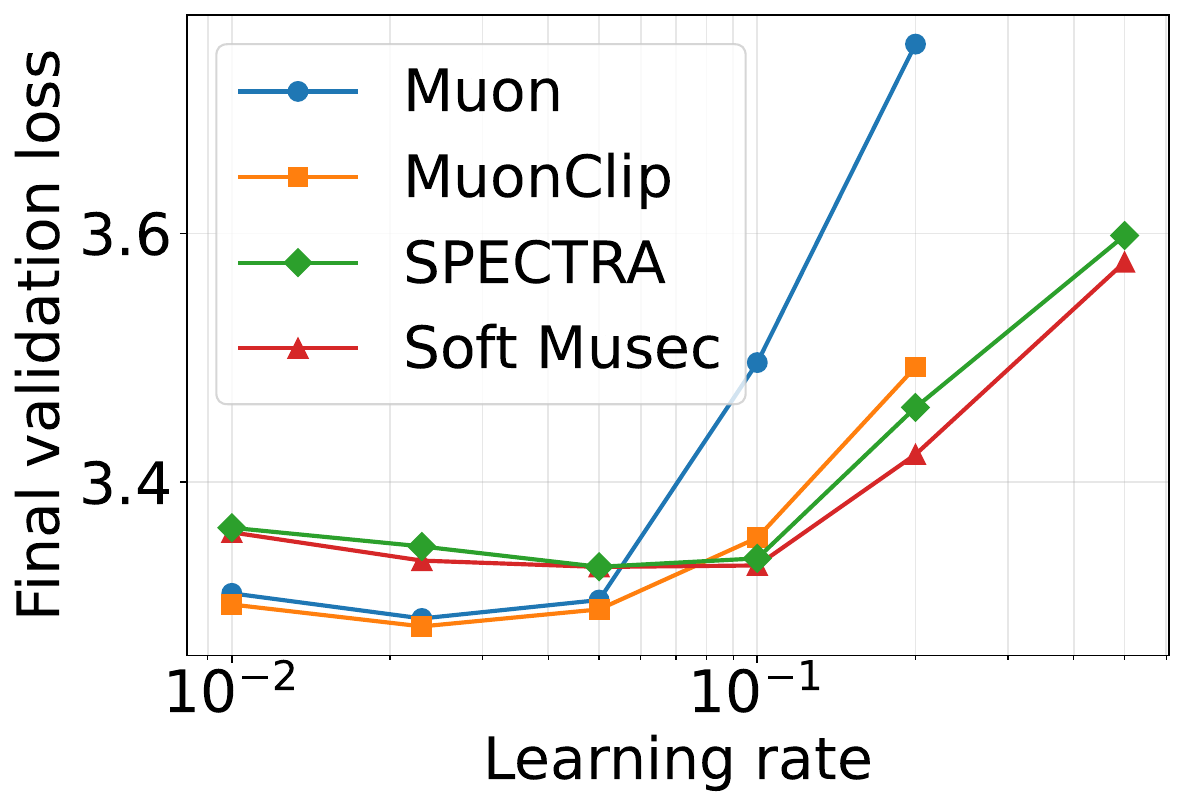} \\
        (a) FineWeb & (b) OpenWebText & (c) C4 \\[1em]
    \end{tabular}
    \caption{Validation loss versus learning rate for NanoGPT-Small across three datasets.}
    \label{fig:lr_sweep_small}
\end{figure}

\begin{figure}[htbp]
    \centering
    \begin{tabular}{ccc}
        % First Row
        \includegraphics[width=0.31\textwidth]{figures/lr_sweep/fineweb_medium.pdf} &
        \includegraphics[width=0.31\textwidth]{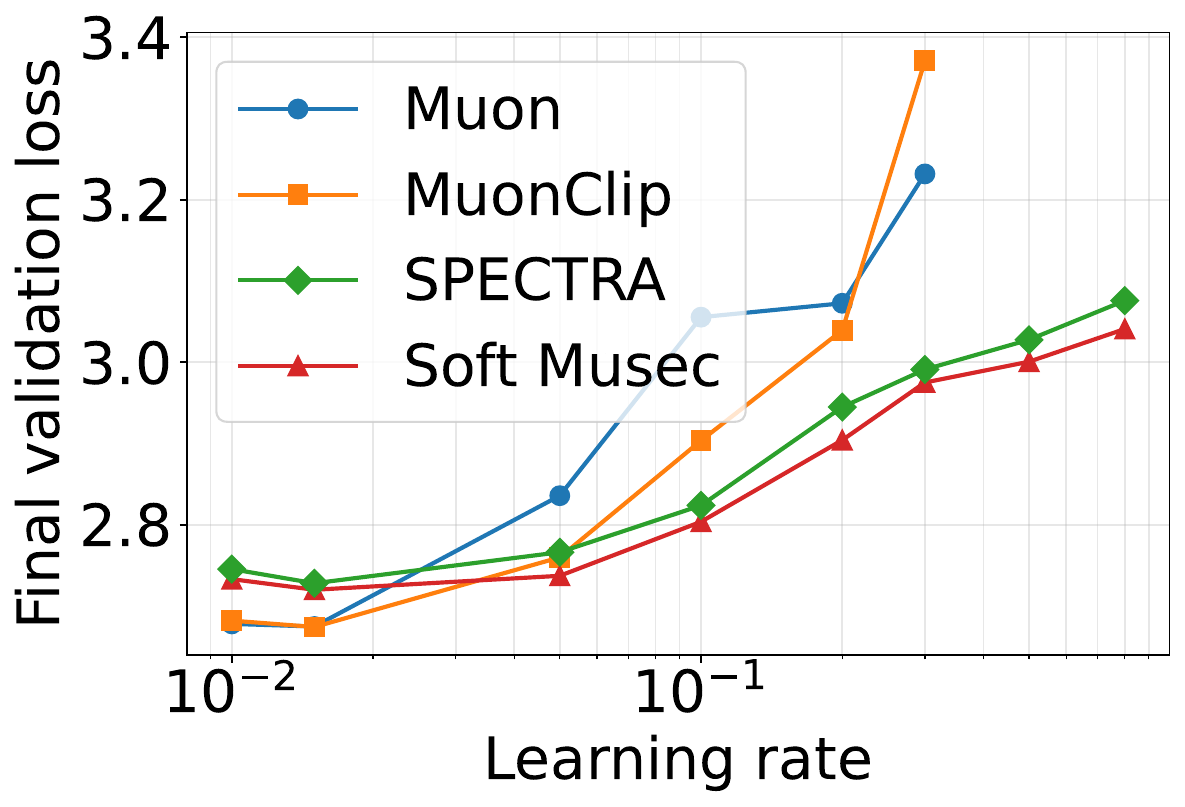} &
        \includegraphics[width=0.31\textwidth]{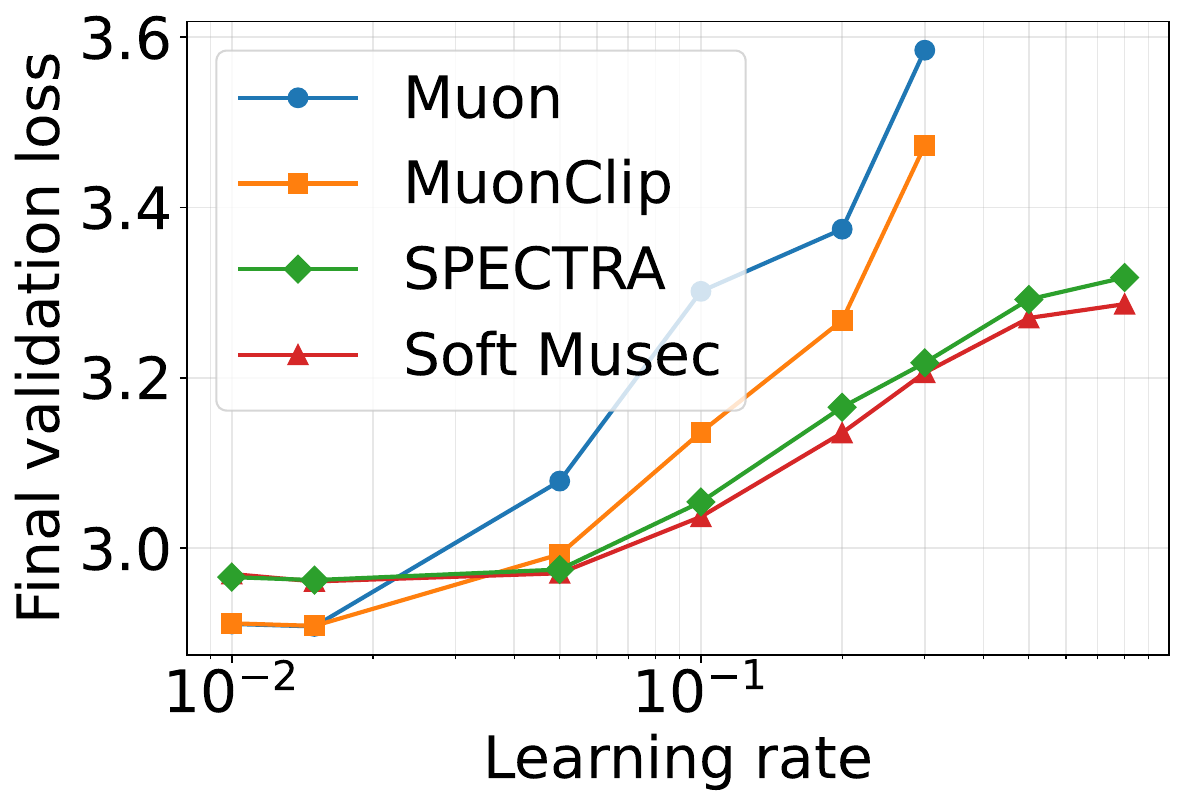} \\
        (a) FineWeb & (b) OpenWebText & (c) C4 \\[1em]
    \end{tabular}
    \caption{Validation loss versus learning rate for NanoGPT-Medium across three datasets.}
    \label{fig:lr_sweep_medium}
\end{figure}

Figures~\ref{fig:lr_sweep_small}--\ref{fig:lr_sweep_medium} present the complete learning rate sweep results across all model configurations and datasets; runs that diverge during training are omitted.
For NanoGPT-Small (Figure~\ref{fig:lr_sweep_small}), all methods achieve comparable validation loss at small learning rates. As the learning rate increases, Muon degrades most rapidly, followed by MuonClip, while Soft Musec and SPECTRA maintain lower validation loss across the full range. For NanoGPT-Medium (Figure~\ref{fig:lr_sweep_medium}), the advantage of spectral clipping methods becomes more pronounced: at learning rates $\eta \geq 0.1$, both Soft Musec and SPECTRA substantially outperform Muon and MuonClip, with Soft Musec achieving the best or near-best validation loss at the largest learning rates across all three datasets.

\subsection{Additional Training Dynamics}\label{app:training_dynamics}
\begin{figure}[p]
\centering
\begin{tabular}{cccc}
& \textbf{FineWeb} & \textbf{OpenWebText} & \textbf{C4} \\
\rowlabel{Best Tuned} &
\includegraphics[width=0.3\textwidth]{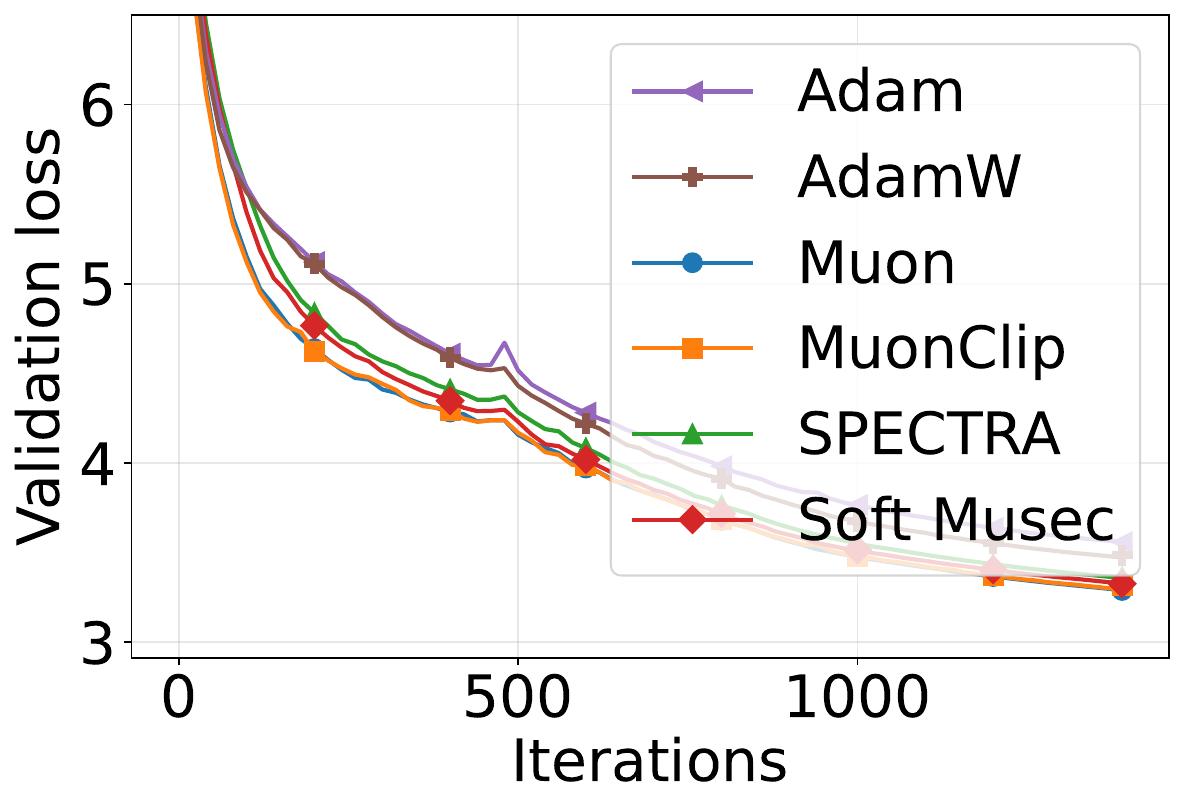} &
\includegraphics[width=0.3\textwidth]{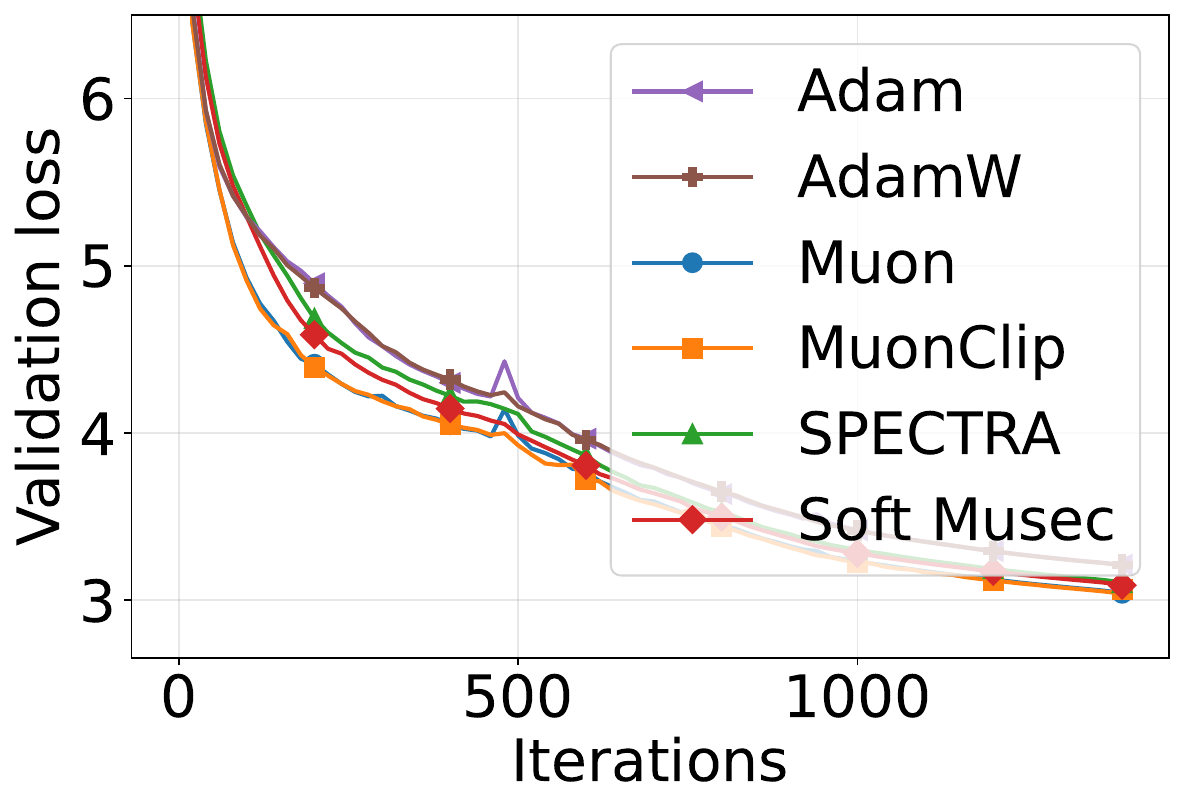} &
\includegraphics[width=0.3\textwidth]{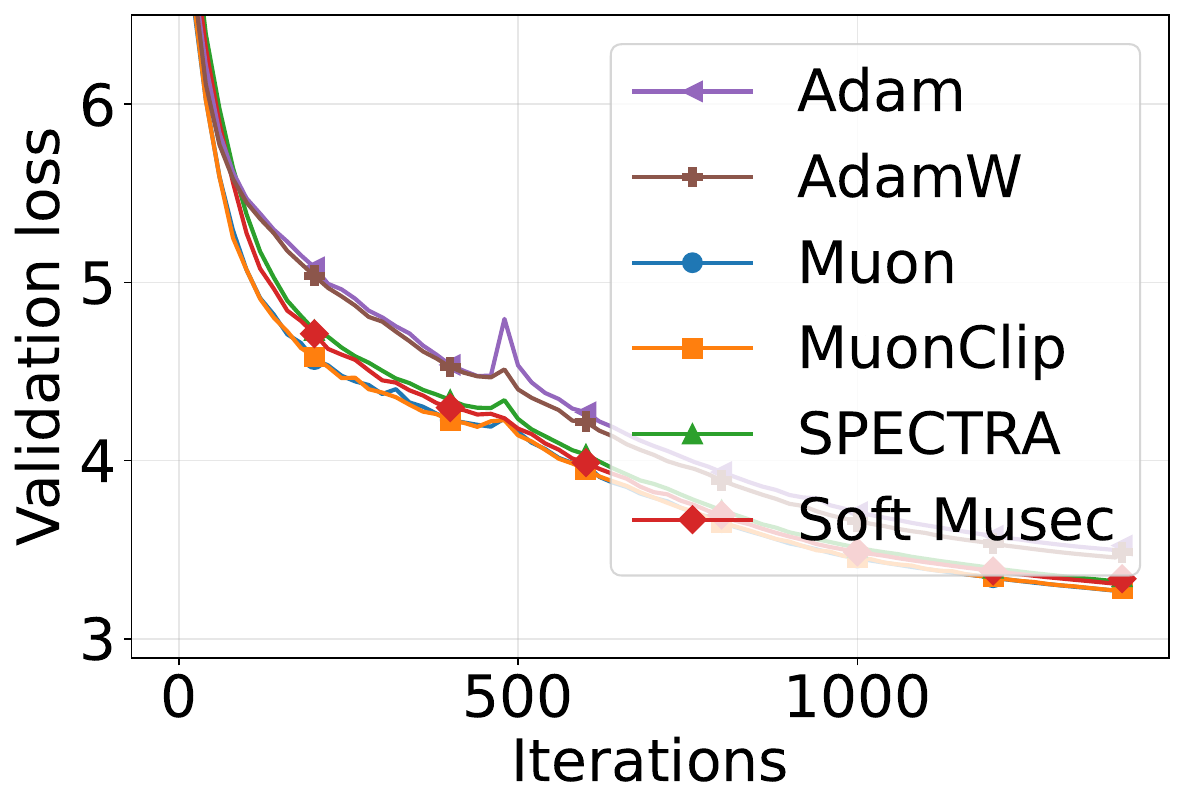} \\
\rowlabel{$\eta = 0.1$} &
\includegraphics[width=0.3\textwidth]{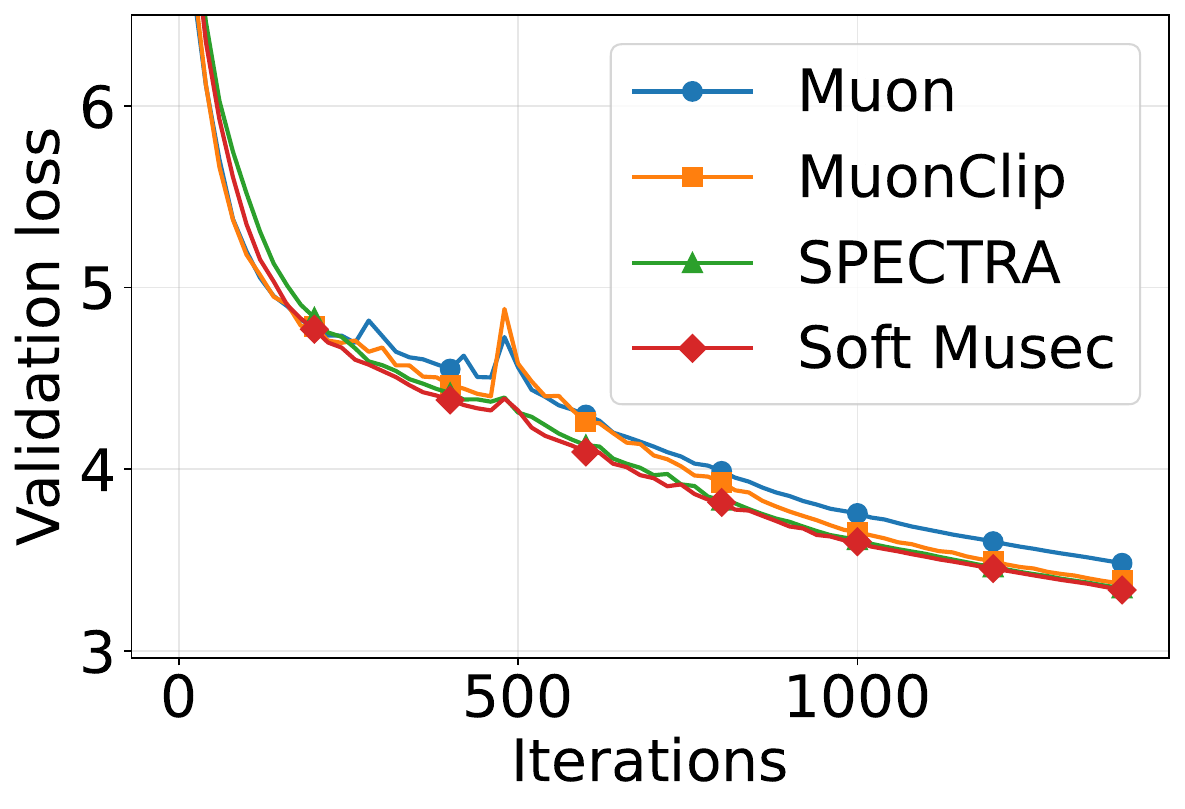} &
\includegraphics[width=0.3\textwidth]{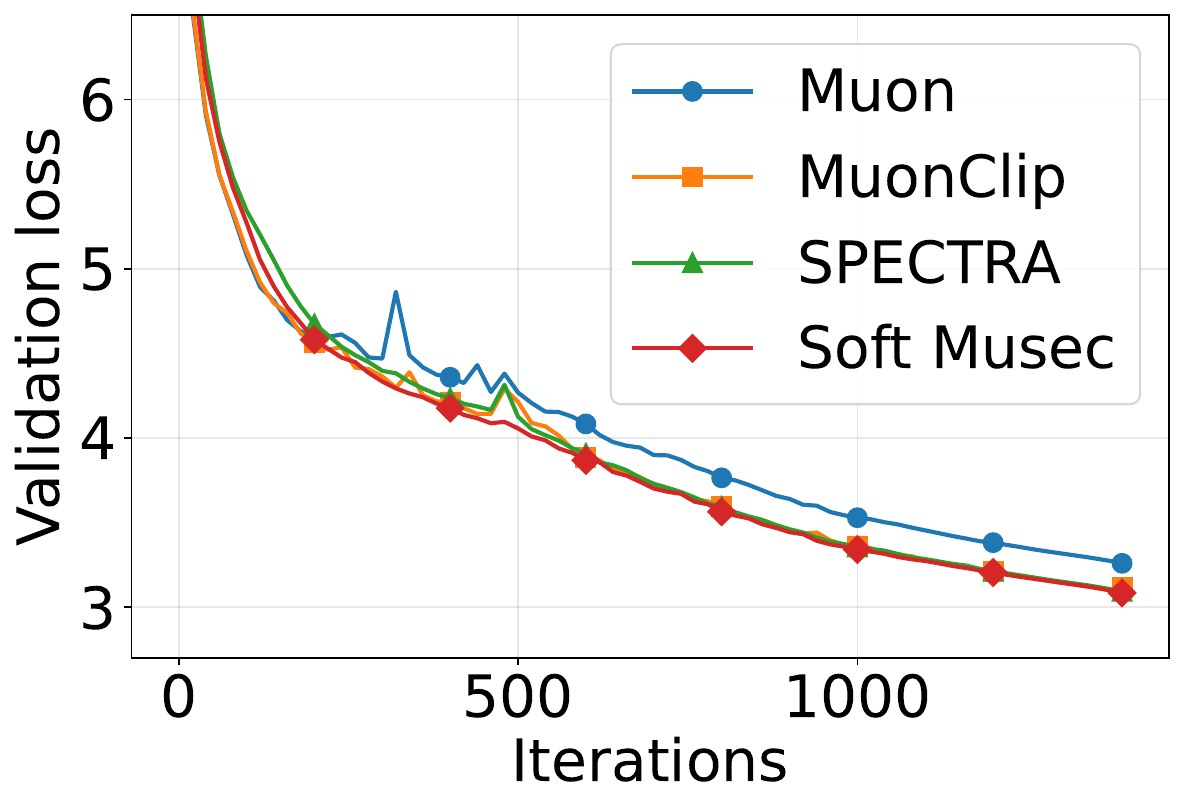} &
\includegraphics[width=0.3\textwidth]{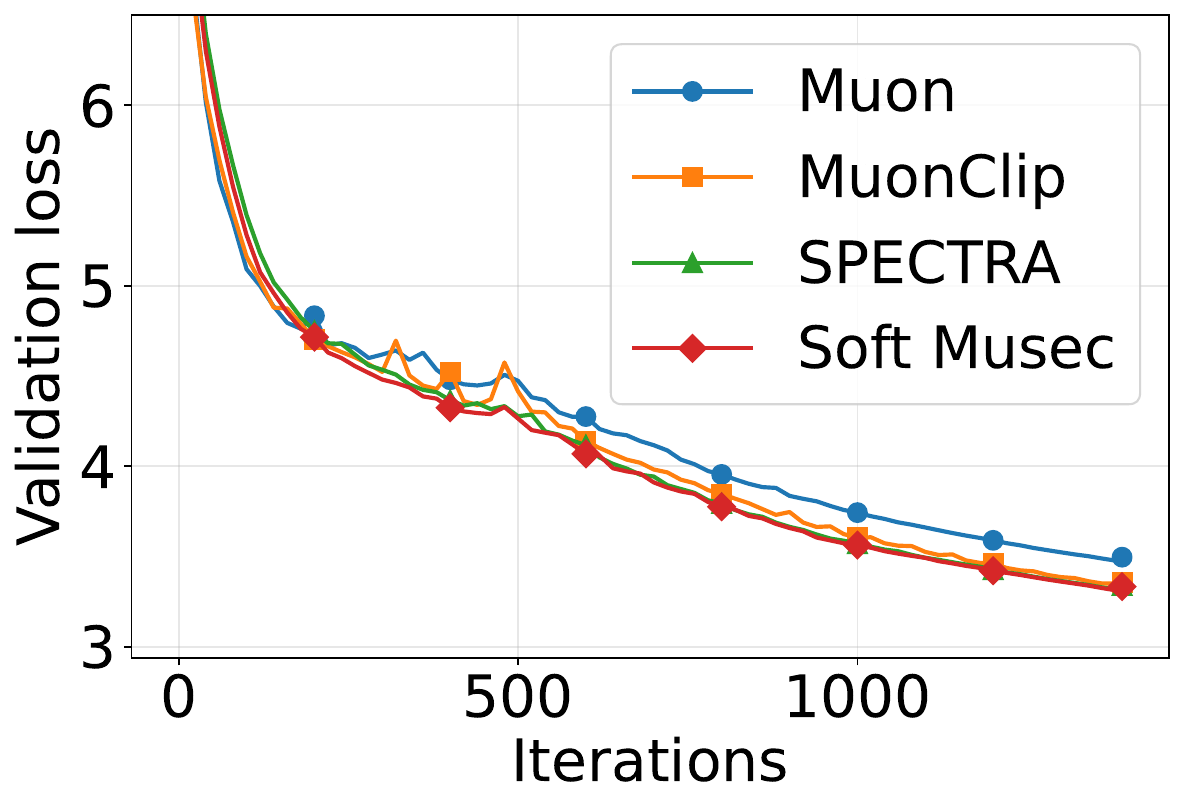} \\
\rowlabel{$\eta = 0.2$} &
\includegraphics[width=0.3\textwidth]{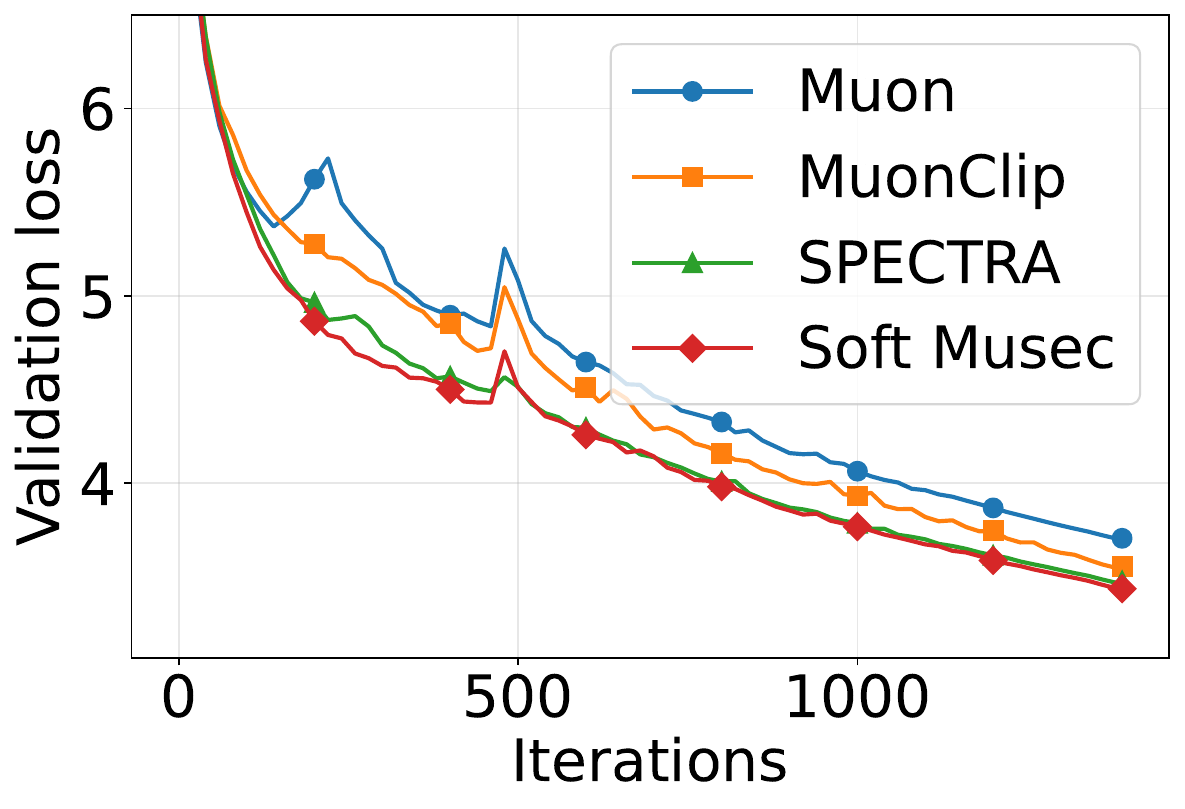} &
\includegraphics[width=0.3\textwidth]{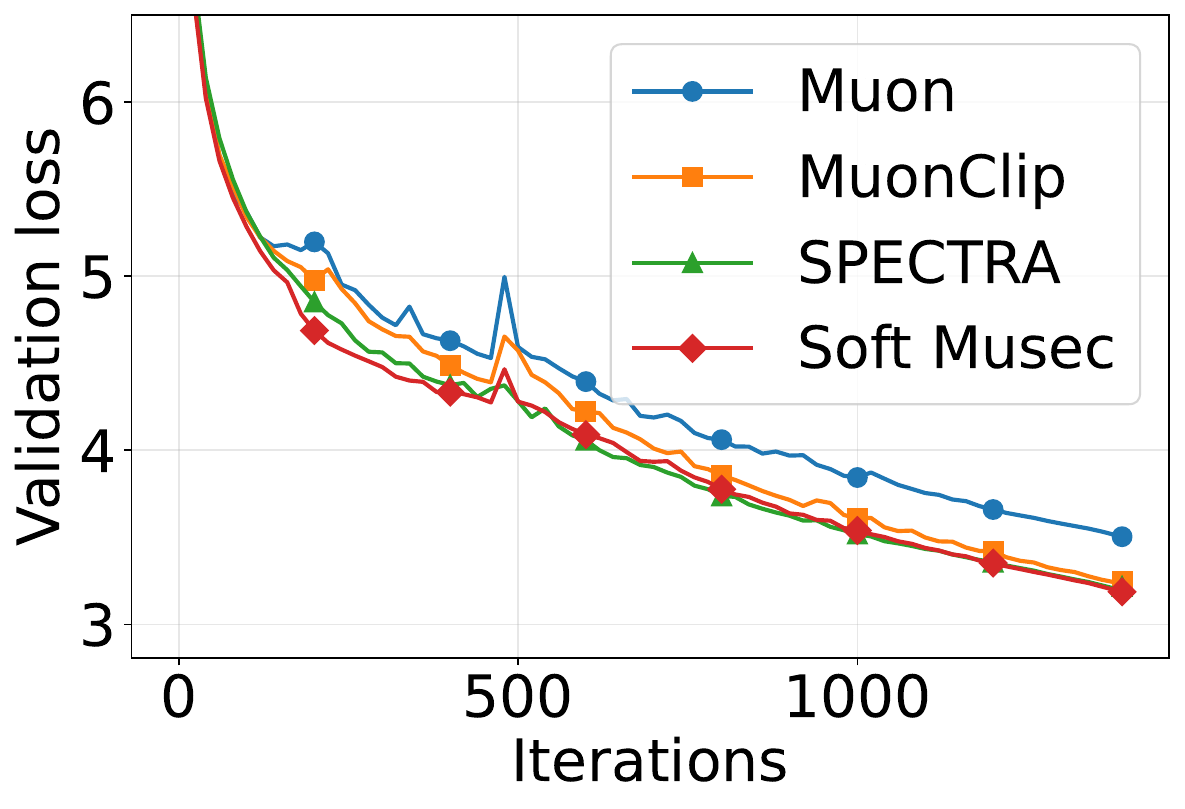} &
\includegraphics[width=0.3\textwidth]{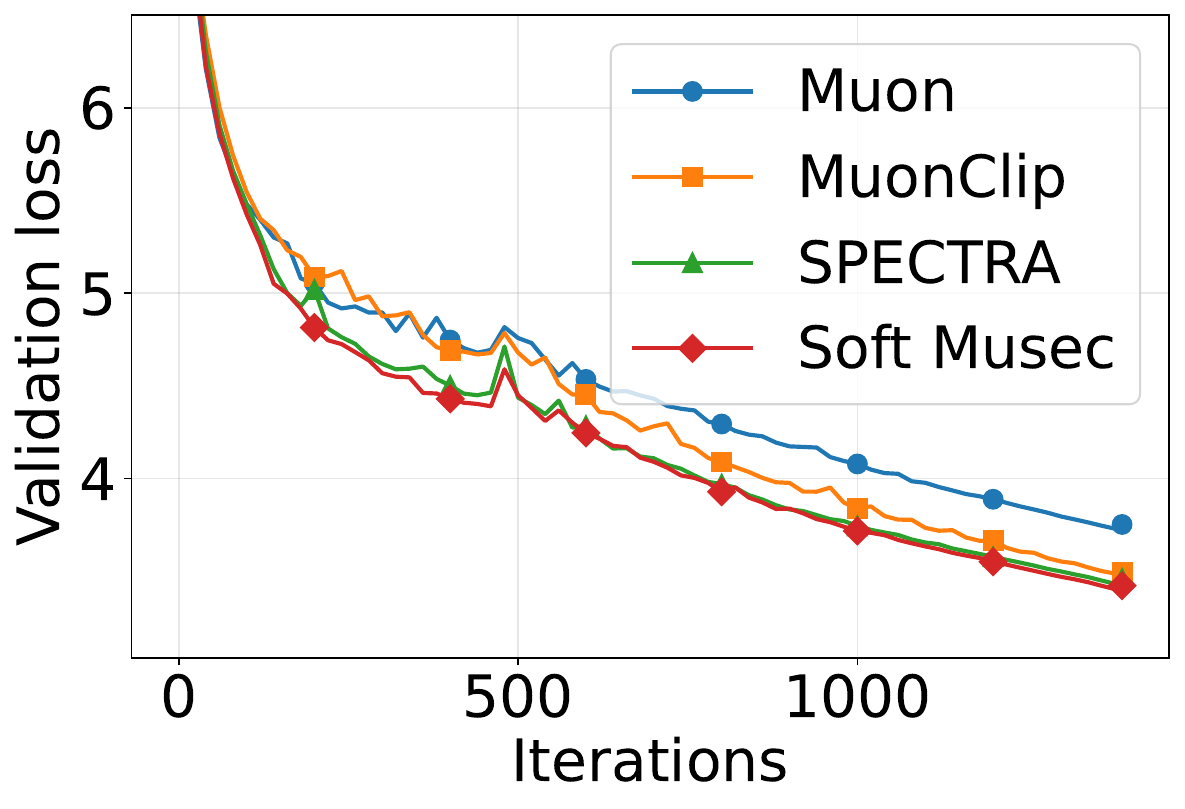} \\
\rowlabel{$\eta = 0.5$} &
\includegraphics[width=0.3\textwidth]{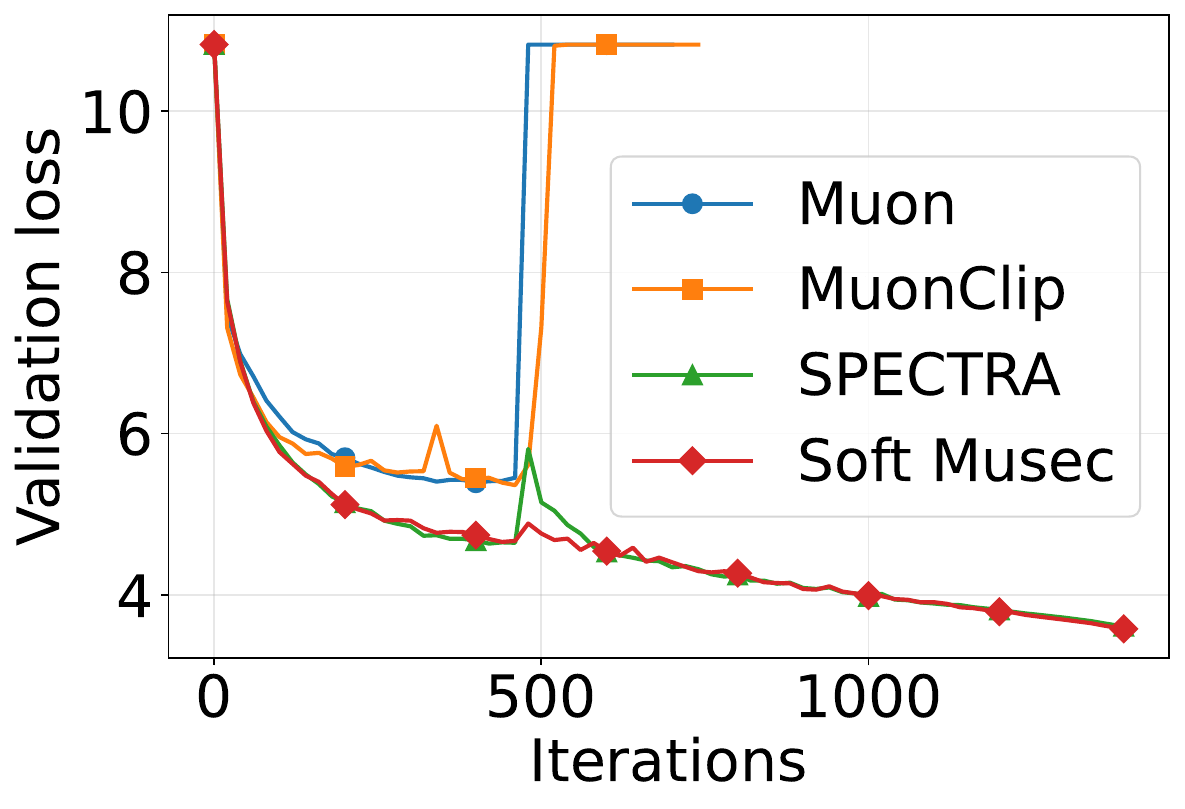} &
\includegraphics[width=0.3\textwidth]{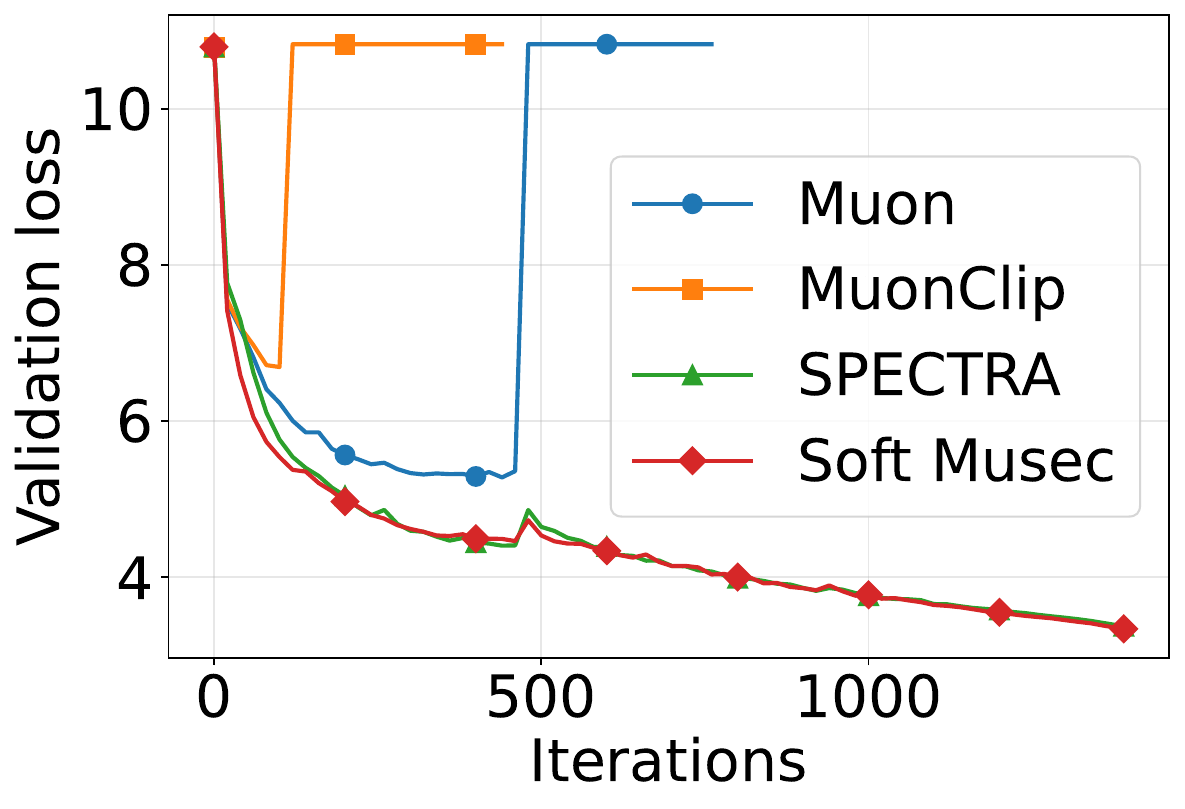} &
\includegraphics[width=0.3\textwidth]{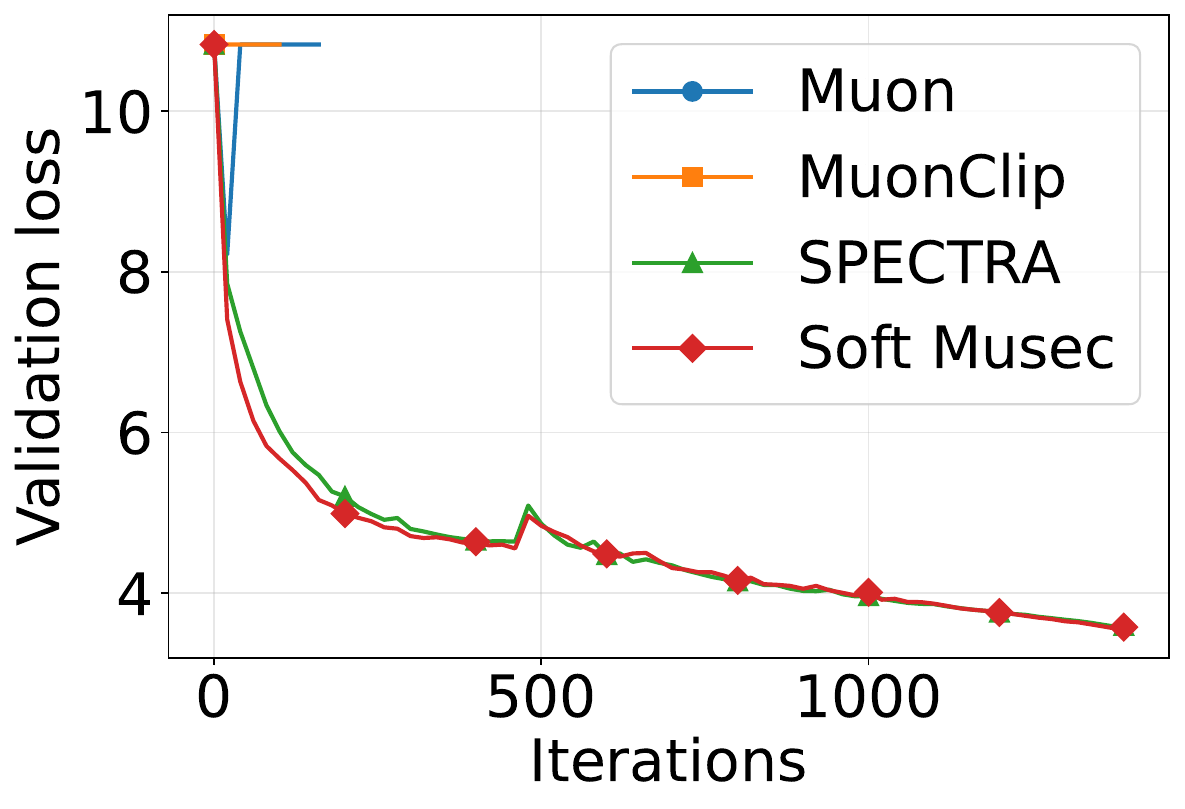} \\
\end{tabular}
\caption{Training loss over steps for NanoGPT-Small. Rows: learning rates; columns: datasets. Only learning rates where methods exhibit meaningfully different behavior are shown. NaN values within runs are omitted.}
\label{fig:td_small}
\end{figure}

\begin{figure}[p]
\centering
\begin{tabular}{cccc}
& \textbf{FineWeb} & \textbf{OpenWebText} & \textbf{C4} \\
\rowlabel{Best Tuned} &
\includegraphics[width=0.3\textwidth]{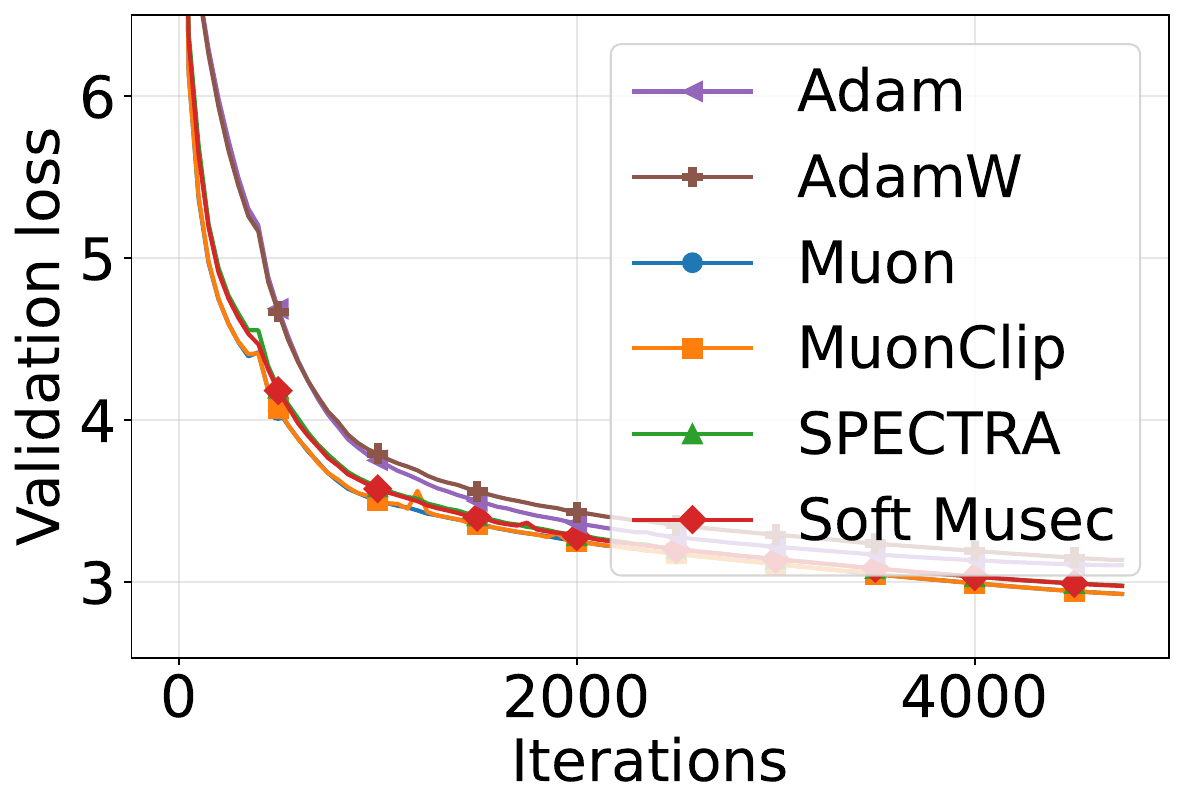} &
\includegraphics[width=0.3\textwidth]{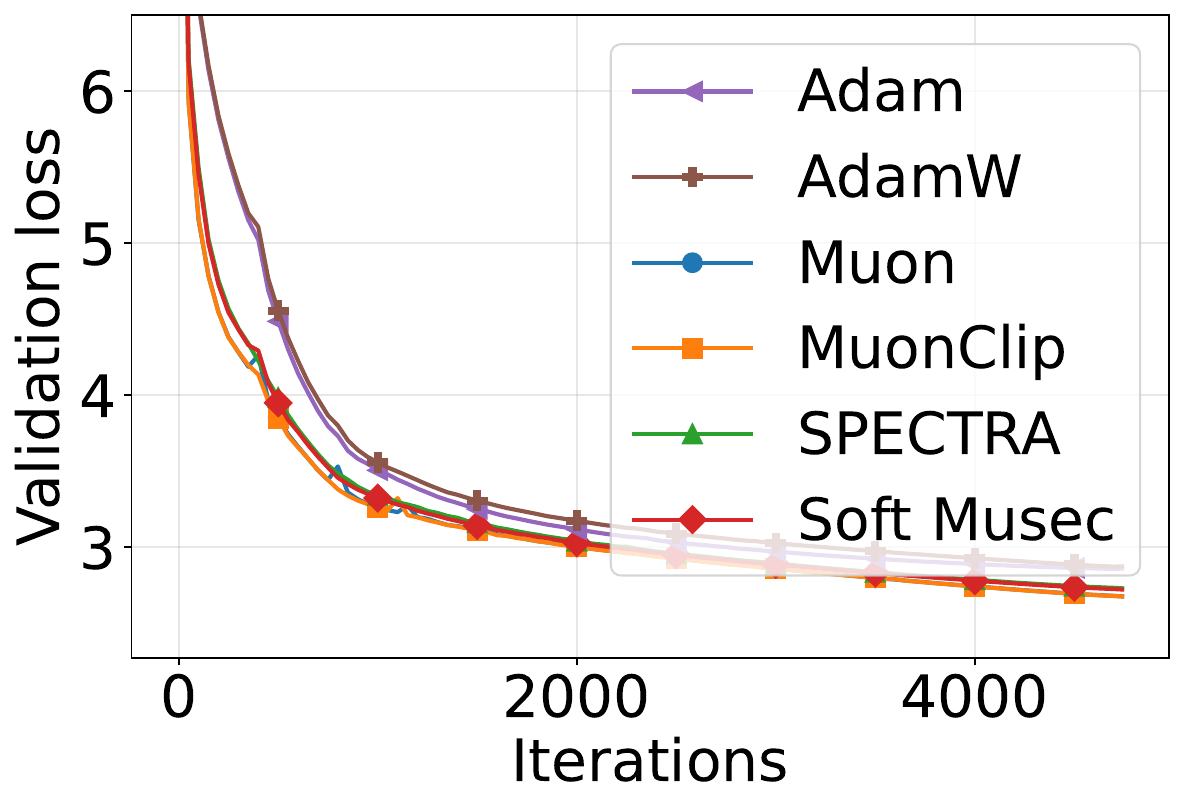} &
\includegraphics[width=0.3\textwidth]{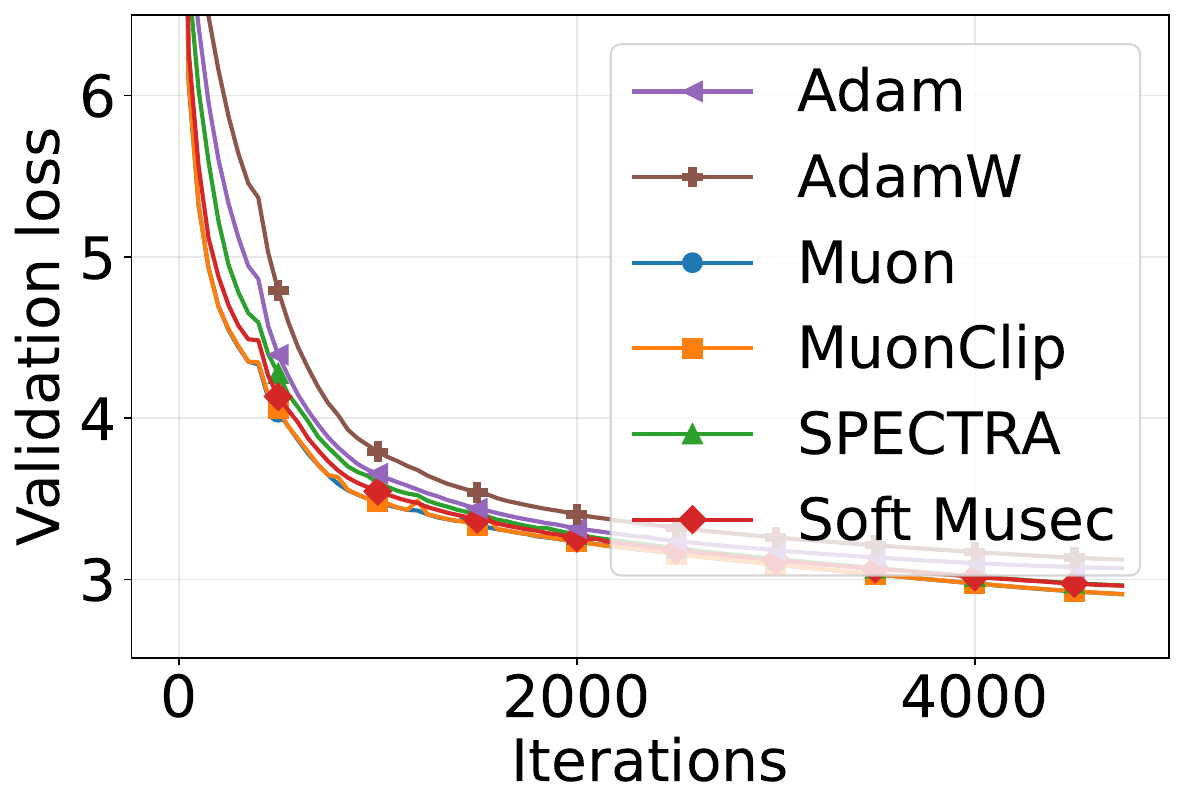} \\
\rowlabel{$\eta = 0.1$} &
\includegraphics[width=0.3\textwidth]{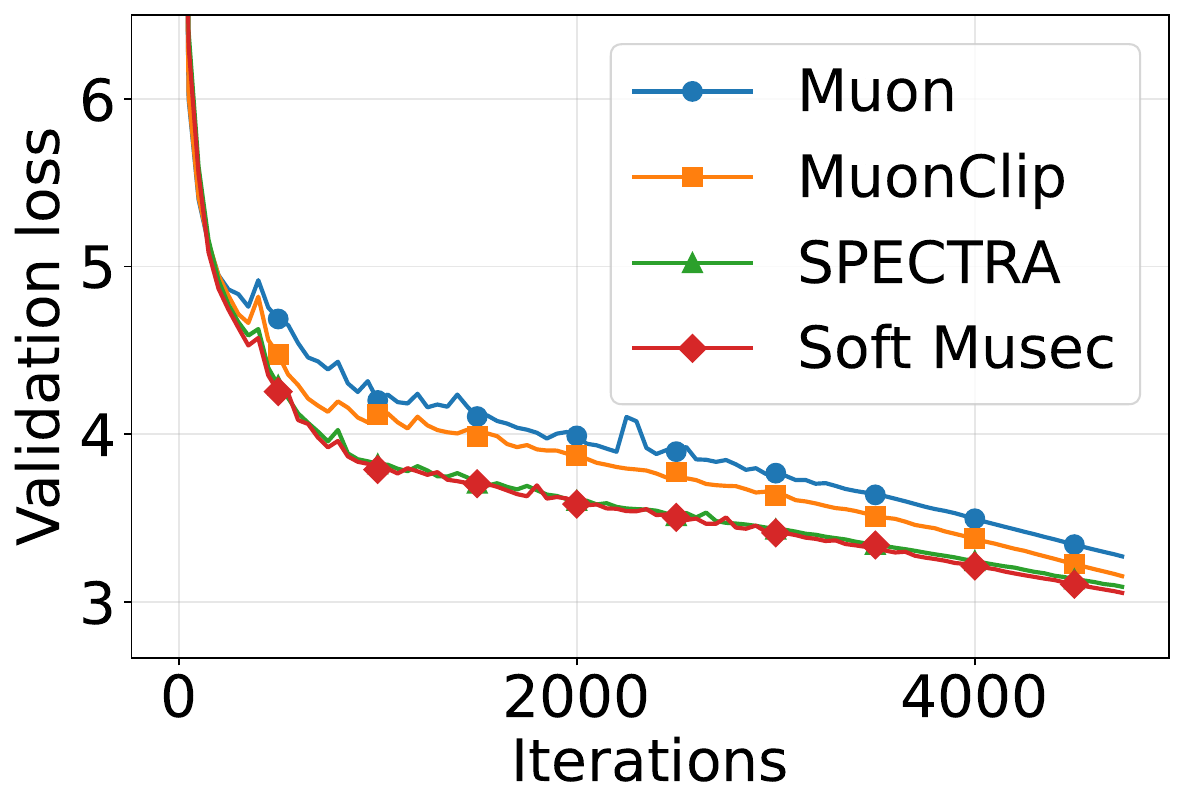} &
\includegraphics[width=0.3\textwidth]{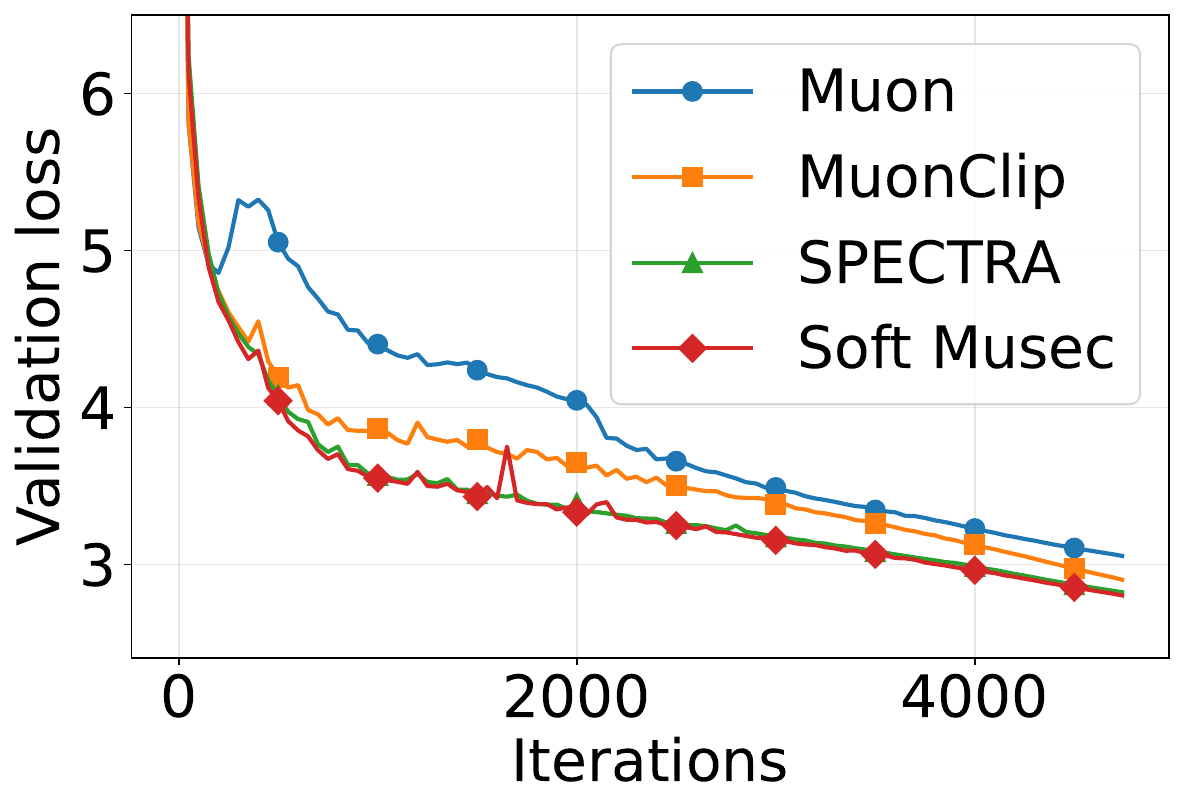} &
\includegraphics[width=0.3\textwidth]{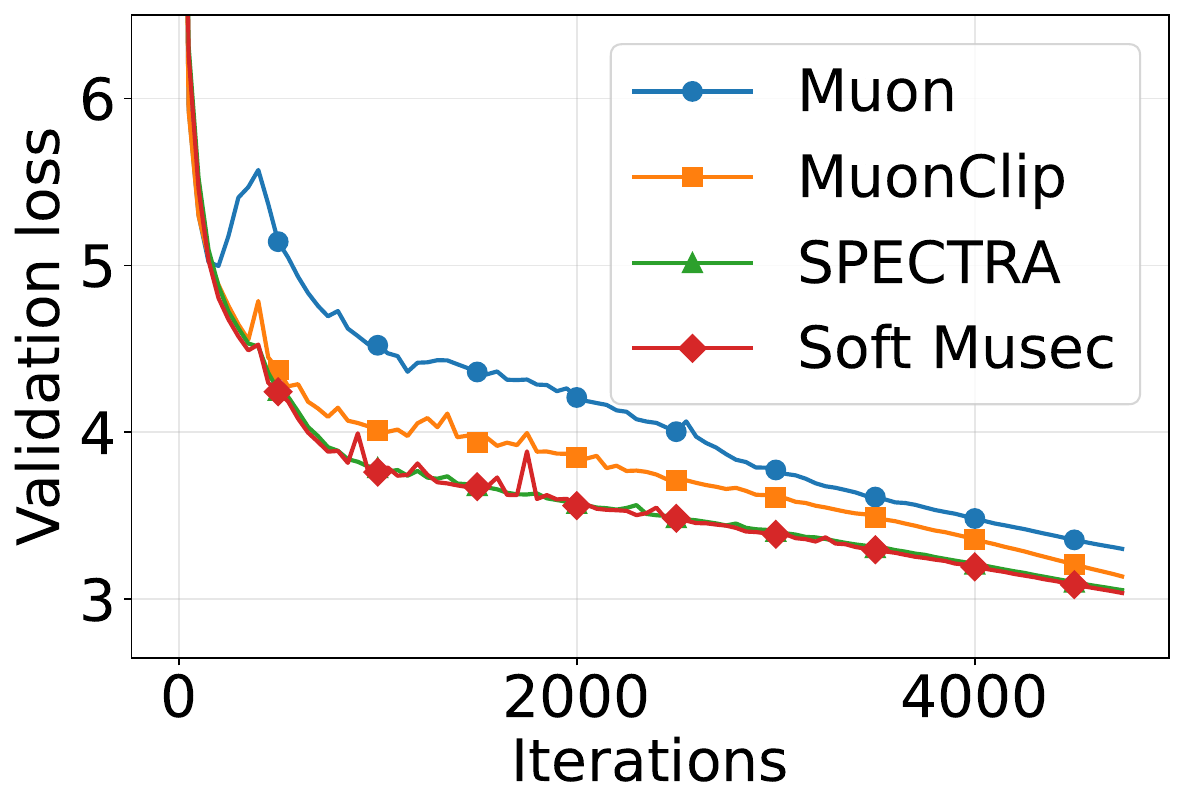} \\
\rowlabel{$\eta = 0.2$} &
\includegraphics[width=0.3\textwidth]{figures/train_dynamics/fineweb/val_loss_vs_steps_medium_0.2.pdf} &
\includegraphics[width=0.3\textwidth]{figures/train_dynamics/openweb/val_loss_vs_steps_medium_0.2.pdf} &
\includegraphics[width=0.3\textwidth]{figures/train_dynamics/c4/val_loss_vs_steps_medium_0.2.pdf} \\
\rowlabel{$\eta = 0.5$} &
\includegraphics[width=0.3\textwidth]{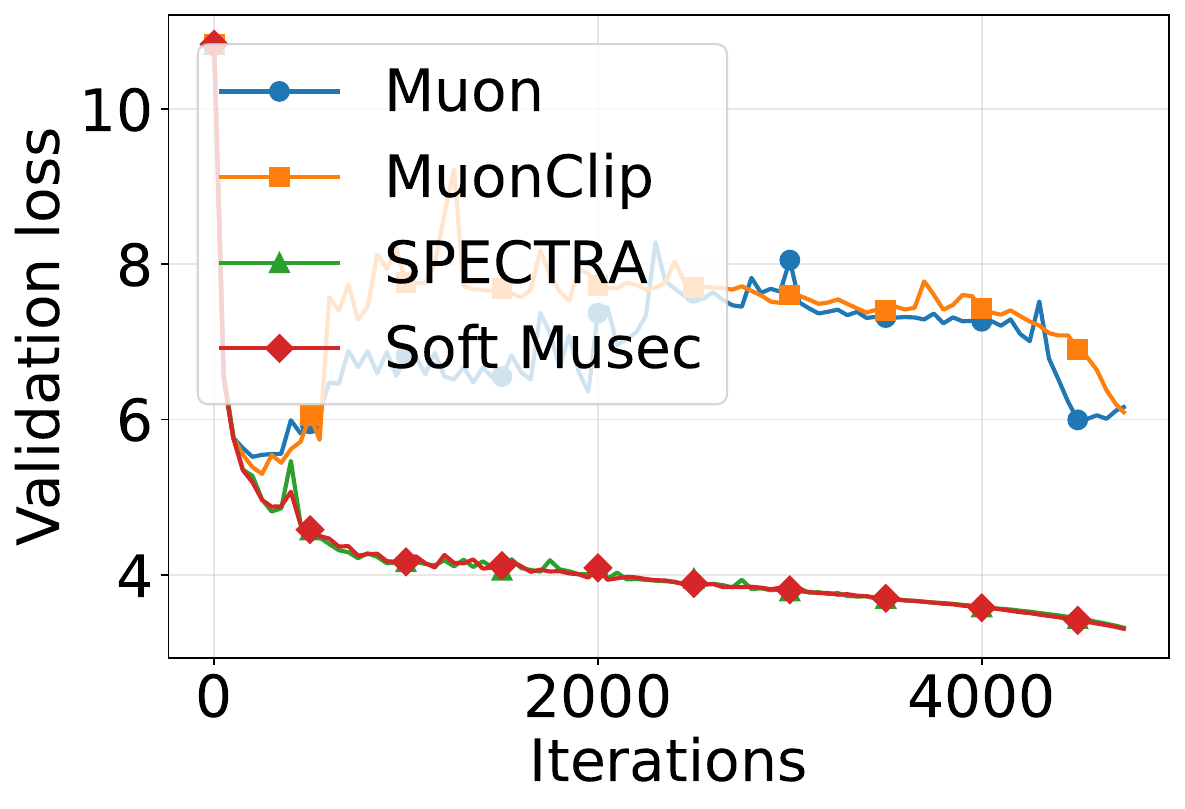} &
\includegraphics[width=0.3\textwidth]{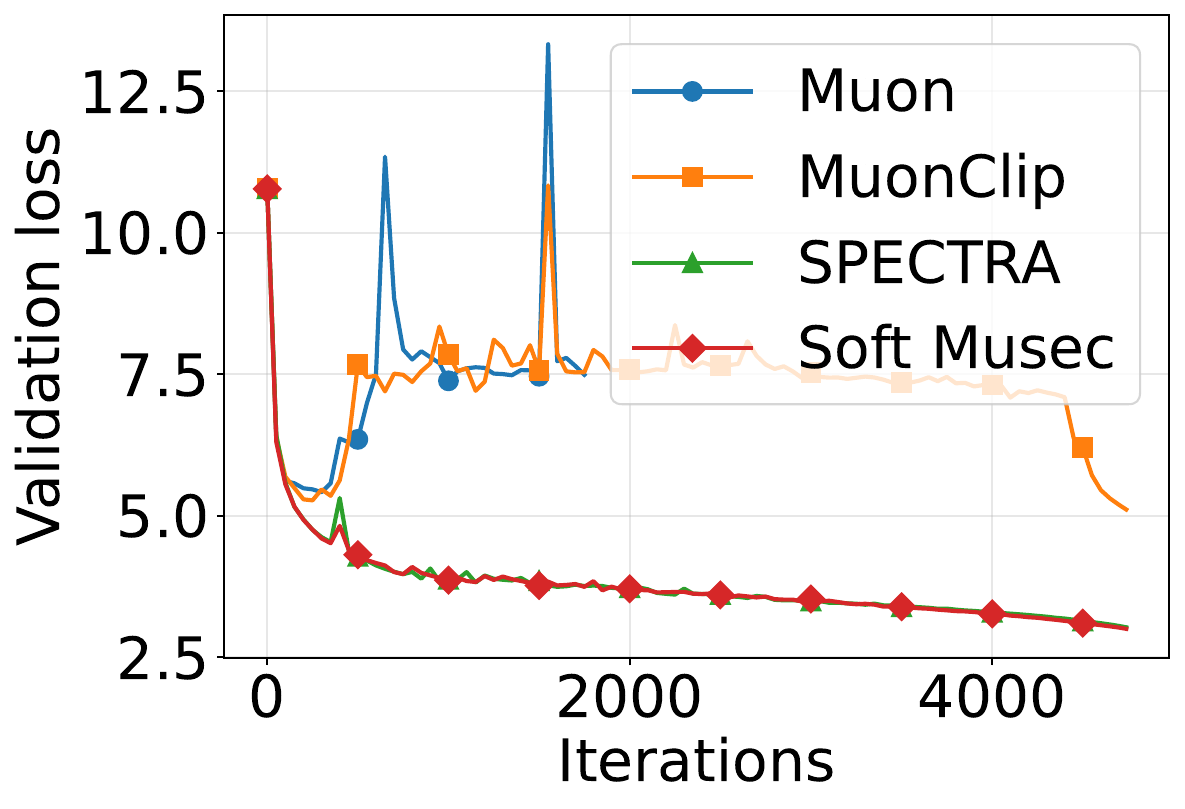} &
\includegraphics[width=0.3\textwidth]{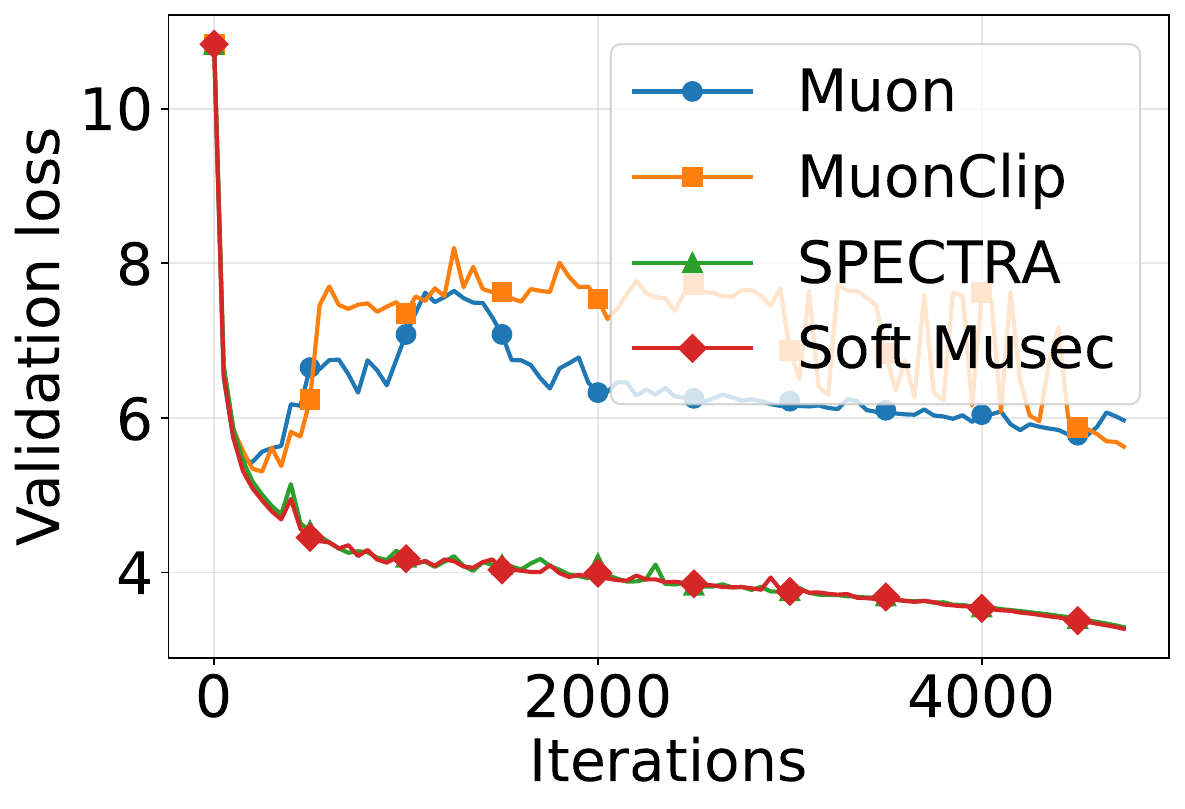} \\
\end{tabular}
\caption{Training loss over steps for NanoGPT-Medium. Rows: learning rates; columns: datasets. Only learning rates where methods exhibit meaningfully different behavior are shown. NaN values within runs are omitted.}
\label{fig:td_medium}
\end{figure}

\begin{figure}[p]
\centering
\begin{tabular}{cccc}
\includegraphics[width=0.4\textwidth]{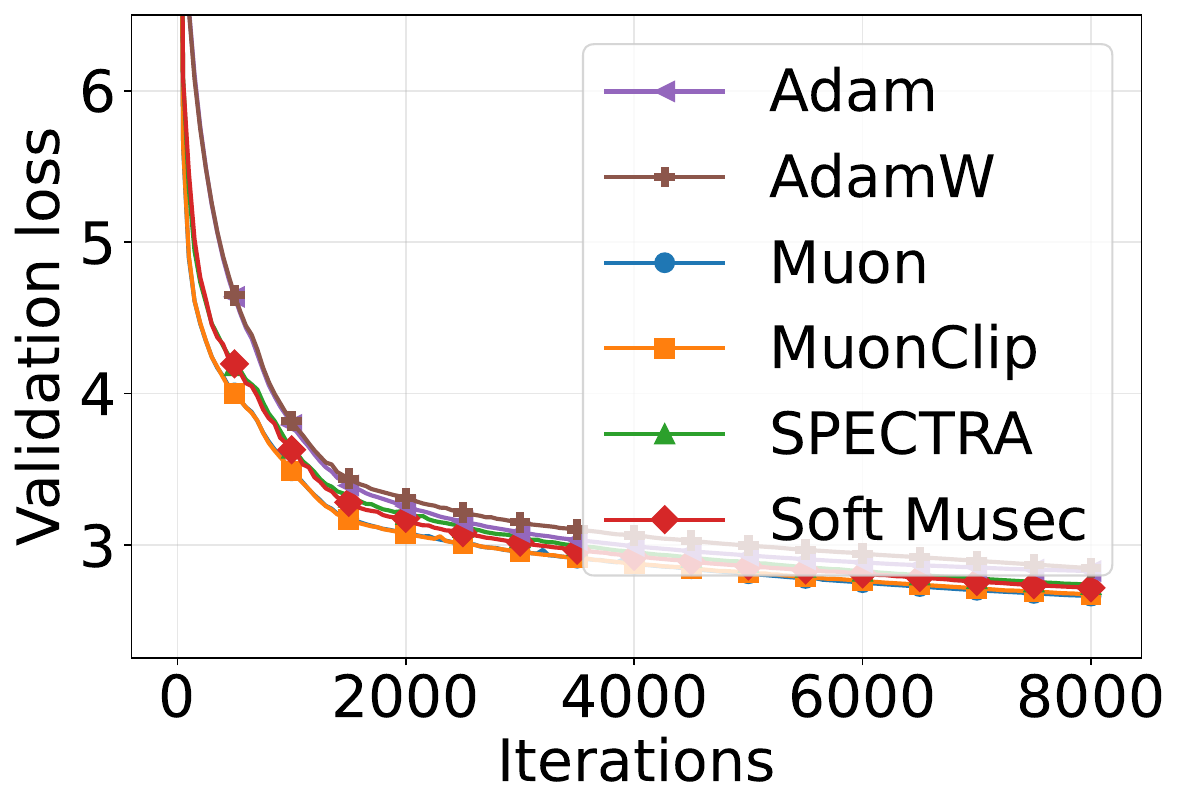}  &
\includegraphics[width=0.4\textwidth]{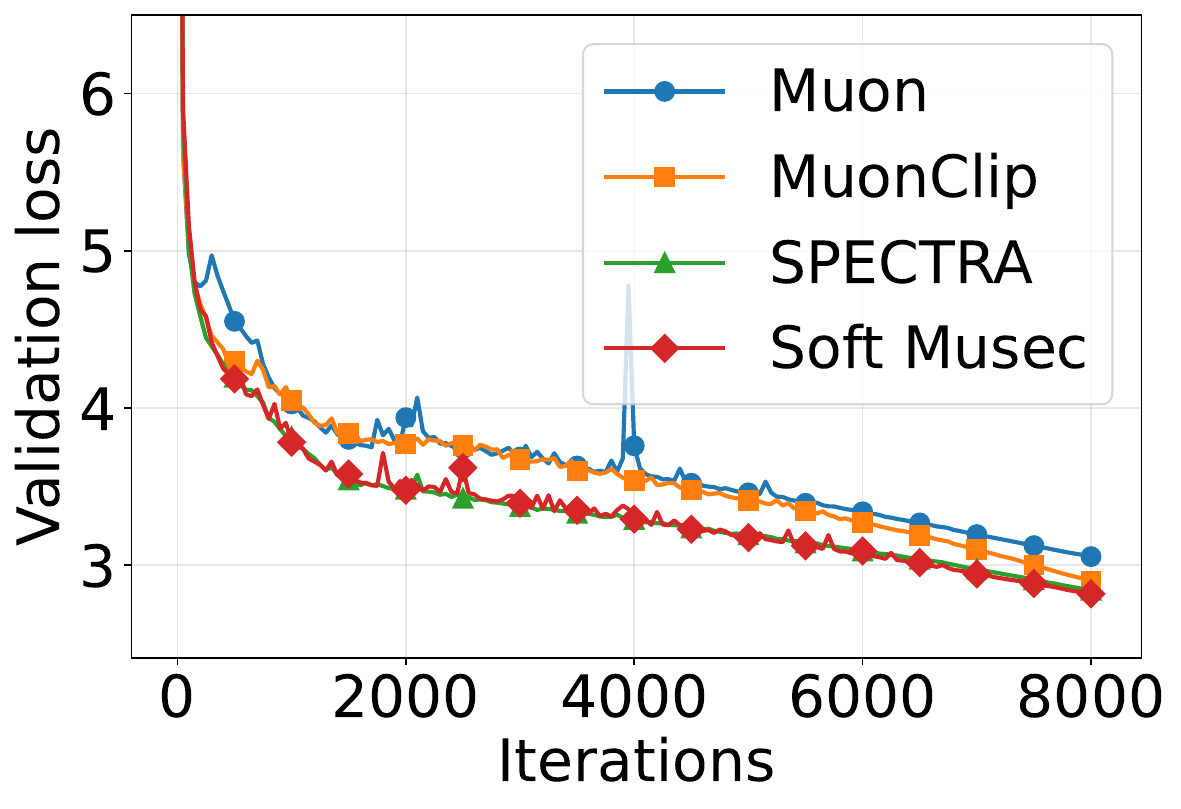}  \\
(a) Best Tuned & (b) $\eta = 0.1$ \\[6pt]
\includegraphics[width=0.4\textwidth]{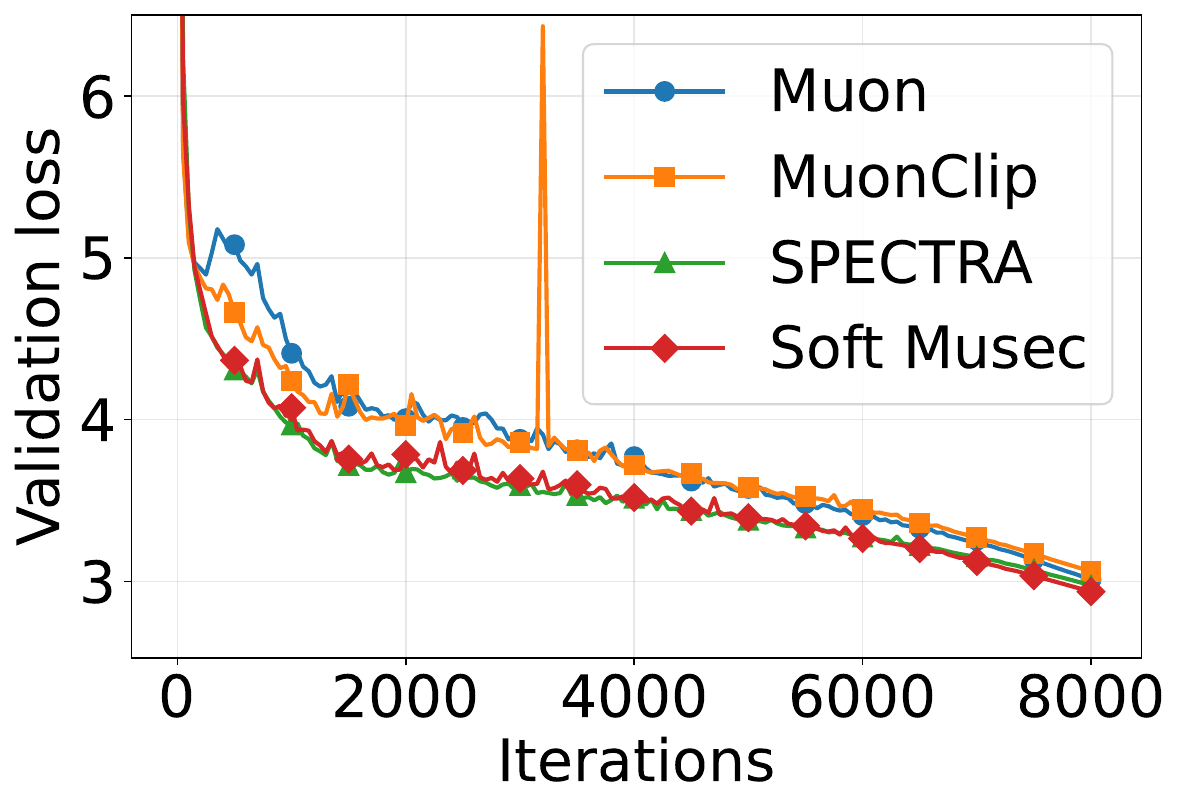}   &
\includegraphics[width=0.4\textwidth]{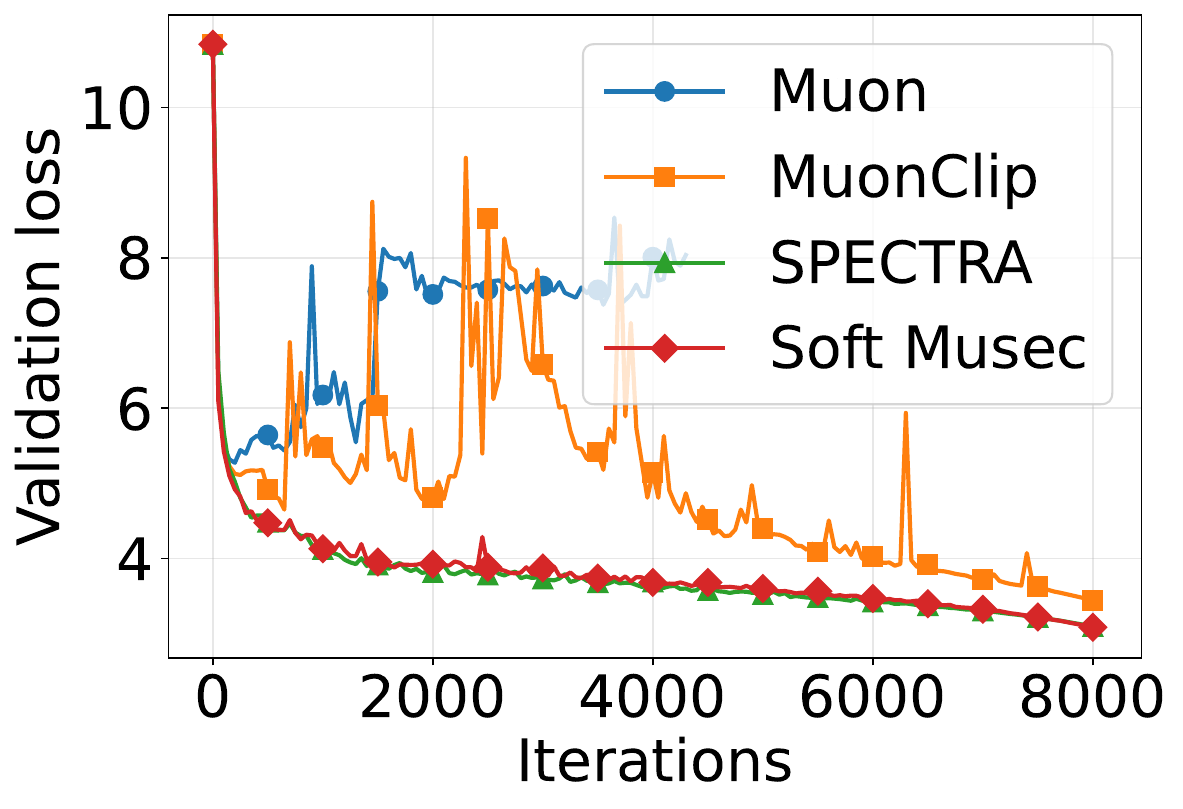}  \\
(c) $\eta = 0.2$ & (d) $\eta = 0.5$ \\
\end{tabular}
\caption{Training loss over steps for NanoGPT-Wide on FineWeb dataset. Only learning rates where methods exhibit meaningfully different behavior are shown. NaN values within runs are omitted.}
\label{fig:td_wide}
\end{figure}

Figures~\ref{fig:td_small}--\ref{fig:td_wide} present training loss curves across all model configurations, datasets, and selected learning rates. At the best-tuned learning rate, all Muon-type optimizers converge to lower validation loss than Adam and AdamW. As the learning rate increases, the stability gap between methods widens, and this effect is amplified at larger model scales. For NanoGPT-Small, instability only manifests at the largest learning rate ($\eta=0.5$), where Muon and MuonClip diverge on some datasets while Soft Musec and SPECTRA remain stable. For NanoGPT-Medium and NanoGPT-Wide, Muon and MuonClip exhibit increasingly severe instability starting from $\eta=0.1$, with large loss spikes and sustained oscillations. 
Soft Musec and SPECTRA both maintain stable convergence across all configurations, with Soft Musec achieving comparable or slightly smoother training dynamics.

\subsection{Effective Rank Analysis}\label{app:effective_rank}
\begin{figure}[htbp]
    \centering
    \begin{tabular}{ccc}
        % First Row
        \includegraphics[width=0.31\textwidth]{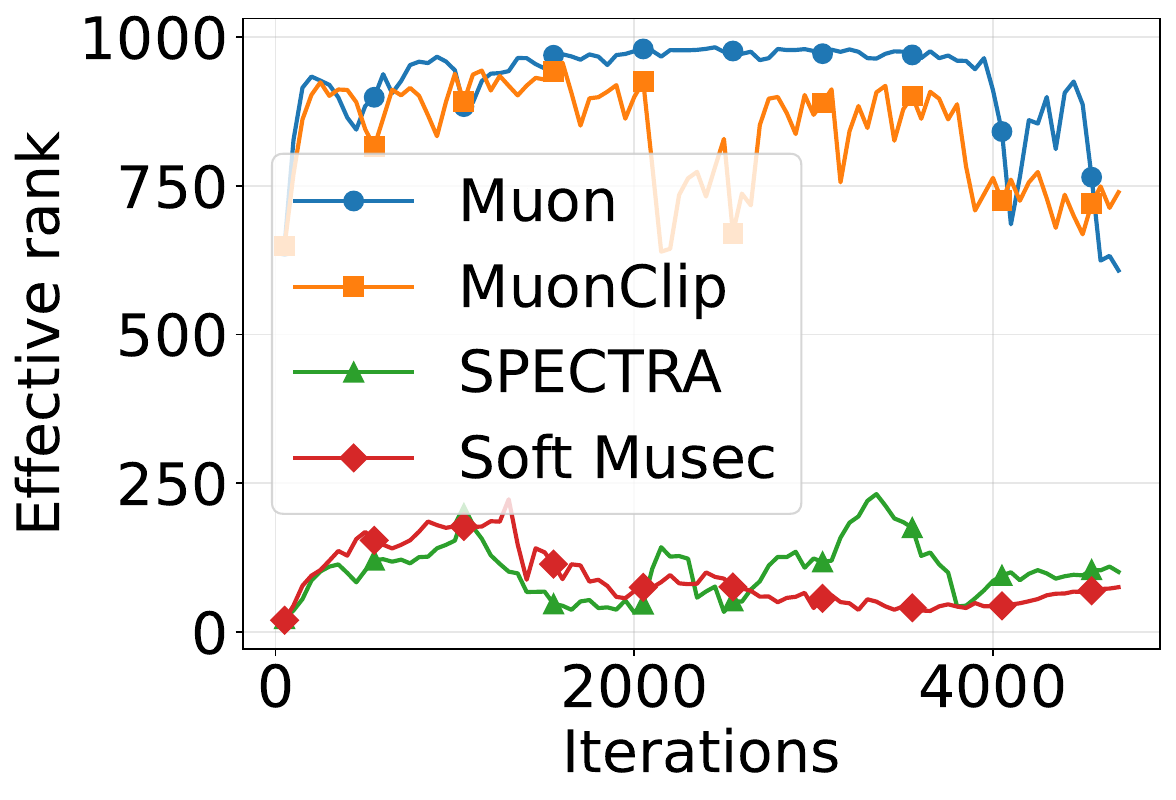} &
        \includegraphics[width=0.31\textwidth]{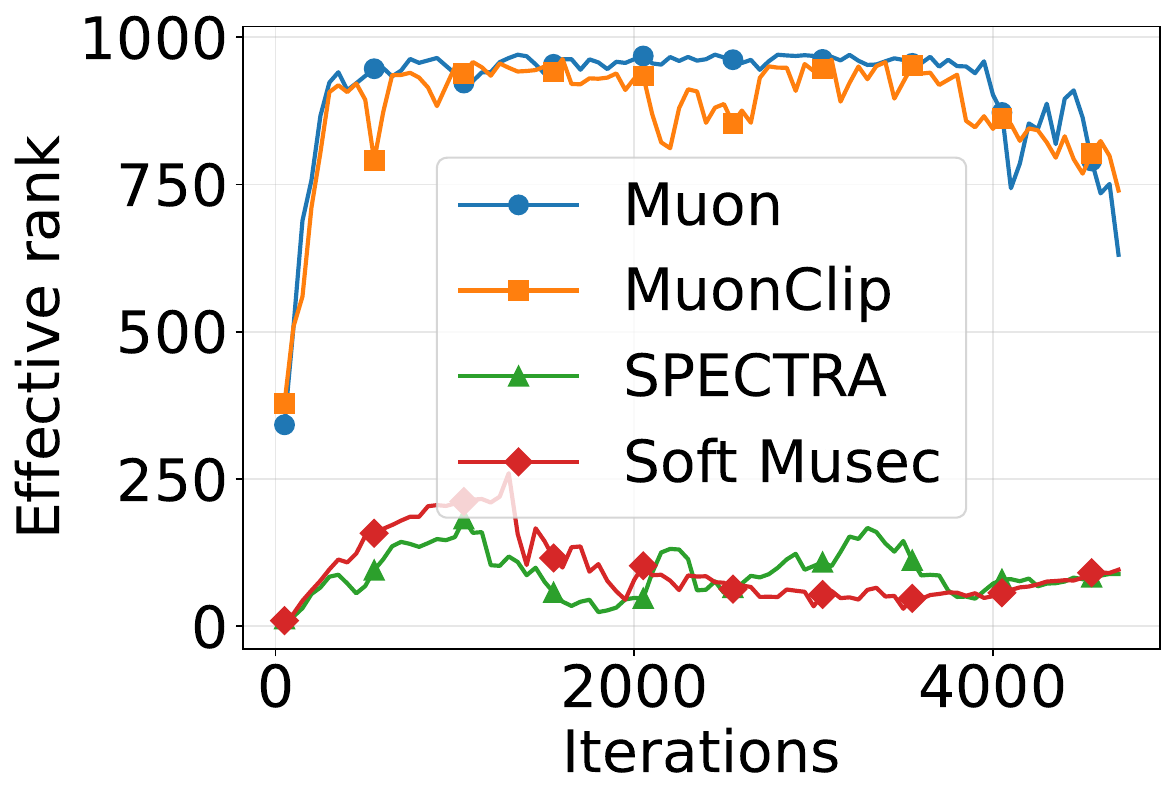} &
        \includegraphics[width=0.31\textwidth]{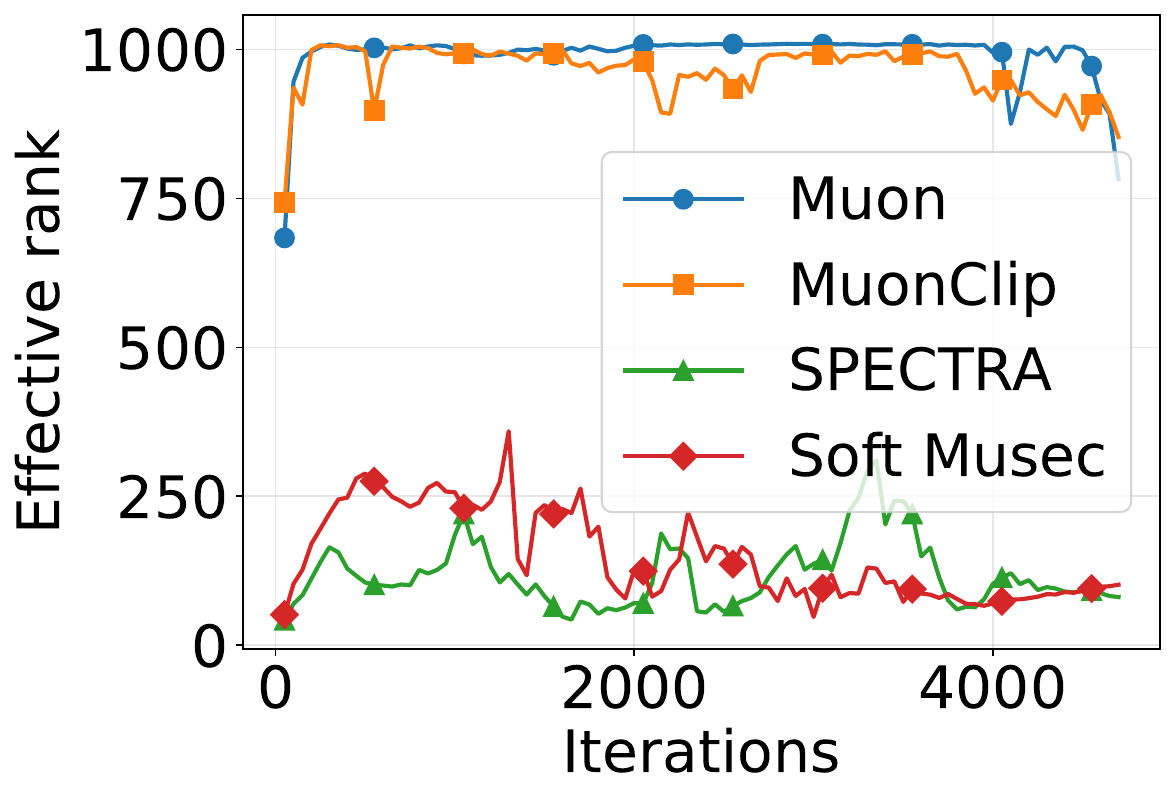} \\
        (a) QK projections & (b) VO projections & (c) MLP \\[1em]
    \end{tabular}
    \caption{Effective rank of update matrices over training steps for NanoGPT-Medium on FineWeb.}
    \label{fig:er_res}
\end{figure}

To directly validate our theoretical motivation, we track the effective rank of the update matrices across QK projections, VO projections, and MLP weights during NanoGPT-Medium training on FineWeb at $\eta = 0.5$.

Let us first recall the definition of the effective rank \citep{roy2007effective,garrido2023rankme} of a matrix $\mA \in \R^{m \times n}$ with singular values $\sigma_1 \geq \sigma_2 \geq \dots \geq \sigma_r \geq 0$, where $r = \min(m,n)$.

We define the singular value distribution as
\begin{align*}
    p_i = \frac{\sigma_i}{\sum_{j=1}^r \sigma_j}, \qquad i \in [r]
\end{align*}
so that $\sum_i p_i = 1$. 
We next compute the Shannon entropy of this distribution as
\begin{align*}
    H(p_1, \dots, p_r) = - \sum_{i=1}^r p_i \log(p_i).
\end{align*}
The effective rank of the matrix $\mA$, denoted ${\rm erank}(\mA)$, is defined as
\begin{align*}
    {\rm erank}(\mA) = \exp(H(p_1, \dots, p_r)).
\end{align*}

The complete results are shown in Figure~\ref{fig:er_res}. Muon exhibits near-full effective rank across all components, confirming that its polar step discards all spectral magnitude information. MuonClip does not meaningfully reduce the effective rank of the update, as it targets the weight matrices rather than the update itself. Both SPECTRA and Soft Musec achieve substantially lower effective rank than Muon, consistent with the effect of spectral clipping. Notably, as training progresses, Soft Musec maintains lower effective rank than SPECTRA across all three component types. This is consistent with the momentum feedback mechanism in Musec: by propagating the clipped momentum, spectral components are continuously regulated rather than allowed to accumulate in an unclipped buffer, as discussed in Section~\ref{sec:methodology}.

\subsection{Computational Overhead}\label{app:overhead}

\begin{figure}[htbp]
    \centering
    \begin{tabular}{ccc}
        % First Row
        \includegraphics[width=0.31\textwidth]{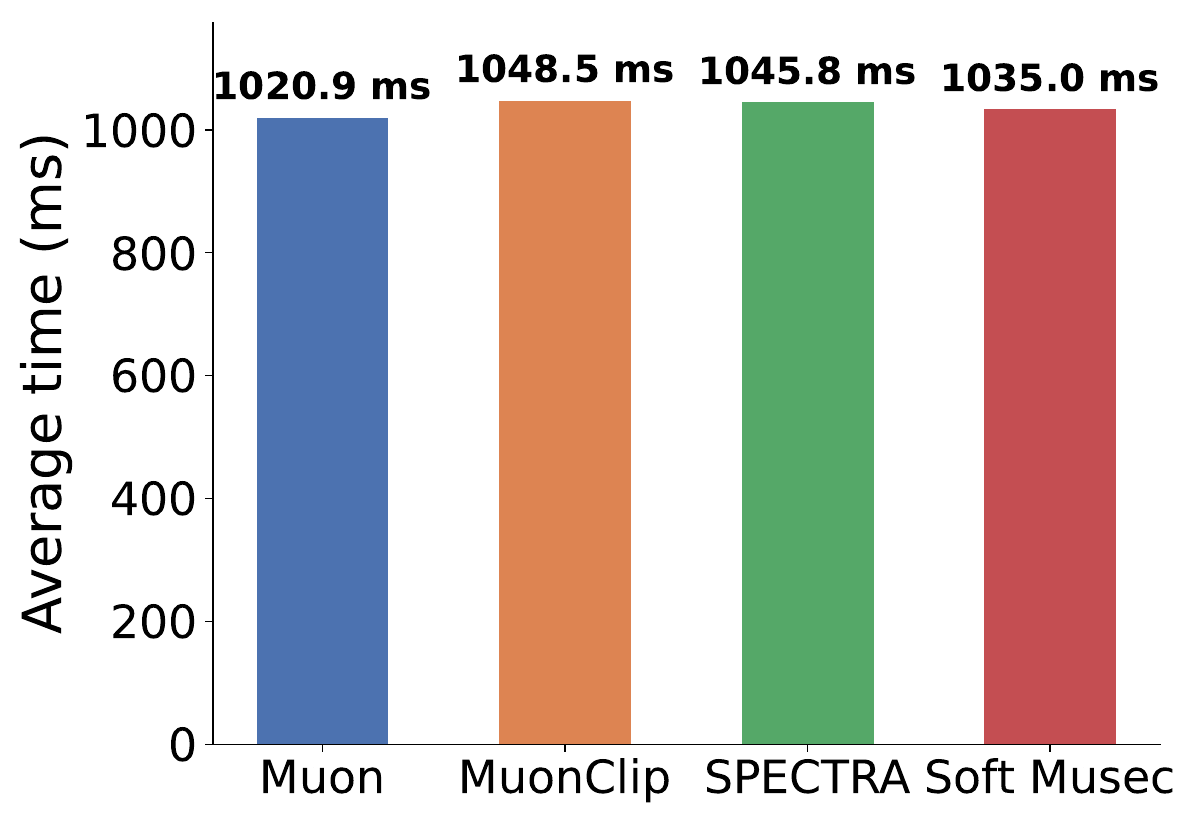} &
        \includegraphics[width=0.31\textwidth]{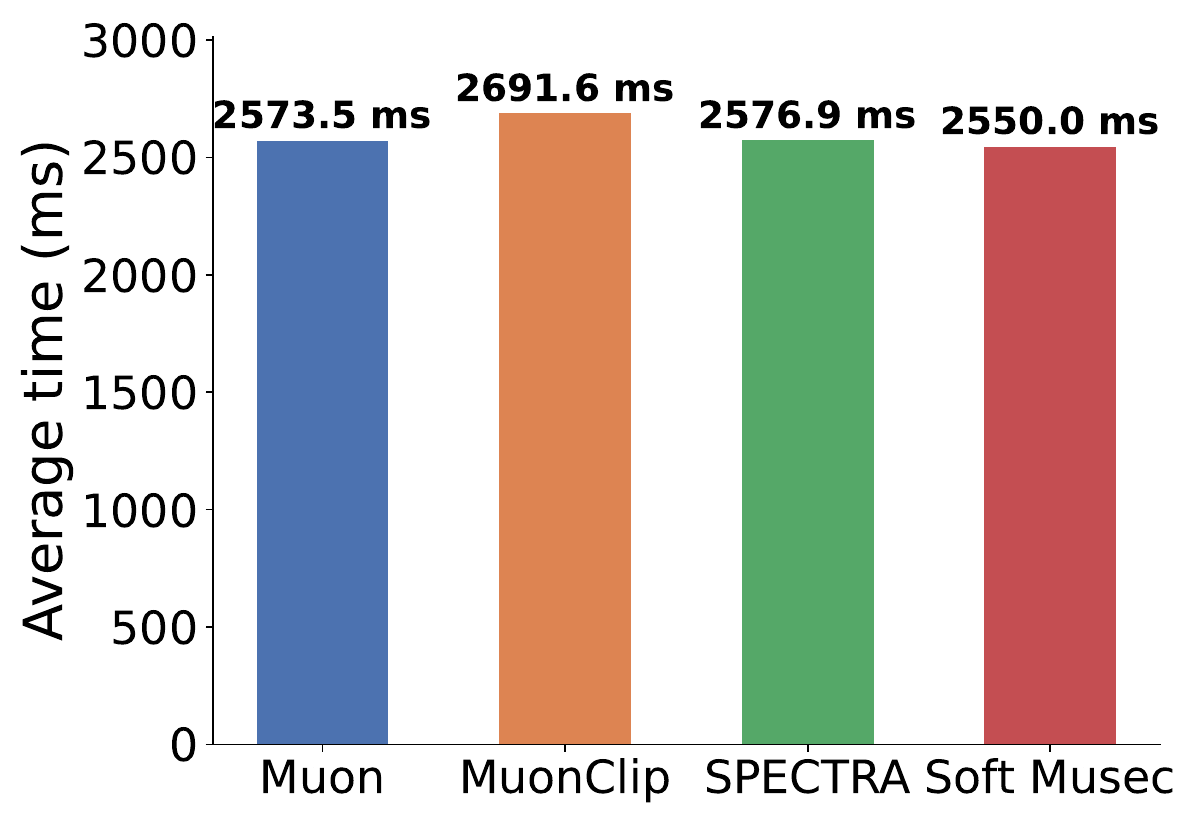} &
        \includegraphics[width=0.31\textwidth]{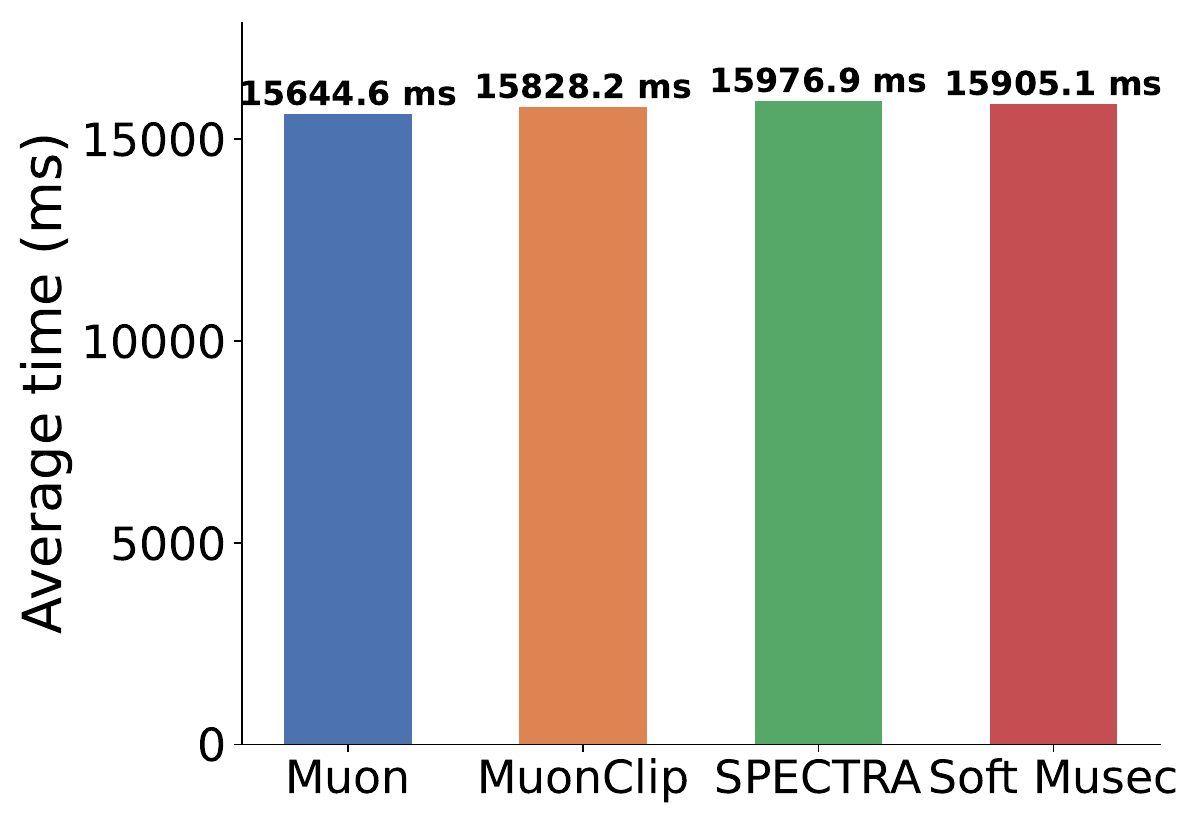} \\
        (a) NanoGPT-Small & (b) NanoGPT-Medium & (c) NanoGPT-Wide \\[1em]
    \end{tabular}
    \caption{Average wall clock time comparison of methods for training FineWeb.}
    \label{fig:avg_time}
\end{figure}

We compare the average wall-clock time per training step across all methods on FineWeb, with results shown in Figure~\ref{fig:avg_time}. Across all three model configurations, Soft Musec introduces negligible computational overhead compared to Muon. For NanoGPT-Small, all methods are within $3\%$ of each other (1021--1049 ms). 
For NanoGPT-Medium, all Muon-type methods achieve similar wall-clock time (2550--2692 ms), with MuonClip being the slowest.
For NanoGPT-Wide, all methods remain within $2\%$ (15645--15977 ms). 
MuonClip is consistently the slowest method due to its additional weight clipping step.
These results confirm that the spectral clipping operation in Soft Musec introduces minimal computational overhead relative to Muon, as both methods rely on Newton--Schulz iterations of similar complexity.

\subsection{Ablation Studies}\label{app:ablation}

In this subsection, we ablate two key hyperparameters of Soft Musec: the clipping threshold $D$ and the number of Newton--Schulz iterations.

\subsubsection{Sensitivity to Clipping Threshold}
\begin{figure}[htbp]
    \centering
    \begin{tabular}{ccc}
        % First Row
        \includegraphics[width=0.31\textwidth]{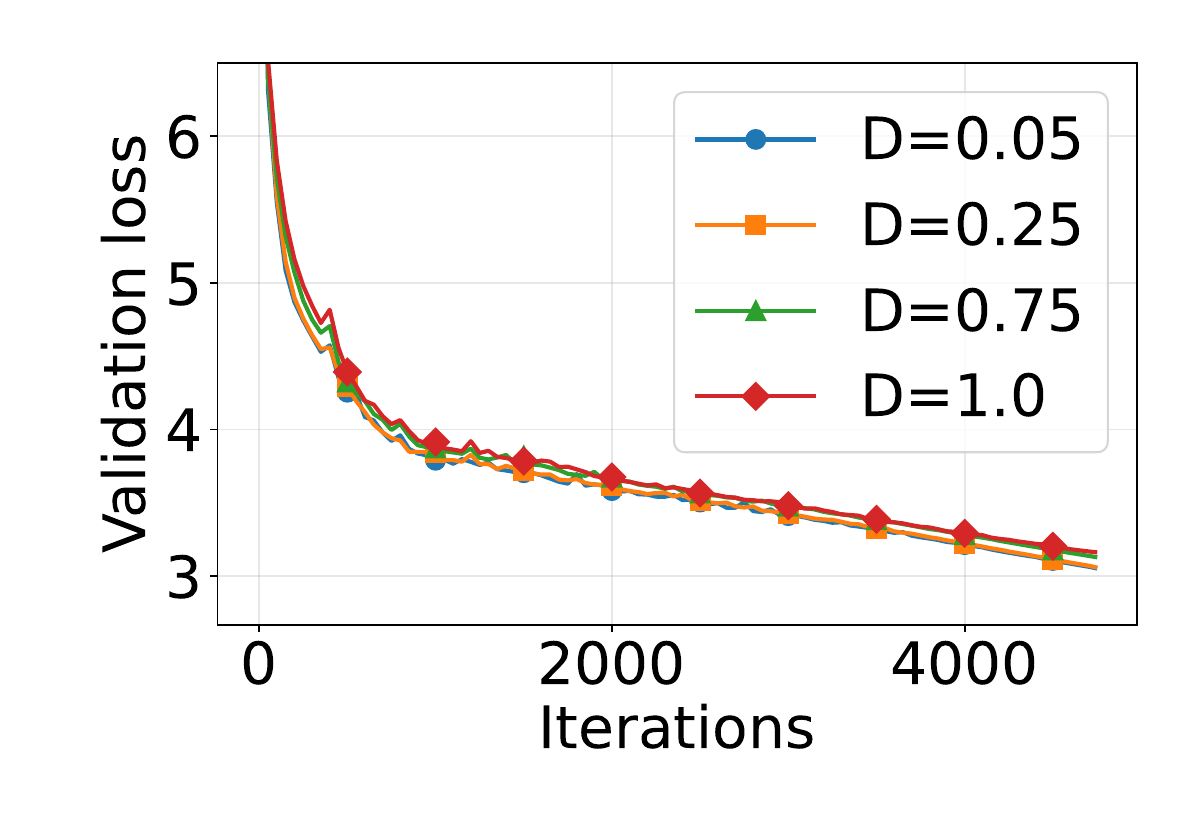} &
        \includegraphics[width=0.31\textwidth]{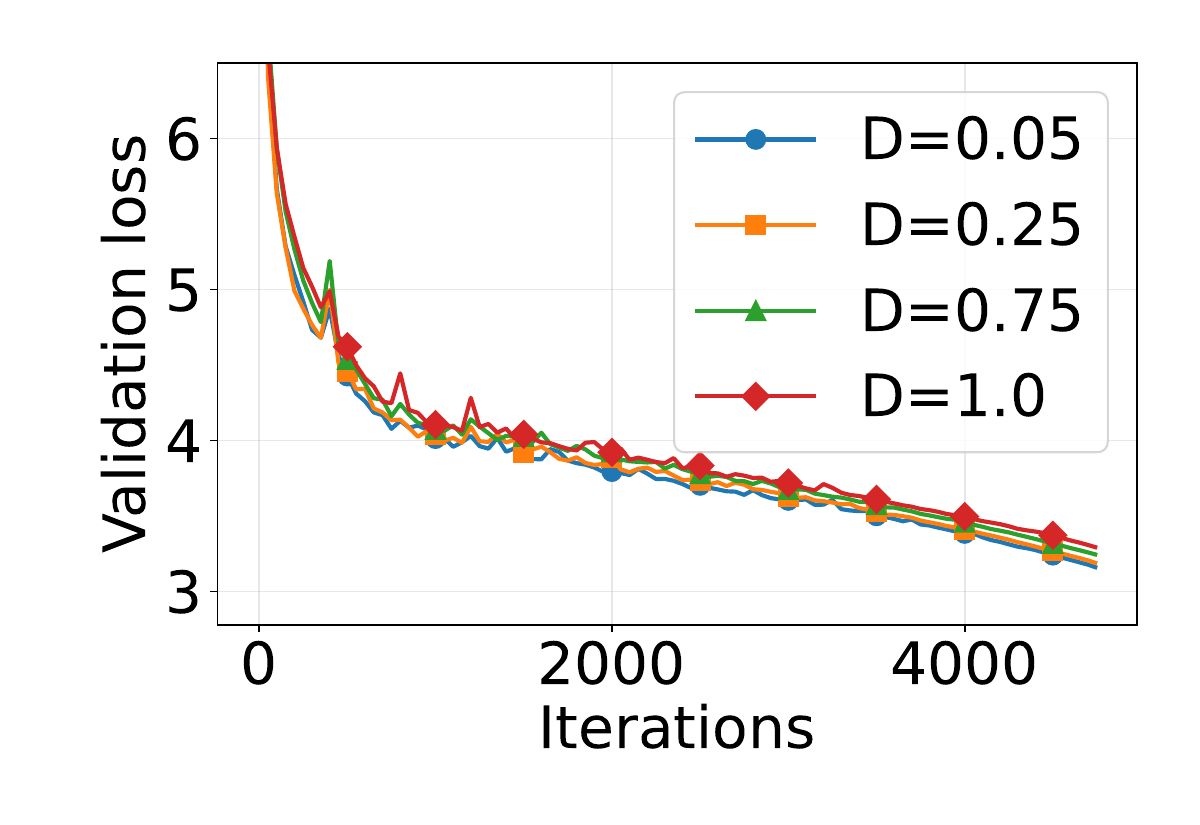} &
        \includegraphics[width=0.31\textwidth]{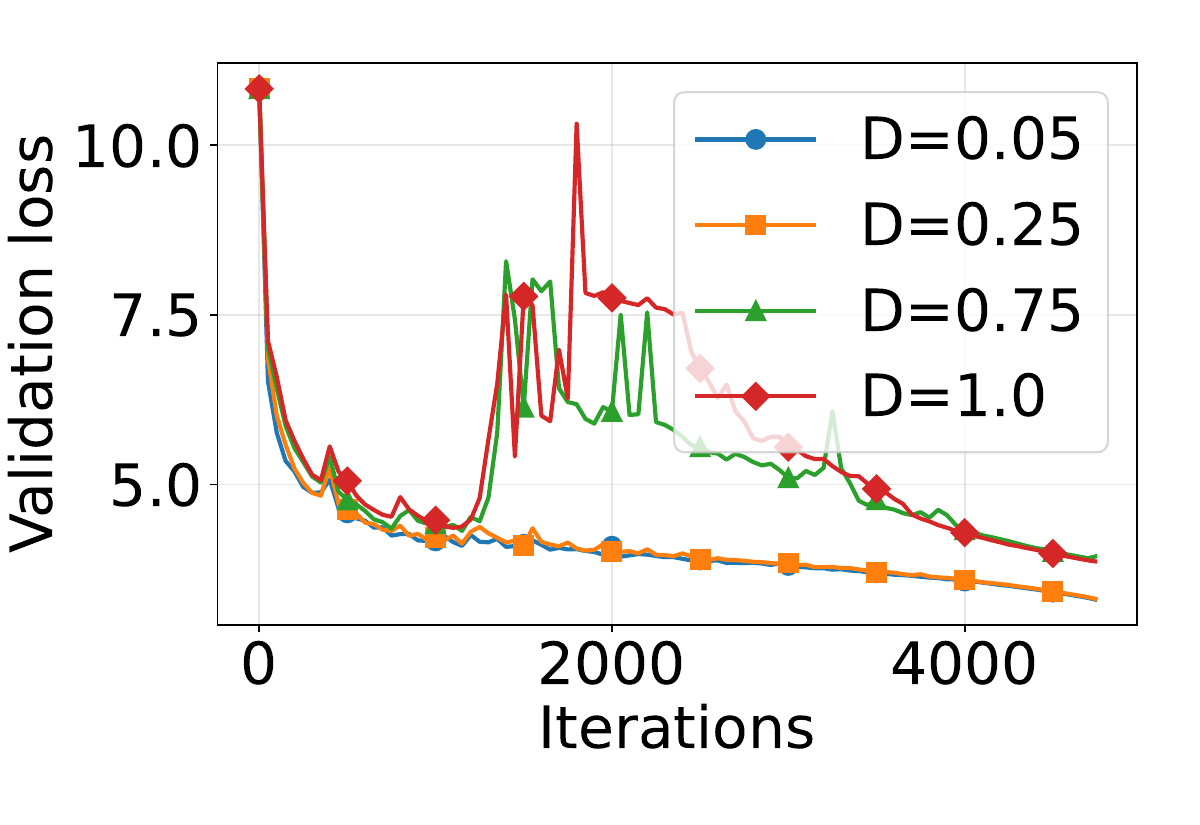} \\
        (a) $\eta = 0.1$ & (b) $\eta = 0.2$ & (c) $\eta = 0.5$ \\[1em]
    \end{tabular}
    \caption{Effect of clipping threshold on Soft Musec training stability across learning rates.}
    \label{fig:clip_ablation}
\end{figure}

We study the sensitivity of Soft Musec to the clipping threshold $D$ on NanoGPT-Medium trained on FineWeb across three learning rates $(\eta \in \{0.1, 0.2, 0.5\})$, with results shown in Figure~\ref{fig:clip_ablation}.

At moderate learning rates ($\eta=0.1$ and $\eta=0.2$), Soft Musec is robust to the choice of $D$: all values from $D=0.05$ to $D=1.0$ converge to comparable validation loss with similar training dynamics. At the largest learning rate ($\eta=0.5$), the choice of $D$ becomes more important. Smaller thresholds ($D=0.05$ and $D=0.25$) maintain smooth convergence, while larger thresholds ($D=0.75$ and $D=1.0$) exhibit loss spikes in the later stages of training. Notably, even with the largest clipping threshold ($D=1.0$), Soft Musec still converges — in contrast to Muon and MuonClip, which diverge entirely at this learning rate (Figure~\ref{fig:td_medium}). This demonstrates that spectral clipping provides a meaningful stability benefit regardless of the threshold, while a well-tuned $D$ further improves training smoothness.

\subsubsection{Number of Newton-Schulz Iterations}
\begin{figure}[htbp]
    \centering
    \begin{tabular}{ccc}
        % First Row
        \includegraphics[width=0.31\textwidth]{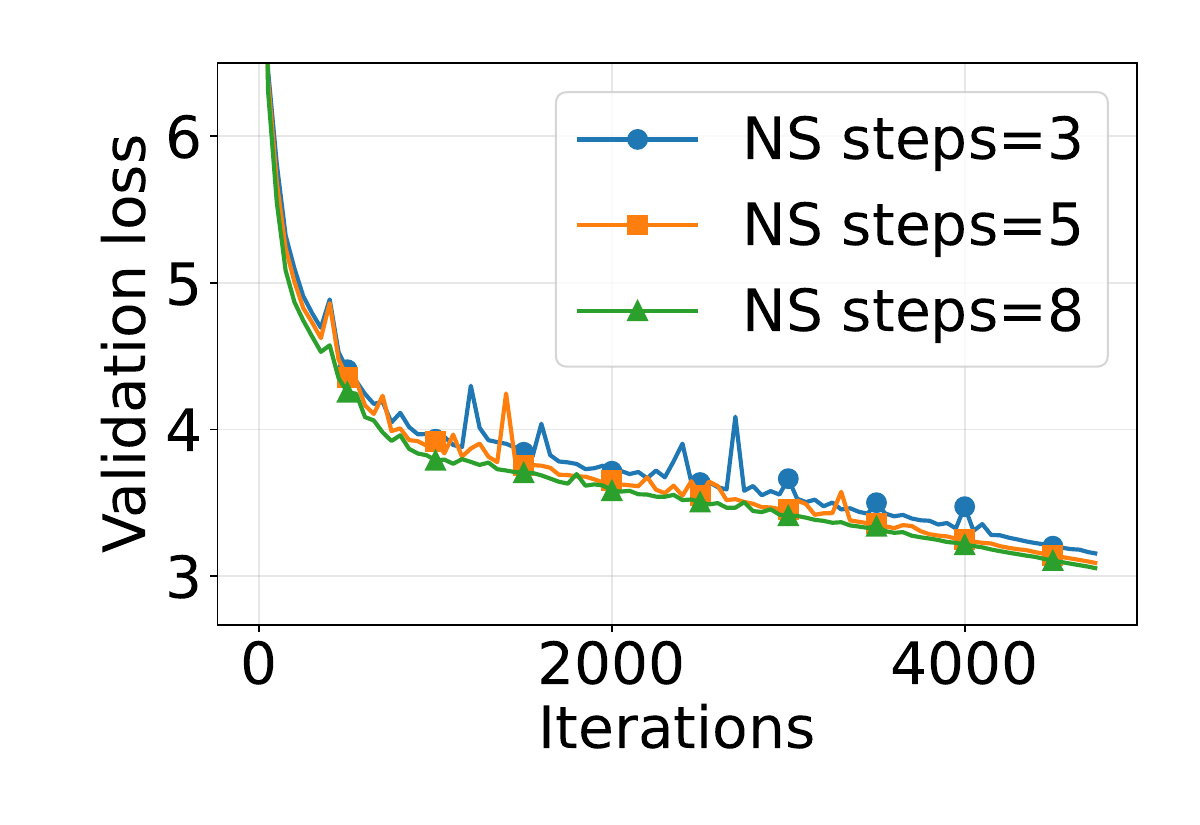} &
        \includegraphics[width=0.31\textwidth]{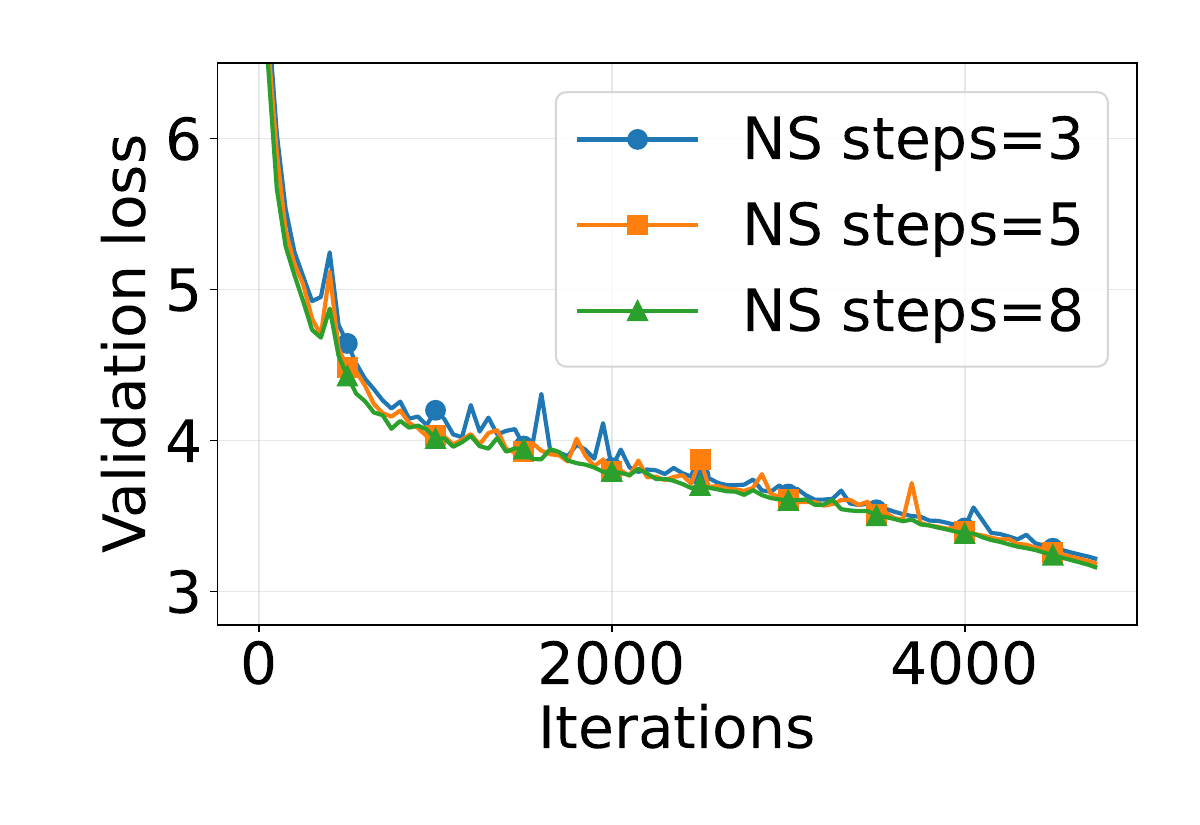} &
        \includegraphics[width=0.31\textwidth]{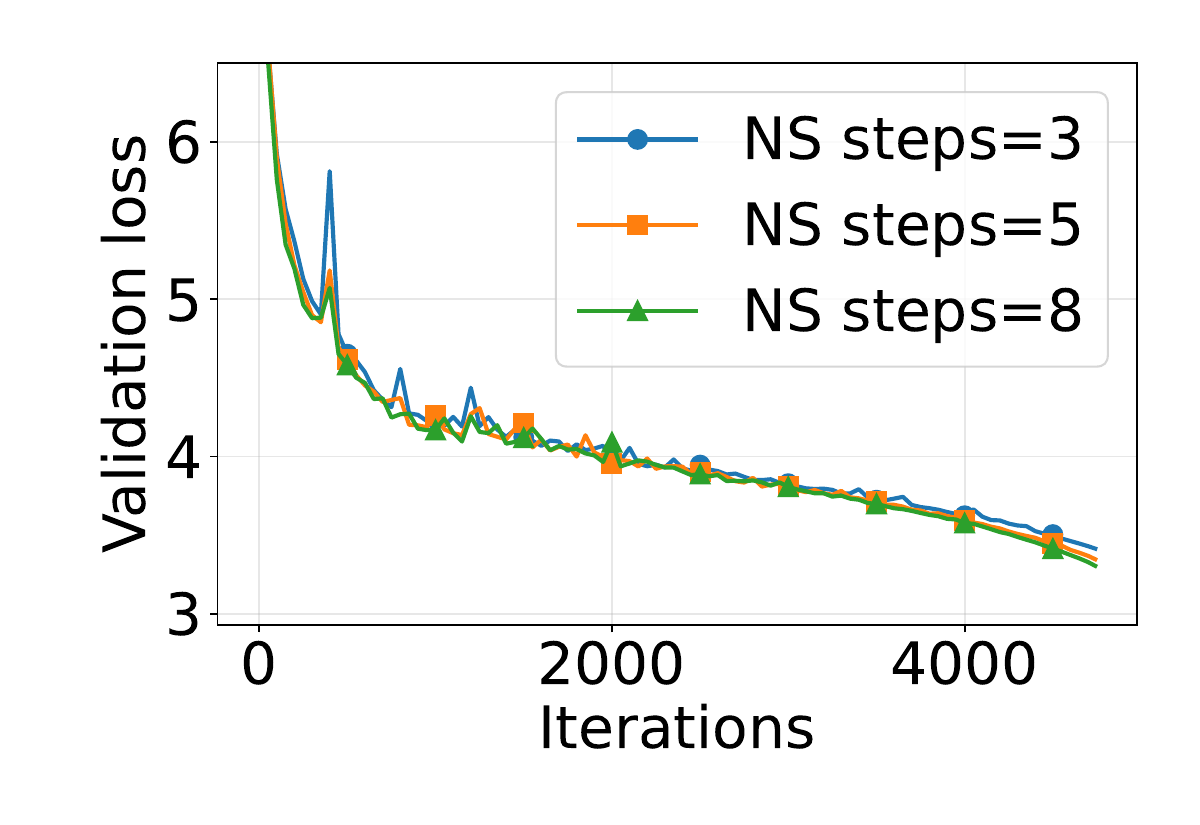} \\
        (a) $\eta = 0.1$ & (b) $\eta = 0.2$ & (c) $\eta = 0.5$ \\[1em]
    \end{tabular}
    \caption{Effect of Newton--Schultz steps on Soft Musec training stability across learning rates.}
    \label{fig:ns_ablation}
\end{figure}

We study the sensitivity of Soft Musec to the number of Newton--Schulz iterations on NanoGPT-Medium trained on FineWeb across three learning rates ($\eta \in \{0.1, 0.2, 0.5\}$), with the clipping threshold fixed at $D=0.05$. Results are shown in Figure~\ref{fig:ns_ablation}.

Across all learning rates, Soft Musec is largely insensitive to the number of Newton--Schulz iterations. More iterations yield slightly better validation loss, consistent with a more accurate approximation of the spectral clipping operator. 
However, the differences are marginal, indicating that even a coarse approximation with three iterations provides most of the stability benefit. This robustness stems from the well-conditioned nature of the matrix $\widehat\mM \widehat{\mM}^\top + D^2\mI$, whose eigenvalues are bounded below by $D^2$, enabling rapid convergence of the iteration. 
In practice, five iterations offer a good balance between approximation quality and computational cost.

\end{document}